%% file: main.tex
\documentclass[twoside]{article}

\usepackage[preprint]{aistats2027}
\usepackage[utf8]{inputenc} 
\usepackage[T1]{fontenc}    
\usepackage{url}            
\usepackage{booktabs}       
\usepackage{amsfonts}       
\usepackage{nicefrac}       
\usepackage{microtype}      
\usepackage[dvipsnames]{xcolor}         

\usepackage[
    backend=biber,
    style=phys,
    sorting=none,
    sortcites=true,
    natbib=true,
    url=false, 
    doi=true,
    eprint=false,
    isbn=false,
    giveninits=true,
    date=year,
    maxcitenames=2,
    maxbibnames=5,
    hyperref=true
]{biblatex}

\DeclareFieldFormat
  [article, inbook, incollection, inproceedings, patent, thesis, unpublished]
  {title}{#1\isdot}

\AtEveryBibitem{
\ifentrytype{article}{
    \clearfield{number}
    \clearfield{series}
}{}
}

\renewbibmacro{in:}{%
  \ifentrytype{inproceedings}{}{\printtext{\bibstring{in}\intitlepunct}}}

\DeclareNameAlias{sortname}{given-family}

\usepackage{booktabs}
\usepackage{amsmath}
\usepackage{amssymb}
\usepackage{amsthm}
\usepackage{mathtools}
\usepackage{scalerel,stackengine} 
\usepackage{graphicx}
\usepackage{float}
\usepackage{derivative}
\usepackage{enumitem}
\usepackage{bm}
\usepackage{upgreek}
\usepackage{derivative}
\usepackage{tikz}
\usetikzlibrary{tikzmark}
\usepackage{algpseudocodex}
\usepackage{algorithm}
\usepackage{siunitx}
\usepackage{afterpage}
\usepackage{subcaption}
\usepackage{wrapfig}
\usepackage[colorlinks=true,citecolor=BlueViolet,linkcolor=BlueViolet]{hyperref}

\DeclareSIUnit\hartree{\text{\ensuremath {E}}_{\textup{h}}}
\newcommand{\printmem}[2]{\SI[round-mode=places, round-precision=#1, scientific-notation = engineering, exponent-to-prefix]{#2}{\byte}}

\newcommand{\dhstrut}[2]{\rule[-#1]{0pt}{#1}\rule[0pt]{0pt}{#2}}  

\newtheoremstyle{bfnote}%
{}{}%
{\itshape}{}%
{\bfseries}{.}%
{ }%
{\thmname{#1}\thmnumber{ #2}\thmnote{ (#3)}}

\theoremstyle{bfnote}

\newtheorem{thm}{Theorem}[section]
\newtheorem{definition}[thm]{Definition}
\newtheorem{cor}[thm]{Corollary}
\newtheorem{exmp}[thm]{Example}

\newcommand{\trans}{\mspace{-0mu}\mathsf{T}}
\newcommand{\norm}[1]{\lVert#1\rVert}
\newcommand{\weights}{\bm{\upalpha}}
\newcommand{\gridmcr}{\mathcal{M}_{\mathrm{grid}}}
\newcommand{\kernelmcr}{\mathcal{M}_{\mathrm{kernel}}}
\newcommand{\hatcal}[1]{\widehat{\mathcal{#1}}}
\newcommand{\hatbold}[1]{\widehat{\mathbf{#1}}}
\newcommand{\hatboldsuper}[2]{\widehat{\mathbf{#1}^{#2}}}
\newcommand{\hatboldsupertilde}[2]{\widehat{\bar{\mathbf{#1}}^{#2}}}
\newcommand{\condition}[1]{\frac{\lambda_{\mathrm{max}}(#1)}{\lambda_{\mathrm{min}}(#1)}}
\DeclareMathOperator{\trace}{tr}

\stackMath
\newcommand\reallywidehat[1]{%
\savestack{\tmpbox}{\stretchto{%
  \scaleto{%
    \scalerel*[\widthof{\ensuremath{#1}}]{\kern.5pt\mathchar"0362\kern.5pt}%
    {\rule{0ex}{\textheight}}
  }{\textheight}%
}{2.4ex}}%
\stackon[-6.9pt]{#1}{\tmpbox}%
}
\makeatletter
\renewenvironment{dcases}[1][l]{\matrix@check\dcases\env@dcases{#1}}{\endarray\right.}
\def\env@dcases#1{%
  \let\@ifnextchar\new@ifnextchar
  \left\lbrace\def\arraystretch{1.2}%
  \array{@{}#1@{\quad}l@{}}}
\makeatother

\input{abbreviations.tex}

\newcommand{\ans}{\\[0.1em] \textit{Answer}:}
\newcommand{\jus}{\\[0.1em] \textit{Justification}:}

\begin{document}

%

%

\twocolumn[

\aistatstitle{Don’t Get Your Kroneckers in a Twist: Gaussian Processes on High-Dimensional Incomplete Grids}

\aistatsauthor{ Mads G. Højlund\textsuperscript{1} \And August S. Lykke-Møller\textsuperscript{1} \And  Henry Moss\textsuperscript{2} \And Ove Christiansen\textsuperscript{1} }

\aistatsaddress{ \textsuperscript{1}Department of Chemistry, Aarhus University \\ \textsuperscript{2} School of Mathematical Sciences, Lancaster University } ]

\begin{abstract}
  We introduce CUTS-GPR, a new method for performing numerically exact \ac{gpr}
  in high-dimensional settings. 
  The key component of CUTS-GPR is an extremely fast kernel matrix-vector product, 
  which exhibits near-linear or even linear scaling with the amount of training data, $N$,
  and low-order polynomial scaling with dimensionality, $D$.
  This is obtained by combining an \textit{additive kernel} with an 
  \textit{incomplete grid} and exploiting the resulting structure of the kernel matrix.
  The scalability of the matrix-vector product is verified by benchmarks with billions of data points and thousands of dimensions.   
  We demonstrate the end-to-end scalability of CUTS-GPR by running full GPR calculations, including hyperparameter optimization,
  on synthetic datasets with up to $N = \num{4 494 001}$ and $D = 500$.
  As a realistic and challenging test, we finally apply CUTS-GPR to a set of ten \acp{pes} with $N = \num{447 265}$ and $D = 24$.
  The calculations are completed in a matter of hours, showing that CUTS-GPR enables Bayesian modelling
  of high-dimensional \acp{pes} -- a longstanding challenge in computational chemistry.
\end{abstract}


\acresetall

\section{INTRODUCTION} \label{sec:introduction}

\Ac{gpr} is a powerful nonparametric, probabilistic machine learning method that enables 
predictions with quantified uncertainty and modeling of complex data. 
For many applications in science and engineering -- where accuracy, interpretability, 
and reliable uncertainty estimates are essential -- \ac{gpr} is highly attractive. 
It is particularly well suited to settings where data is expensive and error control is critical. 
However, the high computational cost of \ac{gpr} remains a major limitation. 
For $N$ training points, conventional \ac{gpr} implementations require $\mathcal{O}(N^3)$ operations and $\mathcal{O}(N^2)$ storage, which becomes prohibitive for large datasets. 
The computational cost of \ac{gpr} becomes particularly severe in high-dimensional spaces. As the dimensionality, $D$, increases, 
the data and computation requirements quickly compromise practicality -- this is the \textit{curse of dimensionality}.
\Ac{gpr} could potentially see much broader adoption if (near-)linear scaling with $N$ and low scaling with $D$
were achievable at the same time.

\begin{figure*}[t]
  \centering
  \includegraphics[width=1.0\linewidth, clip, trim=0.5cm 0cm 0.5cm 0.7cm]{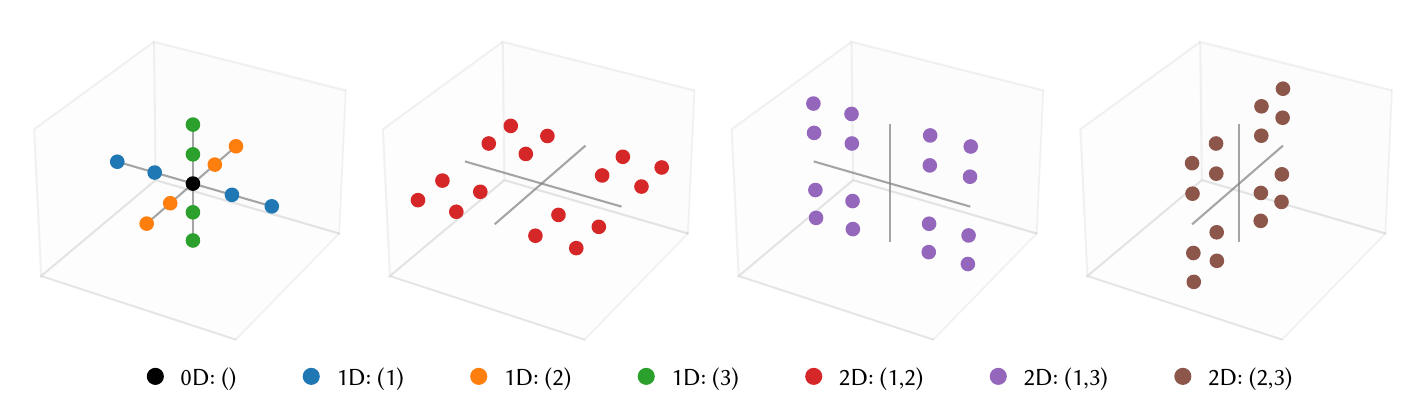}

  \caption{An incomplete grid in three dimensions with a cut level of $\alpha = 2$. 
  The grid includes the reference point (0D), 1D cuts and 2D cuts.
  A subgrid is defined by its \textit{mode combination} (MC), 
  which is simply the list of dimensions or modes that are displaced from the reference.
  The red 2D cut labelled by $(1,2)$ thus contains 
  points displaced along dimensions 1 and 2.
  An incomplete grid is defined by its \textit{mode combination range} (MCR),
  which is simply the list of MCs contained.}
  \label{fig:subgrids}
\end{figure*}

In this paper, we introduce CUTS-GPR, a new method targeting the
regime where $N$ and $D$ are simultaneously large. The method is suitable for regression
tasks where the function of interest, $f : \mathbb{R}^D \rightarrow \mathbb{R}$, 
is well-described by a low-order additive model including terms up to order $\omega$:
\begin{equation} \label{eq:low_order_model}
  f(\mathbf{x}) \approx f_0 + \sum_{i} f_{i} \big( x^{(i)} \big) + \sum_{i < j} f_{ij} \big( x^{(i)}, x^{(j)} \big) + \cdots {}
\end{equation}
Such a model is known variously as an
$n$-mode expansion \citep{carterVibrationalSelfconsistentField1997,bowmanMULTIMODECodeCalculate2003,rauhutEfficientCalculationPotential2004,kongstedAutomaticGenerationForce2006,yagiMultiresolutionPotentialEnergy2007,spartaAdaptiveDensityguidedApproach2009}, 
high-dimensional model representation (HDMR) \citep{rabitzEfficientInputOutput1999,rabitzGeneralFoundationsHighdimensional1999,liHighDimensionalModel2001,maAdaptiveHighdimensionalStochastic2010,zhangResearchHighdimensionalModel2019,zhaoDimensionalDecompositionaidedMetamodels2024}
or analysis of variance (ANOVA) decomposition \citep{stoneUsePolynomialSplines1994,huangProjectionEstimationMultiple1998,griebelDimensionwiseIntegrationHighdimensional2010} 
and has been applied successfully in diverse fields
such as chemistry \citep{carterVibrationalSelfconsistentField1997,bowmanMULTIMODECodeCalculate2003,rauhutEfficientCalculationPotential2004,kongstedAutomaticGenerationForce2006,yagiMultiresolutionPotentialEnergy2007,spartaAdaptiveDensityguidedApproach2009,rabitzEfficientInputOutput1999}, 
physics \citep{maAdaptiveHighdimensionalStochastic2010,rogersBypassingComputationalBottleneck2021}, 
materials science \citep{rabitzEfficientInputOutput1999,yangImplementationPerformanceEvaluation2021},
engineering \citep{xuDecompositionMethodsStructural2005,zhangResearchHighdimensionalModel2019,zhaoDimensionalDecompositionaidedMetamodels2024},
statistics \citep{stoneUsePolynomialSplines1994,huangProjectionEstimationMultiple1998} and 
finance \citep{wangEffectiveDimensionQuasiMonte2003,wangWhyAreHighDimensional2005,wangEffectsDimensionReduction2006,griebelDimensionwiseIntegrationHighdimensional2010}. 
Additive models have also been considered in a 
GP context \citep{duvenaudAdditiveGaussianProcesses2011, durrandeAdditiveCovarianceKernels2012, luAdditiveGaussianProcesses2022,ishidaHierarchicalAdditiveInteraction2025}, 
and it is known that the appropriate kernel is an \textit{additive kernel} that mirrors the structure of 
\eqref{eq:low_order_model}.
We propose a data structure called an \textit{incomplete grid} that naturally aligns with \eqref{eq:low_order_model}.
The incomplete grid consists of a set of low-dimensional subgrids (Figure~\ref{fig:subgrids})
that match the low-dimensional subfunctions in \eqref{eq:low_order_model}. In contrast
to existing grid-based methods \citep{saatciScalableInferenceStructured2012, wilsonFastKernelLearning2014, gilboaScalingMultidimensionalInference2015, flaxmanFastKroneckerInference2015,ishidaHierarchicalAdditiveInteraction2025}
where $N$ scales exponentially with $D$, the incomplete grid defines a 
training set that scales polynomially with $D$, thus mitigating the curse of dimensionality.
The key finding of the paper is that this natural and well-motivated combination of two structures 
(the additive kernel and the incomplete grid) can be exploited to drastically lower computational cost.
As a practical example, we apply CUTS-GPR to \acp{pes}, 
which play a central role in understanding and predicting chemical processes.
CUTS-GPR thus paves the way for truly high-dimensional
applications with additive models.
In particular, we make the following contributions:
\begin{enumerate}[label=(\roman*), nosep]
  \item We introduce a class of structured, incomplete grids suitable for high-dimensional applications. 
  The grid includes a reference point and a set of low-dimensional
  subgrids or cuts (see Figure~\ref{fig:subgrids}). In the simplest case, we include all subgrids of order $0, 1, 2, \ldots, \alpha$, 
  in which case the amount of training data scales as $\mathcal{O}(D^\alpha)$.
  \item Combining this type of grid with an additive kernel, we derive and implement a 
  kernel \ac{mvp} with a computational complexity of only
  $\mathcal{O}(n \alpha N)$ (assuming $n$ grid points per dimension). Focusing on $D$, the complexity is
  only $\mathcal{O}(D^{\alpha})$.
  This remarkable scaling is achieved by exploiting structure in the kernel matrix.
  No approximation is involved in the kernel \ac{mvp}.
  \item In a multidimensional setting, it is natural to use one hyperparameter per dimension.
  Naively, this implies a cost of at least $\mathcal{O}(D^{\alpha+1})$ for the \ac{mll} gradient.
  We use a novel factorization to remove the extra factor of $D$, preserving $\mathcal{O}(D^{\alpha})$ complexity throughout.
  \item The availability of a fast, scalable \ac{mvp}
  allows predictions and hyperparameter optimization using 
  iterative numerical techniques. 
  This enables high-dimensional, low-scaling \ac{gpr}
  tailored for applications of the additive model – a major advance
  in many areas of science and engineering.
 \end{enumerate}

\subsection{Related Work}
A substantial literature has sought to address the $O(N^3)$ computational cost of conventional GPs.
Sparse inducing-point methods, such as FITC, 
yield approximate posteriors with improved 
scalability \citep{snelsonSparseGaussianProcesses2005, titsiasVariationalLearningInducing2009, hensmanGaussianProcessesBig2013}.
In a complementary direction, structured GP models
leverage Toeplitz or Kronecker structure, enabling exact inference
at much reduced complexity \citep{saatciScalableInferenceStructured2012, wilsonFastKernelLearning2014, gilboaScalingMultidimensionalInference2015, flaxmanFastKroneckerInference2015,ishidaHierarchicalAdditiveInteraction2025}. 
Kronecker-based methods are suitable for multidimensional data, but their reliance on complete grids leads to exponential scaling with dimensionality.
The two strategies may also be combined as in KISS-GP \citep{wilsonKernelInterpolationScalable2015}
or SKIP \citep{gardnerProductKernelInterpolation2018}, the latter of which improves
scalability with $D$ as well as $N$.
A complementary branch of the GP literature addresses
high-dimensional problems through dimension reduction. 
Approaches such as 
principal component analysis \citep{jolliffePrincipalComponentAnalysis2016,zhouKernelPrincipalComponent2020},
active subspaces \citep{russiUncertaintyQuantificationExperimental2010,constantineActiveSubspaceMethods2014},
latent variables \citep{lawrenceProbabilisticNonlinearPrincipal2005,titsiasBayesianGaussianProcess2010},
deep kernels \citep{wilsonDeepKernelLearning2016}
and manifold GPs \citep{calandraManifoldGaussianProcesses2016}
thus seek to construct lower-dimensional representations of the data.
Distributed and local GPs  \citep{snelsonLocalGlobalSparse2007,gramacyLocalGaussianProcess2015,deisenrothDistributedGaussianProcesses2015,liuWhenGaussianProcess2020,ledererGaussianProcessbasedRealtime2021}
address large problems by combining predictions from smaller GP models.
Despite these advances, the regime of numerically exact \ac{gpr} for large, intrinsically high-dimensional problems 
remains an open challenge. 

\section{BACKGROUND}

\subsection{Gaussian Process Regression}
The training data is a set of $D$-dimensional inputs 
$\mathcal{X} = \{ \mathbf{x}_i\}_{i = 1}^{N}$ along with the
corresponding noisy outputs, which are collected in a vector $\mathbf{y} \in \mathbb{R}^{N}$.
The task is now to predict the unknown outputs at a set of test points
$\mathcal{X}_{*} = \{ \mathbf{x}_i^{*} \}_{i = 1}^{N_{*}}$,
and for this purpose we need the training--training, training--test and test--test
covariance matrices, which are defined in terms of the kernel function, $k$:
\begin{alignat}{4}
  \mathbf{K}      = k(\mathcal{X}    , \mathcal{X}),\quad
  \mathbf{K}_{*}  = k(\mathcal{X}    , \mathcal{X}_{*}),\quad
  \mathbf{K}_{**} = k(\mathcal{X}_{*}, \mathcal{X}_{*}). \nonumber
\end{alignat}
We also introduce the noisy training--training covariance matrix, 
$\mathbf{C}  = \mathbf{K} + \sigma^2 \mathbf{I}$, and
the weights, $\weights = \mathbf{C}^{-1} \mathbf{y}$.
With these definitions in place, we can compute the predictive 
mean and covariance as 
\begin{equation}
  \bm{\mu} = \mathbf{K}_{*}^{\trans} \mathbf{C}^{-1} \mathbf{y} = \mathbf{K}_{*}^{\trans} \weights,\quad
  \mathbf{\Sigma} = \mathbf{K}_{**} -  \mathbf{K}_{*}^{\trans} \mathbf{C}^{-1} \mathbf{K}_{*}. \label{eq:predictive_mean_variance}
\end{equation}
Any hyperparameter, $\theta$, can be determined by
maximizing the \ac{mll},
\begin{align} \label{eq:mll}
  \mathcal{L} 
  = - \frac{1}{2} \left( \mathbf{y}^{\trans} \weights + \log{|\mathbf{C}|} + N \log{2 \pi} \right).
\end{align}
For optimization using a gradient-based algorithm, we also require the derivatives:
\begin{align} \label{eq:mll_gradient}
  \pdv{\mathcal{L}}{\theta} 
  = \frac{1}{2} \weights^{\trans} \pdv{\mathbf{C}}{\theta} \weights
  - \frac{1}{2} \trace\left( \mathbf{C}^{-1} \pdv{\mathbf{C}}{\theta} \right). 
\end{align}
\Ac{gpr} thus requires the following ingredients:
\begin{enumerate}[label=(\roman*), nosep]
 \item Linear solves with $\mathbf{C}$, for example $\weights = \mathbf{C}^{-1} \mathbf{y}$;
 \item The log-determinant, $\log{|\mathbf{C}|} = \trace \log(\mathbf{C})$;
 \item The trace terms, $\trace ( \mathbf{C}^{-1} \partial \mathbf{C} / \partial \theta )$;
 \item The quadratic terms, $\weights^{\trans} (\partial \mathbf{C} / \partial \theta) \weights$.
\end{enumerate}
Conventional implementations use the Cholesky decomposition of $\mathbf{C}$,
which takes $\mathcal{O}(N^3)$ time to compute and $\mathcal{O}(N^2)$ space to store. 
However, $\mathbf{C}$ often allows a fast \ac{mvp},
which is sufficient for numerically exact predictions
and accurate stochastic estimates of the \ac{mll} and the \ac{mll} gradient
\citep{gardnerGPyTorchBlackboxMatrixMatrix2018, wengerPreconditioningScalableGaussian2022}.
Appendix~\ref{appendix:gpr_mvp} gives a short account of \ac{mvp}-based \ac{gpr},
including the use of preconditioning.
In the following, we focus on developing a fast, low-scaling kernel \ac{mvp}.

\subsection{Complete Grids} \label{sec:complete_grids}

Throughout the paper, we are concerned with $D$-dimensional inputs $\mathbf{x}$ that are located on a grid.
We will start by defining a \textit{complete} grid, by which we mean
a Cartesian product grid like
\begin{equation} \label{eq:complete_grid_def}
  \mathcal{X} = \mathcal{X}^{(1)} \times \cdots \times \mathcal{X}^{(D)} \subset \mathbb{R}^D. 
\end{equation}
Each $\mathcal{X}^{(d)}$ is a 1D grid (not necessarily regular) of size $n_d$, and the full grid
has size $N = |\mathcal{X}| = n_1 n_2 \cdots n_D$. It is natural in this setting to index
each 1D grid by an index set $\mathcal{I}^{(d)} = \{0, 1, \ldots, n_d - 1 \}$ and each multidimensional grid point
by a multi-index 
\begin{align}
\mathbf{i} = (i_1, \ldots, i_D) \in \mathcal{I} = \mathcal{I}^{(1)} \times \cdots \times \mathcal{I}^{(D)}.
\end{align}
It is well known
that the combination of a complete grid and a separable kernel,
\begin{align} \label{eq:separable_kernel}
    k(\mathbf{x}, \tilde{\mathbf{x}}) = \prod_{m=1}^{D} k^{(m)} \big( x^{(m)}, \tilde{x}^{(m)} \big),
\end{align}
enables \ac{gpr} with near-linear complexity \citep{saatciScalableInferenceStructured2012, wilsonFastKernelLearning2014, gilboaScalingMultidimensionalInference2015, flaxmanFastKroneckerInference2015}.
This is due to the fact that the kernel matrix obtained by evaluating \eqref{eq:separable_kernel} over a complete grid
is a Kronecker product. Specifically, one finds that
\begin{align} 
    \mathbf{K} = \mathbf{K}^{(1)} \otimes \mathbf{K}^{(2)} \otimes \cdots \otimes \mathbf{K}^{(D)},
    \quad \mathbf{K}^{(m)} \in \mathbb{R}^{n_m \times n_m}, \nonumber
\end{align}
where the matrices $\mathbf{K}^{(m)}$ are called base matrices.
Using the properties of Kronecker products
one can, for instance, compute inverse kernel \acp{mvp} with a computational cost of only
$\mathcal{O}\big( (n_1 + \cdots + n_D) N \big) = \mathcal{O}(nDN)$, 
assuming 1D grids of equal size $n$ \citep{saatciScalableInferenceStructured2012}.
In spite of the attractive computational complexity, there
are two major downsides to this approach: (1) The separable kernel is often a poor choice in high dimension; and (2)
the $N$ required to preserve modelling accuracy scales exponentially with the number of dimensions, thus limiting applications to $D < 10$ or so (depending on $n$).

\subsection{Additive Kernels} \label{sec:additive_kernels}

A more attractive kernel format for high-dimensional applications is the \textit{additive} kernel \citep{duvenaudAdditiveGaussianProcesses2011, durrandeAdditiveCovarianceKernels2012, luAdditiveGaussianProcesses2022}, 
which is typically defined by a maximum interaction order, $\omega$:
\begin{multline} \label{eq:max_order_kernel_def}
  k =
    \sigma_0^2 + \sigma_1^2 \sum_{m} k^{(m)} + \sigma_2^2 \sum_{\mathclap{m < m'}} k^{(m)} k^{(m')} \\
    {} + \cdots + 
    \sigma_{\!\omega}^2 \sum_{\mathclap{m_1 < m_2 < \cdots < m_{\omega}}}  k^{(m_1)} k^{(m_2)} \cdots k^{(m_\omega)}.
\end{multline}
The hyperparameters $\sigma_{k}^2$ are called order variances and control
the relative importance of different interaction orders.
In addition, the 1D kernels (base kernels) may each depend on one or more hyperparameters.
For definiteness, we will assume that the 1D kernels 
each depend on a \textit{length scale}, $\ell^{(m)}$.
If $\omega = D$, the kernel is \textit{saturated} and includes all possible interactions.
Such a kernel is highly descriptive, but it is usually not possible or necessary 
to include high-order terms, as shown by \citet{luAdditiveGaussianProcesses2022}
who introduced the \ac{oak}.
Compared to a generic additive kernel, the \ac{oak} approach includes an 
extra step where the 1D kernels are orthogonalized or centered \citep{durrandeAdditiveCovarianceKernels2012},
which ensures an identifiable model \citep{luAdditiveGaussianProcesses2022}.

\citet{ishidaHierarchicalAdditiveInteraction2025} introduce
a more general and flexible class of additive kernels, which we write as
\begin{align} \label{eq:sop_kernel_def}
  k(\mathbf{x}, \tilde{\mathbf{x}}) 
    &= \sum_{\mathbf{m} \in \kernelmcr} \sigma_{|\mathbf{m}|}^2 \prod_{m \in \mathbf{m}} k^{(m)} \big( x^{(m)}, \tilde{x}^{(m)} \big)
\end{align}
where $\kernelmcr$ lists the interaction terms included in the kernel.
In addition, they argue that $\kernelmcr$
should be \textit{hierarchical}, but not necessarily of the simple form implied by \eqref{eq:max_order_kernel_def}. 
Comparing with Section~\ref{sec:complete_grids}, it is clear that \eqref{eq:sop_kernel_def} evaluates to a sum of Kronecker products 
when using a complete grid, but it is not at all obvious that this kind of structure can be exploited. 
However, \citet{ishidaHierarchicalAdditiveInteraction2025} prove the surprising fact that \eqref{eq:sop_kernel_complete_grid} 
can be diagonalized in closed form (assuming centered base kernels).
Their approach allows GPR with near-linear scaling and a flexible additive kernel for complete-grid datasets.
However, the exponential scaling with $D$ remains as a serious limitation.

\section{INCOMPLETE GRIDS} \label{sec:incomplete_grids}

To address the curse of dimensionality in a way that aligns naturally with the low-order assumption \eqref{eq:low_order_model},
we need a version of grid-based GPR that allows 
sampling only a subset of the full grid, referred to as an \textit{incomplete} grid. We 
consider $\widehat{N}$ inputs located on an incomplete grid $\hatcal{X} \subseteq \mathcal{X}$ indexed by the multi-indices $\hatcal{I} \subseteq \mathcal{I}$.
We will start by introducing a specific type of incomplete grid and later describe how it
relates to a more general class of grids.
The incomplete grid used in this paper is defined with respect to
a \textit{reference point} that we label by the zero multi-index, $\mathbf{i} = (0, 0, \ldots, 0)$.
The reference point can be chosen freely according to the problem at hand.
In addition to the reference, the grid includes a set of low-dimensional \textit{subgrids} or \textit{cuts}.
\begin{exmp} \label{example:subgrids}
  Consider an incomplete grid including the reference, all 1D subgrids and all 2D subgrids.
  With $D = 3$ this construction includes the points corresponding to
  the multi-indices
  \begin{subequations} \label{eq:indices_MCR_dim3_cut2}
    \begin{alignat}{3}
      \text{0D:} \quad &(0,0,0)  \nonumber \\
      \text{1D:} \quad &(a,0,0), \;\; &&(0,b,0), \;\; &&(0,0,c)  \nonumber \\
      \text{2D:} \quad &(a,b,0), \;\; &&(a,0,c), \;\; &&(0,b,c)  \nonumber 
    \end{alignat}
  \end{subequations}
  for all $0 < a < n_1$, $0 < b < n_2$ and $0 < c < n_3$ (see also Figure~\ref{fig:subgrids}).
\end{exmp}
Index sets like Example~\ref{example:subgrids} occur frequently in the tensor literature \citep[see e.g.][]{koldaTensorDecompositionsApplications2009}
and in vibrational structure theory \citep{christiansenSecondQuantizationFormulation2004}.
In both of these fields,
the pertinent dimensions are called \textit{modes}, which is a terminology that we
will use in the following.
Given a set of 1D grids, a subgrid is uniquely defined by its \ac{mc}, 
which is simply the list of modes with non-zero indices.
An incomplete grid is in turn defined by its \ac{mcr},
which is simply the list of \acp{mc}. Example~\ref{example:subgrids} thus corresponds to the \ac{mcr}
\begin{align}
  \gridmcr = \{ (),\; (1),\; (2),\; (3),\; (1,2),\; (1,3),\; (2,3) \}. \nonumber
\end{align}
A formal definition of subgrids and incomplete grids is given in Appendix~\ref{appendix:incomplete_grid_formal_def}.
In the following, we exclusively consider incomplete grids $\hatcal{I}$ 
whose \ac{mcr} is \textit{hierarchical} or \ac{cuts}:
\begin{definition} \label{def:cuts}
  The set $\mathcal{M}$ is \acf{cuts} if
  $\mathbf{m}' \subset \mathbf{m} \in \mathcal{M}$ implies $\mathbf{m}' \in \mathcal{M}$.
\end{definition}
%
The \ac{mcr} framework is flexible and can be used to put more emphasis
on certain dimensions and less on others. However, it is often useful
to define a \textit{simple} \ac{mcr} that includes all cuts of order $0, 1, 2, \ldots, \alpha$.
This construction is obviously \ac{cuts} and allows definite computational 
complexities to be derived and stated concisely.
Assuming for simplicity that $n_1 = \cdots = n_d = n$, we find that an incomplete 
grid with maximum cut level $\alpha$ has size
\begin{align}
  \widehat{N} 
  =    \sum_{k = 0}^{\alpha} \binom{D}{k} (n - 1)^k 
  \sim \mathcal{O} \! \left( \frac{D^\alpha n^{\alpha}}{\alpha!} \right).
\end{align}
Importantly, we see that $\widehat{N}$ grows polynomially with $D$, 
in contrast to the exponential growth of the complete grid. Thus, with $\alpha < D$, 
this data format allows us to mitigate the curse of dimensionality.
 
A cut-based grid 
with an \ac{mcr} that it \ac{cuts} 
is just one example of a grid that is \ac{cud}:
\begin{definition} \label{def:cud}
  A set of multi-indices $\hatcal{I}$ is \acf{cud} if
  \begin{multline}
    \mspace{10mu} (i_1, \ldots, i_d, \ldots, i_D) \in \hatcal{I}, \; i_d > 0 \\
      \Rightarrow\; (i_1, \ldots, i_d - 1, \ldots, i_D) \in \hatcal{I}. \mspace{10mu} \nonumber
  \end{multline}
  $\hatcal{I}$ is also said to be downward closed.
\end{definition}
It turns out that being \ac{cud} is the crucial property
that enables a fast \ac{mvp}.
Moreover, the notion of \ac{cud} can be generalized compared to Definition~\ref{def:cud}
\citep{holzmullerFastSparseGrid2021} (see also Appendix~\ref{appendix:general_cud_framework}), which opens 
the possibility of scalable \ac{gpr} with other types of incomplete grids.

\section{COMBINING INCOMPLETE GRIDS AND ADDITIVE KERNELS} \label{sec:incomplete_grids_and_sop_kernels}

We now combine the additive kernel with the incomplete grid from Section~\ref{sec:incomplete_grids}. To that end, we introduce the following convenient notation:
\begin{definition} \label{def:kronecker_product}
Consider the square matrices $\mathbf{A}^{\!(m)} \in \mathbb{R}^{n_m \times n_m}$ for $m = 1, \ldots, D$ and let $N = n_1 n_2 \cdots n_D$.
\begin{enumerate}[label=(\roman*)]
  \item We define the Kronecker product
  \begin{align}
    \mathbf{A} = \mathbf{A}^{\!(1)} \otimes \cdots \otimes \mathbf{A}^{\!(D)} \in \mathbb{R}^{N \times N} \nonumber
  \end{align}
  via $A_{\mathbf{i} \mathbf{j}} = A^{(1)}_{i_1 j_1} A^{(2)}_{i_2 j_2} \cdots A^{(D)}_{i_D j_D}$
  for the multi-indices $\mathbf{i} = (i_1, i_2, \ldots, i_D)$ and $\mathbf{j} = (j_1, j_2, \ldots, j_D)$.
  \item A one-mode bracket matrix is defined as
  \begin{align} \label{eq:one_mode_bracket_matrix}
    \mathbf{A}^{\![m]} = \mathbf{I}^{(1)} \otimes \cdots \otimes \mathbf{A}^{\!(m)} \otimes \cdots \otimes \mathbf{I}^{(D)} \nonumber 
  \end{align}
  with identity matrices in all positions except one.
  \item A many-mode bracket matrix is defined as
  \begin{equation}
    \mathbf{A}^{\![\mathbf{m}]} = \prod_{m \in \mathbf{m}} \mathbf{A}^{\![m]}.
  \end{equation}
\end{enumerate}
\end{definition}
Introducing the all-ones matrices, $\mathbf{J}^{(m)} \in \mathbb{R}^{n_m \times n_m}$, the complete-grid additive kernel matrix can be written as
\begin{align} \label{eq:sop_kernel_complete_grid}
  \mathbf{K}
    = \sum_{\mathbf{m} \in \kernelmcr} \sigma_{|\mathbf{m}|}^2 \mathbf{K}^{[\mathbf{m}]} \mathbf{J}^{[\neg \mathbf{m}]}.
\end{align}
The notation $\neg \mathbf{m}$ means the set of modes $m \notin \mathbf{m}$ (the complement of $\mathbf{m}$).
The kernel matrix over the incomplete grid, $\hatbold{K} \in \mathbb{R}^{\widehat{N} \times \widehat{N}}$, can be formally obtained
from the complete-grid matrix, $\mathbf{K} \in \mathbb{R}^{N \times N}$, by appropriately deleting rows and columns. This can
be written as
\begin{align} \label{eq:sop_kernel_complete_grid_chopped_simple}
  \hatbold{K} 
  = \mathbf{\Gamma}^{\trans} \mathbf{K} \mathbf{\Gamma}
  = \sum_{\mathbf{m} \in \kernelmcr} 
  \sigma_{|\mathbf{m}|}^2 \mathbf{\Gamma}^{\trans} \mathbf{K}^{[\mathbf{m}]} \mathbf{J}^{[\neg \mathbf{m}]} \mathbf{\Gamma},
\end{align}
where $\mathbf{\Gamma} \in \mathbb{R}^{N \times \widehat{N}}$ is a \textit{chopping} matrix (the columns of $\mathbf{\Gamma}$
are the standard basis vectors $\mathbf{e}_\mathbf{i} \in \mathbb{R}^{N}$ with $\mathbf{i} \in \hatcal{I}$).
Equation~\eqref{eq:sop_kernel_complete_grid_chopped_simple} is thus a sum of chopped Kronecker products.
It would seem that chopping destroys the structure of the complete-grid kernel,
but we will see that chopping can in fact be handled 
efficiently and very elegantly if $\hatcal{I}$ is \ac{cud}.
Chopped Kronecker products have been used extensively for solving the vibrational 
Schrödinger equation \citep{wodraszkaPrunedCollocationbasedMulticonfiguration2019, zakUsingCollocationHierarchical2019, simmonsComputingVibrationalSpectra2023},
and the mathematical aspects have been formalized by
Holzmüller and Pflüger \citep{holzmullerFastSparseGrid2021} to whom we refer for details and additional references.
We will use the following theorem:
\begin{thm} \label{thm:chopping_main_text}
  Let the index set $\hatcal{I}$ be \ac{cud} and let $\mathbf{A}, \mathbf{B} \in \mathbb{R}^{N \times N}$.
  In addition, let $\mathbf{M}^{[m]} \in \mathbb{R}^{N \times N}$
  be an arbitrary one-mode bracket matrix (see Definition \ref{def:kronecker_product}).
  Using hats to denote chopping, e.g. $\hatbold{A} = \mathbf{\Gamma}^{\trans} \mathbf{A} \mathbf{\Gamma}$,
  we then have
  \begin{enumerate}[label=(\roman*)]
  \item  If $\mathbf{A}$ is lower triangular or $\mathbf{B}$ is upper triangular, 
  then $\widehat{\mathbf{A} \mathbf{B}} = \hatbold{A} \hatbold{B}$. \label{thm:chopping_main_text_i}
  \item If $\mathbf{A}$ is lower (upper) triangular and invertible, then $\mathbf{A}^{-1}$ is lower (upper) triangular
  and $\hatbold{A}$ is invertible with $\hatbold{A}^{-1} = \widehat{\mathbf{A}^{-1}}$. \label{thm:chopping_main_text_inv}
  \item The matrix-vector product $\hatboldsuper{M}{[m]} \hatbold{v} = \hatbold{w}$
  can be computed with a complexity of at most $\mathcal{O} ( n_m  \widehat{N} )$
  without referencing the multi-indices $\mathbf{i} \notin \hatcal{I}$. \label{thm:chopping_main_text_ii}
  \end{enumerate}
\end{thm}
\begin{proof}
  See \citet[Theorems 1 and 3]{holzmullerFastSparseGrid2021}.
\end{proof}

Theorem~\ref{thm:chopping_main_text}\ref{thm:chopping_main_text_inv} actually
allows a separable kernel to be inverted directly over an incomplete grid,
but only in the noiseless case (see Appendix~\ref{appendix:chopped_kronecker_products}). 
Our initial investigations
in this direction were not promising: 
(i)   The complete lack of noise leads to very ill-conditioned numerics;
(ii)  the $D$-complexity is higher than CUTS-GPR; and
(iii) an additive kernel aligns more closely with the additive model,
which is our focus in this paper. Based on these observations, 
we therefore abandoned the separable kernel
in favor of the CUTS-GPR setup.

The details of how to compute the \ac{mvp} in Theorem~\ref{thm:chopping_main_text}\ref{thm:chopping_main_text_ii}
are covered by Theorem~\ref{thm:chopped_one_mode_contraction} and Appendix~\ref{appendix:contractions_mcr_tensor}.
The crucial aspect of
Theorem~\ref{thm:chopping_main_text}\ref{thm:chopping_main_text_ii}
is that it only ever references the multi-indices of the incomplete grid.
The remaining multi-indices of the underlying complete grid, $\mathbf{i} \notin \hatcal{I}$, are never referenced, not even
in intermediate steps.
Staying on the incomplete grid at all times is a \textit{necessary} condition for breaking the curse of dimensionality, but
the known complexity offered by Theorem~\ref{thm:chopping_main_text}\ref{thm:chopping_main_text_ii} is not \textit{sufficient} in itself
due the large number of terms in the additive kernel.
Our main contribution is to resolve this problem
by careful using structure 
and sparsity as described in Section~\ref{sec:fast_mvps}.

\section{LOW-SCALING IMPLEMENTATION}

\subsection{Fast Matrix-Vector Products} \label{sec:fast_mvps}

Our task is now to apply the chopping framework to the 
kernel matrix in \eqref{eq:sop_kernel_complete_grid}.
Before proceeding, we observe that the all-ones matrix
can be factorized like
\begin{align} \label{eq:J_decomposition}
  \mathbf{J}^{(m)} =
  \mathbf{L}^{(m)} \mathbf{R}^{(m)} \mathbf{U}^{(m)}
\end{align}
with $\mathbf{U}^{(m)} = (\mathbf{L}^{(m)})^{\trans}$ and
\begin{align} 
  \mathbf{L}^{(m)} =
  \begin{bmatrix} 
    1      & 0      & \cdots & 0     \\
    1      & 1      & \cdots & 0     \\
    \vdots & \vdots & \ddots & \vdots\\
    1      & 0      & \cdots & 1
  \end{bmatrix}, \quad
  \mathbf{R}^{(m)} =
  \begin{bmatrix} 
    1      & 0      & \cdots & 0     \\
    0      & 0      & \cdots & 0     \\
    \vdots & \vdots & \ddots & \vdots\\
    0      & 0      & \cdots & 0
  \end{bmatrix}. \nonumber
\end{align}
The matrix $\mathbf{R}^{(m)}$ projects onto the reference point for mode $m$.
Using \eqref{eq:J_decomposition}, a short derivation (see Appendix~\ref{appendix:deriving_kernel_mvp})
shows that \eqref{eq:sop_kernel_complete_grid} can be rewritten as
\begin{gather}
  \mathbf{K}
  = 
  \mathbf{L}
  \left[ 
    \sum_{\mathbf{m}} \sigma_{|\mathbf{m}|}^2 
    \mathbf{M}^{[\mathbf{m}]}
    \mathbf{R}^{[\neg \mathbf{m}]} 
  \right]
  \mathbf{U}
  \equiv \mathbf{L} \mathbf{M} \mathbf{U}. \label{eq:sop_kernel_complete_grid_rewritten}
\end{gather}
with the definitions $\mathbf{M}^{(m)} = \big(\mathbf{L}^{(m)} \big)^{-1} \mathbf{K}^{(m)} \big(\mathbf{U}^{(m)} \big)^{-1}$, 
$\mathbf{L} = \bigotimes_{m=1}^{D} \mathbf{L}^{(m)}$ and $\mathbf{U} = \bigotimes_{m=1}^{D} \mathbf{U}^{(m)}$.
Each term in \eqref{eq:sop_kernel_complete_grid_rewritten}
has a lower triangular factor to the left ($\mathbf{L}$),
a non-triangular factor in the middle ($\mathbf{M}^{[\mathbf{m}]}$)
and two upper-triangular factors to the right ($\mathbf{R}^{[\neg \mathbf{m}]}$ and $\mathbf{U}$).
We can therefore apply 
Theorem~\ref{thm:chopping_main_text}\ref{thm:chopping_main_text_i}
thrice to get
\begin{align}
  \hatbold{K} 
  = \hatbold{L}
  \left[ 
    \sum_{\mathbf{m}} \sigma_{|\mathbf{m}|}^2 
    \widehat{\mathbf{M}^{[\mathbf{m}]}}
    \widehat{\mathbf{R}^{[\neg \mathbf{m}]}}
  \right]
  \hatbold{U}
  = \hatbold{L} \hatbold{M} \hatbold{U}. \label{eq:sop_kernel_complete_grid_chopped}
\end{align}
Using the factorized form above, we can rephrase the \ac{mvp} with $\hatbold{K}$ as a sequence of three simpler
\acp{mvp}. The matrices 
$\hatbold{L} = \hatboldsuper{L}{[1]} \hatboldsuper{L}{{[2]}} \cdots \hatboldsuper{L}{{[D]}}$ and 
$\hatbold{U} = \hatboldsuper{U}{[1]} \hatboldsuper{U}{{[2]}} \cdots \hatboldsuper{U}{{[D]}}$
are evidently very sparse in addition to being highly structured.
We show that each factor touches only a small part of the vector,
with the result that \acp{mvp} by $\hatbold{L}$ and $\hatbold{U}$
have a complexity of only $\mathcal{O}(\alpha \widehat{N})$ (see Appendix~\ref{appendix:complexity_Lv_Uv}).
This should be compared to the naive complexity of $\mathcal{O}(n D \widehat{N})$ obtained by repeatedly
applying Theorem~\ref{thm:chopping_main_text}\ref{thm:chopping_main_text_ii} without utilizing the special structure of 
$\mathbf{L}^{(m)}$ and $\mathbf{U}^{(m)}$.

The \acp{mvp} with $\hatbold{M}$ is much more challenging, 
simply because the summation contains $\mathcal{O}(D^{\alpha})$ terms.
Applying Theorem~\ref{thm:chopping_main_text}\ref{thm:chopping_main_text_ii} and adding the terms
would lead to a complexity of $\mathcal{O}(n D^{\alpha + 1} \widehat{N})$, which is not competitive.
However, we demonstrate in Appendix~\ref{appendix:complexity_Mv} that 
the special structure of $\hatbold{M}$ can be exploited to reduce the complexity
to only $\mathcal{O}(n \alpha \widehat{N})$. For large $D$, this is a massive improvement.
In conclusion, we have obtained a factorization of $\hatbold{K}$ 
into three factors that are highly sparse \textit{and} structured. 
After careful analysis, this in turn allows an extremely fast \ac{mvp}.

\subsection{The Quadratic Term and Overall Cost} \label{sec:quadratic_term}
Computing the quadratic term of \eqref{eq:mll_gradient}, $\weights^{\trans} (\partial \hatbold{C} / \partial \theta) \weights$,
is not a difficult task on its own, and the cost for each individual hyperparameter is usually not problematic.
Consider, however, the task of computing \textit{all} length scale quadratic terms.
A typical implementation consists of a derivative \ac{mvp},
$(\partial \hatbold{C} / \partial \ell^{(m)}) \weights$, followed by a dot product with $\weights$.
Even neglecting the cost of the \ac{mvp}, this leads to a complexity of $\mathcal{O}(D\widehat{N})$
due to the dot products alone. For a simple \ac{mcr}, this implies
a $D$-complexity of $\mathcal{O}(D^{\alpha + 1})$ compared to $\mathcal{O}(D^{\alpha})$ in the rest of the code.
Appendix~\ref{appendix:quadratic_term} shows how to maintain a $D$-complexity of $\mathcal{O}(D^{\alpha})$
by leveraging the factorized form of $\hatbold{K}$ in
\eqref{eq:sop_kernel_complete_grid_chopped}.

Appendix~\ref{appendix:overal_structure} gives an overview of the code and the cost of the various steps.
The most expensive parts are the hyperparameter optimization and predictive variances
with complexities
$\mathcal{O}\big(N_{\mathrm{opt}} N_{\mathrm{probe}} N_{\mathrm{CG}} (n \alpha \widehat{N} + k \widehat{N} ) \big)$
and
$\mathcal{O}\big( \widehat{N}_{*} N_{\mathrm{CG}} (n \alpha \widehat{N} + k \widehat{N} ) \big)$,
respectively. The symbols $N_{\mathrm{opt}}$, $N_{\mathrm{probe}}$, $N_{\mathrm{CG}}$
and $\widehat{N}_{*}$
denote the number of optimization cycles, probe vectors (for trace estimation),
\ac{cg} iterations and test points, respectively, while $k$ is the preconditioner rank.

\section{NUMERICAL EXPERIMENTS} \label{sec:experimental_results}

\subsection{Complexity of the Kernel MVP} \label{sec:computational_complexity}

The kernel \ac{mvp} has a theoretical complexity of only 
$\mathcal{O}(n \alpha \widehat{N})$.
If $n$ and $\alpha$ are kept constant
this translates to $\mathcal{O}(\widehat{N})$ or, equivalently, $\mathcal{O}(D^\alpha)$.
It is important to verify this complexity and to assess the performance of the \ac{mvp} in absolute terms.
The $D$-complexity is determined 
for $\alpha = 2,3,4$ and $n = 5,10,20$ by running a series of 
simple benchmarks on a single core (see Appendix~\ref{appendix:mvp_scaling} for details and raw timings).
The number of dimensions, $D$, is varied in the range from $16$ to at most $16384$, depending on $\alpha$ and $n$.
For each combination of $\alpha$ and $n$, the largest benchmark has $\widehat{N}$
equal to a few billion and a CPU time of no more than three minutes.
CPU times have been fitted to a power law to determine the empirical $D$-exponents.
The fits are plotted for
$\alpha = 2,4$ in Figure~\ref{fig:gridgpr_benchmark_alpha24} and for $\alpha = 3$ in Figure~\ref{fig:gridgpr_benchmark_alpha3}.
The empirical exponents [2.06--2.08 ($\alpha = 2$),
3.09--3.15 ($\alpha = 3$) and 
4.21--4.29 ($\alpha = 4$)] are slightly 
larger than theoretical, asymptotic exponents of 2, 3 and 4.
As discussed in Appendix~\ref{appendix:mvp_scaling},
this is simply a sign that the asymptotic region has not been fully reached, especially for $\alpha = 4$.
Figures~\ref{fig:gridgpr_scaling_alpha2_n5}--\ref{fig:gridgpr_scaling_alpha4_n20} show the fitted
$\widehat{N}$-exponents which are all in the range 1.00--1.04, confirming that the \ac{mvp} scales linearly 
when $n$ and $\alpha$ are kept constant.


\subsection{Exact Baseline Comparison} \label{sec:exact_baseline}
Next, we verify the correctness of the full CUTS-GPR code by comparing with an exact baseline
for a set of small examples ranging from $D = 3$, $\widehat{N} = 127$ to $D = 13, \widehat{N} = 2887$ 
(all examples use $\omega = \alpha = 2$).
Details and results are given in Appendix~\ref{appendix:exact_baselines}. In every case, we find that CUTS-GPR
reproduces the exact baseline with an empirical end-to-end scaling of $\mathcal{O}(D^{3.22})$
and $\mathcal{O}(\widehat{N}^{1.51})$.
This scaling is somewhat larger than the raw MVP-scaling of $\mathcal{O}(D^{2})$ and $\mathcal{O}(\widehat{N})$,
indicating that the number of \ac{cg} iterations tends to grow with system size. Nevertheless, the
CUTS-GPR scaling is much lower than the exact baseline scaling. In addition, we find that CUTS-GPR is very
fast in absolute term, beating the baseline already at the smallest example.

\begin{figure*}[ht]
  {
  \phantomsubcaption\label{fig:gridgpr_benchmark_alpha24}
  \phantomsubcaption\label{fig:end_to_end_scaling_N}
  \phantomsubcaption\label{fig:gradient_norms}
  \phantomsubcaption\label{fig:one_mode_cut}
  }
  \centering
  \includegraphics[width=1\textwidth]{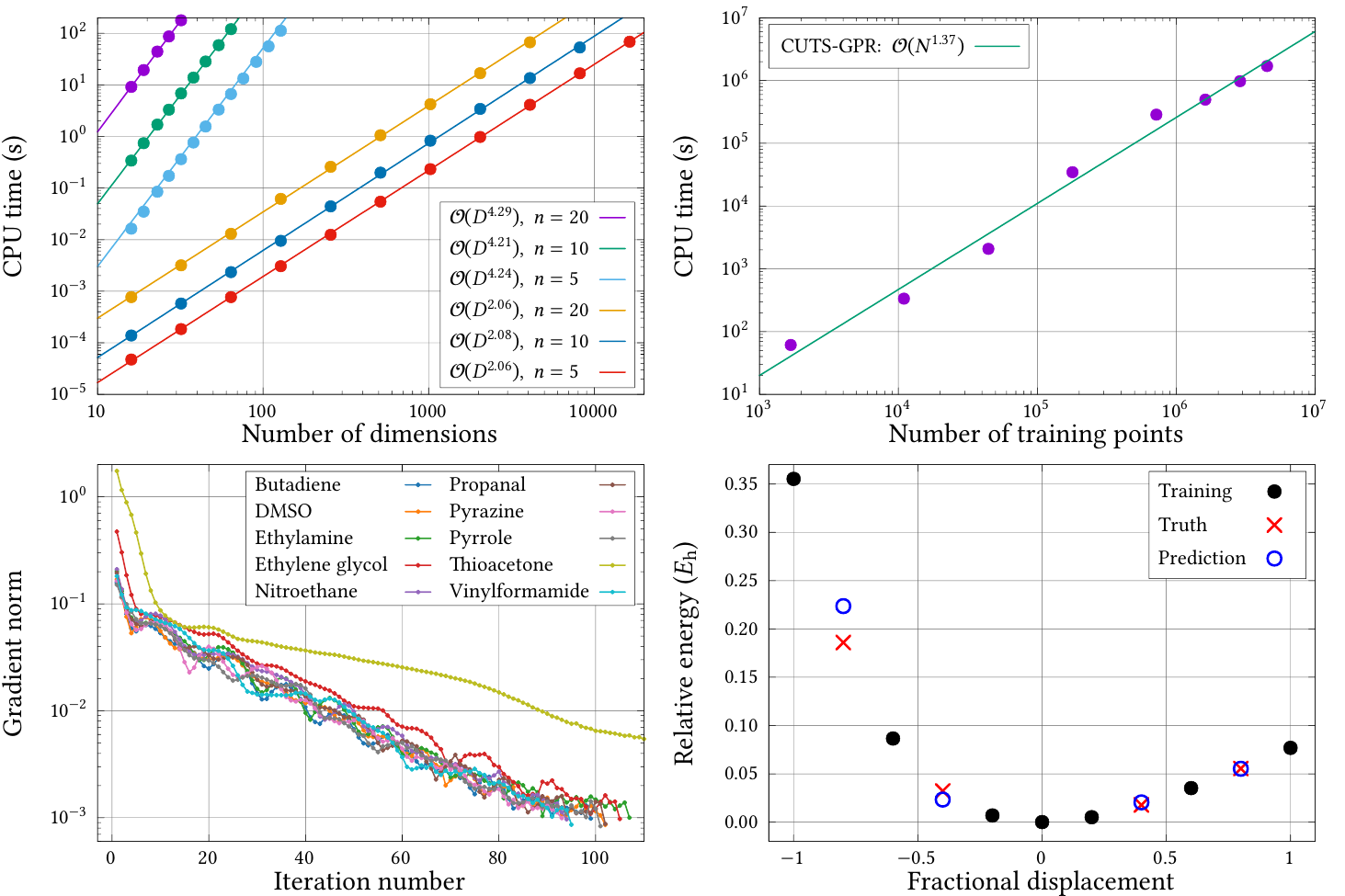}
  \caption{%
  Top left  (a): $D$-scaling of the kernel \ac{mvp} for $\alpha = 2, 4$ and $n = 5,10,20$. 
  Top right (b): Empirical $\widehat{N}$-scaling of CUTS-GPR for a class of synthetic datasets.
  Middle left  (c): Learning curves for all ten molecules. 
  Middle right (d): Training data, predictions and reference data for the 24th 1D cut of the thioacetone PES.}
  \label{fig:4_by_4}
\end{figure*}

\subsection{End-to-end Scaling in High Dimension} \label{sec:end_to_end_scaling}
In order to study the end-to-end scaling under more relevant conditions, we apply CUTS-GPR to
the synthetic datasets described in Appendix~\ref{appendix:end_to_end_scaling} (all datasets use $\omega = \alpha = 2$).
We use a rank-10 preconditioner, 35 probe vectors, fixed noise at $\sigma^2 = 10^{-3}$, 
100 cycles of hyperparameter optimization using the Adam algorithm \citep{kingmaAdamMethodStochastic2017},
and a relative \ac{cg} threshold of 
$\varepsilon_{\mathrm{CG}} = 10^{-3}$
(additional details are given in Appendix~\ref{appendix:end_to_end_scaling}). 
The largest calculation has $D = 500$, $\widehat{N} = \num{4494001}$, $\widehat{N}_{*} = \num{7986000}$
and finishes in just 15 hours using 36 cores. 
In itself, this is remarkable for a numerically exact
GPR calculation, but it is perhaps more interesting to examine the underlying scaling. 
Figure~\ref{fig:end_to_end_scaling_N} shows that
the total CPU time (including hyperparameter optimization and computation of the predictive mean) 
scales as $\mathcal{O}(\widehat{N}^{1.37})$,
which is much lower than the cubic scaling of conventional \ac{gpr}.
Figure~\ref{fig:end_to_end_scaling_D_appendix} shows an empirical $D$-scaling of $\mathcal{O}(D^{2.76})$,
which should be compared to the raw MVP-scaling of $\mathcal{O}(D^{2})$.

\subsection{Application to PES data} \label{sec:application_to_pes}

\Acp{pes} represent a prototypical case where low-order
models offer an attractive balance between effort and accuracy.
Appendix~\ref{appendix:truncating_kernel_and_grid} gives a small, concrete example of this balance.

As a realistic and challenging application of CUTS-GPR,
we consider the 24-dimensional \acp{pes} of ten organic molecules
using $\omega = \alpha = 3$ (full details are given in Appendix~\ref{appendix:application_to_pes}).
Each dataset includes $\widehat{N} = \num{447 265}$ training points and
$\widehat{N}_* = \num{1 604 576}$ test points.
The computational setup is similar to Section~\ref{sec:end_to_end_scaling},
but this time we converge the hyperparameter gradient norm to a threshold of $\varepsilon_{\mathrm{opt}} = 10^{-3}$.

Few \ac{gpr} methods (exact or approximate) can handle the combination of an additive kernel
with large $N$ and large $D$. We were therefore only able to compare against
SVGP \citep{hensmanGaussianProcessesBig2013} 
as implemented in GPyTorch \citep{gardnerGPyTorchBlackboxMatrixMatrix2018} 
with the \ac{oak} available in BoTorch \citep{balandatBoTorchFrameworkEfficient2020}.
The current BoTorch implementation only supports $\omega = 2$, 
so we used a locally modified version to also allow $\omega = 3$.
Settings are described in Appendix~\ref{appendix:svgp_settings}.
All CUTS-GPR and SVGP calculations were run on 36 cores. 

First, we study the convergence of the hyperparameter optimization in CUTS-GPR. Figure~\ref{fig:gradient_norms}
shows the learning curves for all ten molecules.
The norm declines rather quickly during the first 5--10 iterations, after which
the improvement slows down. 
All molecules converge after around 100 iterations,
except for thioacetone, which requires 176 steps to reach the threshold.
Although the gradient used for optimization is a stochastic estimate,
convergence is steady and systematic. We attribute this to the
fact that almost all trace estimates entering the gradient
are remarkably well-determined with a \ac{sem} less than \SI{1}{\percent} (see Table~\ref{tab:trace_estimates} for an example), 
even though the number of probe vectors
(35) and the preconditioner rank (10) are not very large.
The exceptions in Table~\ref{tab:trace_estimates} are $\sigma^2_0$ and $\sigma^2_1$, whose \ac{sem}
are rather large (\SI{518}{\percent} and \SI{37.6}{\percent}, respectively). 
The large uncertainty on the $\sigma^2_0$ trace term does not appear to impair the overall convergence,
something that we ascribe to the robustness of the Adam optimizer and to
the fact that $\sigma^2_0$ is anyway very small.

\begin{figure}[t]
\centering
\includegraphics[width=1.0\linewidth]{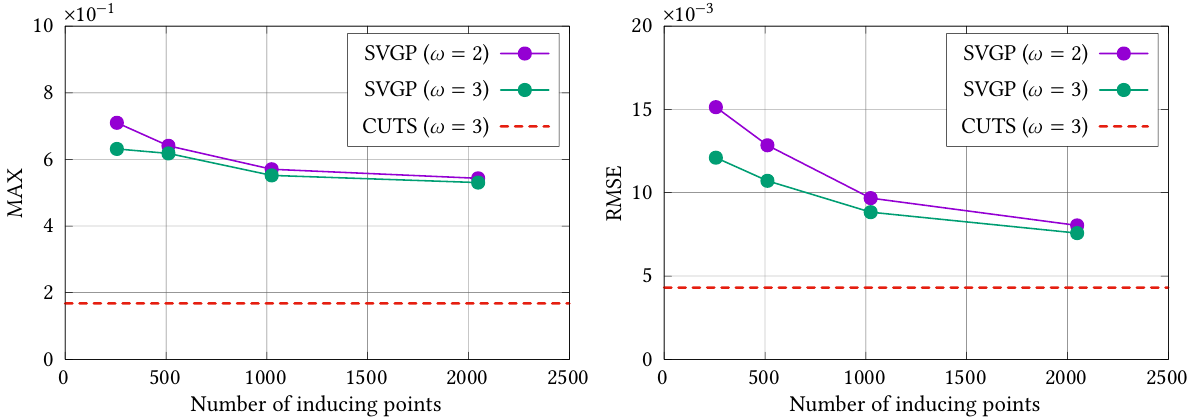}
\caption{Comparison of range normalized MAX and RMSE. The errors are averaged over the
ten molecules in Table~\ref{tab:test_molecules}.}
\label{fig:max_rmse}
\end{figure}

Table~\ref{tab:timings} shows the CPU and wall times for hyperparameter optimization and predictions in CUTS-GPR (only the predictive mean was computed).
The wall time spent on hyperparameter optimization ranges from \SI{1.62}{\hour} to \SI{3.58}{\hour} (mean \SI{2.16}{\hour}), while predictions take less than two minutes for all $\num{1 604 576}$ test points.
This difference simply reflects the fact that the \ac{cg} solver is run \textit{many} times during optimization, while
the predictive mean requires only a single run.
Appendix~\ref{appendix:svgp_timings} contains the SVGP wall times. Even the smallest SVGP
calculations (with $\omega = 2$ and $m = 256$) take between \SI{20}{\min} and \SI{51}{\min} to complete. For $\omega = 3$, SVGP is
almost always slower than CUTS-GPR. Many factors impact
absolute timings, but we think it is fair to say that the cost of CUTS-GPR is comparable to or smaller than SVGP for the cases
in question. This is remarkable, considering the fact that CUTS-GPR is numerically exact.

Figure~\ref{fig:max_rmse} compares the maximum absolute error (MAX) and
root mean square error (RMSE)
for CUTS-GPR and SVGP (both error measures are range normalized and averaged over the ten molecules).
The numerical values can be found in Appendix~\ref{appendix:average_errors}, which
also considers the mean absolute error (MAE).
We find that CUTS-GPR outperforms SVGP for all three error measures (MAX, RMSE and MAE), 
which highlights the advantage of treating the kernel exactly.
This is in fact true for each individual test case (see Appendix~\ref{appendix:all_errors}).
The difference between CUTS-GPR and SVGP is particularly large for the maximum error, which indicates
that SVGP is not able to describe the most difficult points in spite of the
fact that the target function is quite smooth. 
As seen in Figure~\ref{fig:max_rmse},
the SVGP errors seem to stagnate when adding additional inducing points.
Thus, increasing the number of points further is both impractical 
and unlikely to significantly improve accuracy.
Even for CUTS-GPR, the maximum error is significantly larger than the MAE and RMSE.
Large errors chiefly occur at points where the target function is very steep (see Figure~\ref{fig:one_mode_cut} for one example),
which is to be expected considering the relative coarseness of the training grid.

\section{CONCLUSIONS, LIMITATIONS AND EXTENSIONS} \label{sec:limitations_extensions}

We have introduced CUTS-GPR, a new approach for doing numerically exact \ac{gpr} 
with large datasets in high-dimensional settings where an additive model is appropriate. 
CUTS-GPR is based on the combination of two attractive components: (i) Additive kernels and (ii) structured, incomplete grids.
Careful analysis shows the surprising fact that this combination implies a highly structured kernel,
which in turn allows a scalable kernel \ac{mvp}.
The \ac{mvp} offers low-order polynomial scaling with $D$ \textit{and} near-linear or even linear scaling with $N$,
something that we demonstrate theoretically and empirically.
Without approximating the kernel matrix, CUTS-GPR enables \ac{gpr} 
in a previously inaccessible and highly useful setting.

In its current form, CUTS-GPR assumes a particular grid-based data format,
which is a limitation for many existing datasets.
However, when doing high-dimensional regression tasks in science and engineering, 
the user often controls the sampling of data and hence the data structure.
CUTS-GPR is naturally suited for such cases, offering numerically exact results
and very low end-to-end scaling.
Nevertheless, we outline three possible strategies for relaxing the requirements on the data format:
(i) Accidental missing data points can be handled by introducing an extra layer of chopping.
In contrast to the chopping inside $\hatbold{C}$, which is handled \textit{implicitly},
the extra chopping matrix should be applied \textit{explicitly}.
(ii) We have mentioned that the notion of an incomplete grid can be generalized by using the flexibility of
the \ac{cud} property (Definition~\ref{def:cud}). This potentially allows a wider range of grid-type data
sets to be addressed.
(iii) The reliance on strict grid structure in the input data can be lifted altogether by interpolating non-grid data onto an incomplete grid, 
similar to KISS-GP \citep{wilsonKernelInterpolationScalable2015} whose applicability in high
dimension is limited by the use of a complete grid.

Computing predictive variances for very large test sets is a bottleneck in our current implementation.
In the future, we plan on combining CUTS-GPR with Lanczos variance estimates (LOVE) \citep{pleissConstantTimePredictiveDistributions2018a}
to enable much faster computation of variances.


\subsubsection*{AI Use Statement}

In this work, we used generative AI tools for
(i)   suggesting the synthetic functions used in Section~\ref{sec:end_to_end_scaling};
(ii)  implementing the third-order OAK; and
(iii) writing scripts for reformatting datasets and gathering results.
We have not used generative AI tools for other tasks.
We have reviewed all AI-assisted work. LLM-generated code was verified and tested for correctness by one author.
We take responsibility for the final content of this work, including text, claims, or artifacts produced with the aid of generative AI.

\subsubsection*{Acknowledgements}
Anonymous acknowledgements.

\hypersetup{urlcolor=black}
\subsubsection*{References}
{
\printbibliography[heading=none]
}

\hypersetup{urlcolor=magenta}

\clearpage
\section*{Checklist}



\begin{enumerate}

  \item For all models and algorithms presented, check if you include:
  \begin{enumerate}
    \item A clear description of the mathematical setting, assumptions, algorithm, and/or model. 
    \ans{} \textbf{Yes}.
    \jus{} The paper and the appendices state the settings/assumptions, describe the CUTS-GPR method, and provide detailed derivations of algorithms.
    \item An analysis of the properties and complexity (time, space, sample size) of any algorithm.
    \ans{} \textbf{Yes}.
    \jus{} We provide a detailed complexity analysis for the CUTS-GPR kernel matrix-vector product as well as auxiliary algorithms related to the preconditioner.
    \item (Optional) Anonymized source code, with specification of all dependencies, including external libraries.
    \ans{} \textbf{No}.
    \jus{} While we do provide open access to the data (see also Questions 3a and 4c), it has not been possible to make an anonymized version of the code. 
    The code created for this paper is part of a larger program, which is released to the public at regular intervals. 
    The Python scripts used to run the reference calculations are provided as an anonymized zip file (supplemental material).
  \end{enumerate}

  \item For any theoretical claim, check if you include:
  \begin{enumerate}
    \item Statements of the full set of assumptions of all theoretical results.
    \ans{} \textbf{Yes}.
    \item Complete proofs of all theoretical results.
    \ans{} \textbf{Yes}.
    \item Clear explanations of any assumptions.
    \ans{} \textbf{Yes}.
  \end{enumerate}

  \item For all figures and tables that present empirical results, check if you include:
  \begin{enumerate}
    \item The code, data, and instructions needed to reproduce the main experimental results (either in the supplemental material or as a URL). 
    \ans{} \textbf{Yes}.
    \jus{} While the source code is not provided (see Question 1c), we provide the details necessary to implement CUTS-GPR. The datasets
    are publicly available via an anonymous repository (\url{https://anonymous.4open.science/r/DONT_GET_YOUR_KRONECKERS_IN_A_TWIST_DATASETS-F084}),
    which is referenced in Appendices~\ref{appendix:exact_baselines}, \ref{appendix:end_to_end_scaling} and \ref{appendix:application_to_pes}.
    \item All the training details (e.g., data splits, hyperparameters, how they were chosen).
    \ans{} \textbf{Yes}.
    \jus{} We provide all computational details and settings in Appendices~\ref{appendix:exact_baselines} to \ref{appendix:application_to_pes}.
    \item A clear definition of the specific measure or statistics and error bars (e.g., with respect to the random seed after running experiments multiple times). 
    \ans{} \textbf{Yes}.
    \jus{} For the experiments in Section~\ref{sec:application_to_pes} we show simple plots of the average RMSE and MAX (with no error bars)
    and give references to the numerical values (including standard deviations) in Appendix~\ref{appendix:average_errors}.
    Error measures (MAX, MAE and RMSE) for each test case are shown in Tables~\ref{tab:errors_all_a} and \ref{tab:errors_all_b}.
    Full details on timings are given in Tables~\ref{tab:timings}, \ref{tab:timings_all_a} and \ref{tab:timings_all_b}. In the main text, the timings are summarized by their minimum, maximum and mean values.
    \item A description of the computing infrastructure used. 
    \ans{} \textbf{Yes}.
    \jus{} The computational resources are stated in Appendices~\ref{appendix:exact_baselines} to \ref{appendix:application_to_pes}.
  \end{enumerate}

  \item If you are using existing assets (e.g., code, data, models) or curating/releasing new assets, check if you include:
  \begin{enumerate}
    \item Citations of the creator, if your work uses existing assets.
    \ans{} \textbf{Yes}.
    \jus{} We use three publicly available programs (xTB, MidasCpp and Google Benchmark), all of which are cited appropriately (see Appendices~\ref{appendix:mvp_scaling} and \ref{appendix:exact_baselines}).
    \item The license information of the assets, if applicable.
    \ans{} \textbf{Yes}.
    \jus{} The license information of xTB, MidasCpp and Google Benchmark is provided as footnotes in Appendices~\ref{appendix:mvp_scaling} and \ref{appendix:exact_baselines}.
    \item New assets either in the supplemental material or as a URL, if applicable.
    \ans{} \textbf{Yes}.
    \jus{} We make the datasets publicly available via an anonymous repository,
    which is referenced in Appendices~\ref{appendix:exact_baselines}, \ref{appendix:end_to_end_scaling} and \ref{appendix:application_to_pes}.
    \item Information about consent from data providers/curators.
    \ans{} \textbf{Not applicable}.
    \item Discussion of sensible content if applicable, e.g., personally identifiable information or offensive content.
    \ans{} \textbf{Not applicable}.
  \end{enumerate}

  \item If you used crowdsourcing or conducted research with human subjects, check if you include:
  \begin{enumerate}
    \item The full text of instructions given to participants and screenshots.
    \ans{} \textbf{Not applicable}.
    \item Descriptions of potential participant risks, with links to Institutional Review Board (IRB) approvals if applicable.
    \ans{} \textbf{Not applicable}.
    \item The estimated hourly wage paid to participants and the total amount spent on participant compensation.
    \ans{} \textbf{Not applicable}.
  \end{enumerate}

\end{enumerate}

\clearpage
\appendix
\thispagestyle{empty}

\onecolumn
\aistatstitle{Don’t Get Your Kroneckers in a Twist: Gaussian Processes on High-Dimensional Incomplete Grids\\
Supplementary Materials}

\numberwithin{equation}{section}
\numberwithin{figure}{section}
\numberwithin{table}{section}

\section{Gaussian process regression using matrix-vector products} \label{appendix:gpr_mvp}
\renewcommand{\theequation}{\thesection\arabic{equation}}

This appendix gives a brief account of \ac{mvp}-based \ac{gpr} 
\citep{gardnerGPyTorchBlackboxMatrixMatrix2018, wengerPreconditioningScalableGaussian2022},
which is now a well-established approach in the GP literature.
The purpose of this account is to provide an unambiguous definition
of the components entering CUTS-GPR in a notation consistent with the main text.
In addition, this appendix makes the central equations and assumptions available for explicit reference
in other sections and appendices.

\subsection{Conjugate gradient} \label{appendix:cg}
Linear systems like $\mathbf{C} \mathbf{x} = \mathbf{b}$ can be solved
iteratively using the well-known \ac{cg} algorithm \citep{hestenesMethodsConjugateGradients1952,meurantLanczosConjugateGradient2006,saadIterativeMethodsSparse2007}.
Since the overall computational effort scales linearly with the number of \ac{cg}
iterations, it is important to accelerate convergence as much as possible.
This can be achieved by using a preconditioner, 
i.e. a \ac{spd} matrix $\mathbf{P} \approx \mathbf{C}$
that allows fast solves, $\mathbf{P}^{-1} \mathbf{v}$. 
The preconditioner is formally introduced into the \ac{cg} algorithm 
by using its Cholesky
decomposition $\mathbf{P} = \mathbf{E} \mathbf{E}^{\trans}$
to write an equivalent linear system,
\begin{align}
  (\mathbf{E}^{-1} \mathbf{C} \mathbf{E}^{-\trans}) (\mathbf{E}^{\trans} \mathbf{x})
  = (\mathbf{E}^{-1} \mathbf{b})
\end{align}
This system has the same solution as the original system, but
convergence depends on the conditioning of the matrix $\tilde{\mathbf{C}} = \mathbf{E}^{-1} \mathbf{C} \mathbf{E}^{-\trans}$
rather than $\mathbf{C}$ itself. If $\mathbf{P}$ approximates $\mathbf{C}$ well
then $\tilde{\mathbf{C}} \approx \mathbf{I}$
and convergence can be much improved. This is often summarized in terms
of condition numbers:
\begin{align}
  \kappa_{\mathrm{precond}} 
  = \condition{\tilde{\mathbf{C}}} \ll \condition{\mathbf{C}} = \kappa_{\mathrm{orig}}.
\end{align}
The Cholesky factor $\mathbf{E}$ is not needed in the actual 
implementation of the algorithm and serves only
as a formal device. All that is needed is \acp{mvp} with the inverse
preconditioner, $\mathbf{P}^{-1} \mathbf{v}$.

Any convenient
preconditioner can be used, but we will rely on a rank-$k$ pivoted Cholesky
decomposition \citep{beebeSimplificationsGenerationTransformation1977, harbrechtLowrankApproximationPivoted2012, chenRandomlyPivotedCholesky2025} such that
\begin{gather}
  \mathbf{Z} \mathbf{Z}^{\trans} \approx \mathbf{K}, \quad \mathbf{Z} \in \mathbb{R}^{N \times k}  \\
  \mathbf{P} = \mathbf{Z} \mathbf{Z}^{\trans} + \sigma^2 \mathbf{I} \approx \mathbf{C}.
\end{gather}
From the matrix inverse and matrix determinant lemmas
we get the following simple expressions:
\begin{gather}
  \mathbf{P}^{-1} = \frac{1}{\sigma^2}
  \left( \mathbf{I}_{N} - \mathbf{Z} \mathbf{Q}^{-1} \mathbf{Z}^{\trans} \right), \label{eq:Pinv_woodbury} \\
  \log{|\mathbf{P}|} = \log{|\mathbf{Q}|} + (N - k) \log(\sigma^2), \label{eq:logdetP} \\
  \mathbf{Q} = \sigma^2 \mathbf{I}_k + \mathbf{Z}^{\trans} \mathbf{Z} \in \mathbb{R}^{k \times k}.  \label{eq:precond_Q_def}
\end{gather}
The matrix $\mathbf{Q}$ and its Cholesky factorization can be computed in $\mathcal{O}(k^2 N)$
and $\mathcal{O}(k^3)$ time, respectively. Once this is done,
the cost of $\log{|\mathbf{P}|}$ is only $\mathcal{O}(k)$ while $\mathbf{P}^{-1} \mathbf{v}$
has a complexity of $\mathcal{O}(kN)$.

The pivoted Cholesky algorithm evaluates the main diagonal and $k$ columns of $\mathbf{K}$ at a cost of 
$\mathcal{O}(\rho_{\mathrm{diag}})$ and $\mathcal{O}(k \rho_{\mathrm{col}})$, respectively ($\rho_{\mathrm{diag}}$ and $\rho_{\mathrm{col}}$ depend on the structure of the kernel; see Appendices~\ref{appendix:K_column} and \ref{appendix:K_diagonal}).
In addition, the algorithm expends $\mathcal{O}(k^2 N)$ arithmetic operations
and requires $\mathcal{O}(kN)$ storage \citep{chenRandomlyPivotedCholesky2025}.

\subsection{Lanczos}
The (non-preconditioned) Lanczos algorithm \citep{lanczosIterationMethodSolution1950,chenLanczosAlgorithmMatrix2024} can be described very informally as follows:
Running $\ell$ iterations of the Lanczos algorithm with a symmetric matrix $\mathbf{C}$
and starting vector $\mathbf{b}$ builds an approximation
\begin{align}
  \mathbf{C} \approx \mathbf{Q}_{\ell}^{\phantom{\trans}} \mathbf{T}_{\ell}^{\phantom{\trans}} \mathbf{Q}_{\ell}^{\trans}, 
  \quad \mathbf{T}_{\ell} \in \mathbb{R}^{\ell \times \ell},
  \quad \mathbf{Q}_{\ell} \in \mathbb{R}^{N    \times \ell}.
\end{align}
The matrix $\mathbf{T}_{\ell}$ is symmetric tridiagonal
and the columns of $\mathbf{Q}_{\ell}$ (the Lanczos vectors) are
orthonormal (in exact arithmetic) with the first column being $\mathbf{q}_0 = \mathbf{b} / \norm{\mathbf{b}}$.
It follows that
\begin{align} \label{eq:lanczos_QTb}
  \mathbf{Q}_{\ell}^{\trans} \mathbf{b} = \mathbf{e}_0 \norm{\mathbf{b}}.
\end{align}
Lanczos can be used to approximate quadratic forms involving matrix functions,
a technique known as Lanczos quadrature \citep{chenLanczosAlgorithmMatrix2024}:
\begin{align}
  \mathbf{b}^{\trans} f(\mathbf{C}) \mathbf{b} 
  &\approx \mathbf{b}^{\trans} \mathbf{Q}_{\ell}^{\phantom{\trans}} f(\mathbf{T}_{\ell}^{\phantom{\trans}}) \mathbf{Q}_{\ell}^{\trans} \mathbf{b} \nonumber \\
  &= \norm{\mathbf{b}}^2 \mathbf{e}_0^{\trans} f(\mathbf{T}_{\ell}) \mathbf{e}_0.
\end{align}
Although the equality due to \eqref{eq:lanczos_QTb} only holds in exact arithmetic,
one should always use the latter expression \citep{muscoStabilityLanczosMethod2024}. 
This has the added benefit that the Lanczos vectors need not be stored.

The Lanczos algorithm (with starting vector $\mathbf{b}$) 
and the \ac{cg} algorithm (with right-hand side $\mathbf{b}$ and zero starting guess) 
are intimately related,
and the quantities of one can be obtained from the other \citep{meurantLanczosConjugateGradient2006,saadIterativeMethodsSparse2007}.
With a slight (and very inexpensive) modification of \ac{cg} \citep{gardnerGPyTorchBlackboxMatrixMatrix2018} we can thus obtain
the Lanczos tridiagonal matrix $\mathbf{T}_{\ell}$ and compute
\begin{align} \label{eq:lanczos_quadrature_log}
  \mathbf{b}^{\trans} \log(\mathbf{C}) \mathbf{b} 
  \approx \norm{\mathbf{b}}^2 \mathbf{e}_0^{\trans} \log(\mathbf{T}_{\ell}) \mathbf{e}_0.
\end{align}
The matrix logarithm of $\mathbf{T}_{\ell}$ can be computed
by diagonalization, which has a cost of $\mathcal{O}(\ell^2)$ since
$\mathbf{T}_{\ell}$ is symmetric tridiagonal.

If a preconditioner is used,
the equations above must be modified slightly.
In this case the Lanczos algorithm (and the modified \ac{cg} algorithm) 
build an approximation
of the preconditioned matrix $\tilde{\mathbf{C}} = \mathbf{E}^{-1} \mathbf{C} \mathbf{E}^{-\trans}$
and the formal starting vector is the preconditioned vector 
$\tilde{\mathbf{b}} = \mathbf{E}^{-1} \mathbf{b}$ (rather the user-supplied vector $\mathbf{b}$).
Comparing with \eqref{eq:lanczos_quadrature_log} we thus find that
\begin{align} \label{eq:lanczos_quadrature_log_precond_tilde}
  \tilde{\mathbf{b}}^{\trans} \log(\tilde{\mathbf{C}}) \tilde{\mathbf{b}}
  \approx \norm{\tilde{\mathbf{b}}}^2 \mathbf{e}_0^{\trans} \log(\tilde{\mathbf{T}}_{\ell}) \mathbf{e}_0
\end{align}
or, equivalently,
\begin{align} \label{eq:lanczos_quadrature_log_precond}
  \mathbf{b}^{\trans} \mathbf{E}^{-\trans} \log(\tilde{\mathbf{C}}) \mathbf{E}^{-1} \mathbf{b}
  \approx (\mathbf{b}^{\trans} \mathbf{P}^{-1} \mathbf{b}) \mathbf{e}_0^{\trans} \log(\tilde{\mathbf{T}}_{\ell}) \mathbf{e}_0.
\end{align}
Importantly, the right-hand side of \eqref{eq:lanczos_quadrature_log_precond} does
not refer to $\mathbf{E}$, which can once again be considered as a formal device.

\subsection{Stochastic trace estimation} \label{sec:stochastic_trace_estimation}
The trace of a matrix $\mathbf{A} \in \mathbb{R}^{N \times N}$ can be
computed as the expectation value of a quadratic form.
To see this, we consider a random vector $\mathbf{b} \in \mathbb{R}^{N}$
with $\mathbb{E}[\mathbf{b} \mathbf{b}^{\trans}] = \mathbf{I}$,
e.g. a Rademacher random vector.
Using the cyclic property of the trace
and the linearity of the
expectation value, we find that
\begin{align}
    \mathbb{E}[ \mathbf{b}^{\trans} \mathbf{A} \mathbf{b} ] 
  &= \mathbb{E}[ \trace(\mathbf{b}^{\trans} \mathbf{A} \mathbf{b}) ] \nonumber \\
  &= \mathbb{E}[ \trace(\mathbf{A} \mathbf{b} \mathbf{b}^{\trans}) ] \nonumber \\
  &= \trace(\mathbf{A} \mathbb{E}[ \mathbf{b} \mathbf{b}^{\trans}] ) \nonumber \\
  &= \trace(\mathbf{A}). \label{eq:stochastic_trace_estimation}
\end{align}
This observation leads to the Girard--Hutchinson estimator \citep{girardFastMonteCarloCrossvalidation1989,hutchinsonStochasticEstimatorTrace1989} with $m$ probe vectors:
\begin{align}
  \trace(\mathbf{A}) \approx \frac{1}{m} \sum_{i = 1}^{m} \mathbf{b}_i^{\trans} \mathbf{A} \mathbf{b}_i^{\phantom{\trans}}.
\end{align}
Due to the use of a preconditioner, we must 
modify the trace estimation slightly and use a random vector with
$\mathbb{E}[\mathbf{b} \mathbf{b}^{\trans}] = \mathbf{P}$ 
(Section~\ref{sec:sampling_probe_vectors} shows how to sample such a vector).
Using the same
arguments that lead to \eqref{eq:stochastic_trace_estimation} we see that
\begin{align}
  \trace(\mathbf{A}) 
  &= \mathbb{E}[\mathbf{b}^{\trans} \mathbf{A} \mathbf{P}^{-1} \mathbf{b}] \label{eq:stochastic_trace_estimation_Pinv} \\
  &= \mathbb{E}[\mathbf{b}^{\trans} \mathbf{E}^{-\trans} \mathbf{A} \mathbf{E}^{-1} \mathbf{b}]. \label{eq:stochastic_trace_estimation_Einv}
\end{align}
Equation \eqref{eq:stochastic_trace_estimation_Pinv} is used to estimate the trace term,
\begin{align}
  \trace\left( \mathbf{C}^{-1} \pdv{\mathbf{C}}{\theta} \right) 
  &= \mathbb{E}[\mathbf{b}^{\trans} \mathbf{C}^{-1} \pdv{\mathbf{C}}{\theta} \mathbf{P}^{-1} \mathbf{b}] \nonumber \\
  &\approx \frac{1}{m} \sum_{i = 1}^{m} \mathbf{b}_{i}^{\trans} \mathbf{C}^{-1} \pdv{\mathbf{C}}{\theta} \mathbf{P}^{-1} \mathbf{b}_{i}, \label{eq:estimate_trace_term}
\end{align}
while \eqref{eq:stochastic_trace_estimation_Einv} and \eqref{eq:lanczos_quadrature_log_precond}
result in an estimate for the log-determinant of $\tilde{\mathbf{C}}$:
\begin{align}
  \log{|\tilde{\mathbf{C}}|} 
  &= \trace{\log(\tilde{\mathbf{C}})} \nonumber \\
  &= \mathbb{E}[\mathbf{b}^{\trans} \mathbf{E}^{-\trans} \log(\tilde{\mathbf{C}}) \mathbf{E}^{-1} \mathbf{b}] \nonumber \\
  &\approx \frac{1}{m} \sum_{i = 1}^{m} \mathbf{b}_{i}^{\trans} \mathbf{E}^{-\trans} \log(\tilde{\mathbf{C}}) \mathbf{E}^{-1} \mathbf{b}_{i} \nonumber \\
  &\approx \frac{1}{m} \sum_{i = 1}^{m} (\mathbf{b}_{i}^{\trans} \mathbf{P}^{-1} \mathbf{b}_{i}) \mathbf{e}_0^{\trans} \log(\tilde{\mathbf{T}}_{\ell,i}) \mathbf{e}_0. \label{eq:estimate_log_determinant}
\end{align}
It is important to note that \eqref{eq:estimate_trace_term}
uses stochastic trace estimation alone, while 
\eqref{eq:estimate_log_determinant} involves two separate approximations,
namely stochastic trace estimation and Lanczos quadrature (this combination is called stochastic Lanczos quadrature).

For each probe vector $\mathbf{b}_i$ a single run of the modified, preconditioned \ac{cg} algorithm
produces the linear solve $\mathbf{C}^{-1} \mathbf{b}_i$ used in \eqref{eq:estimate_trace_term}
as well as the Lanczos tridiagonal matrix $\tilde{\mathbf{T}}_{\ell,i}$ used in \eqref{eq:estimate_log_determinant}.
The preconditioner solve $\mathbf{P}^{-1} \mathbf{b}_i$ is used in both estimates and is easily obtained
from \eqref{eq:Pinv_woodbury} as already explained.

Having computed an estimate of $\log{|\tilde{\mathbf{C}}|}$, we need to recover
an estimate of $\log{|\mathbf{C}|}$. Using the fact that $\mathbf{C}$
and $\mathbf{P}$ have positive determinants (since they are \ac{spd})
we get
\begin{align}
  \log{|\tilde{\mathbf{C}}|} 
  = \log{|\mathbf{E}^{-1} \mathbf{C} \mathbf{E}^{-\trans}|}
  &= \log{|\mathbf{C} \mathbf{P}^{-1}|} \nonumber \\
  &= \log{|\mathbf{C}|} - \log{|\mathbf{P}|}
\end{align}
or simply
\begin{align}
  \log{|\mathbf{C}|} = 
  \log{|\mathbf{P}|} + \log{|\tilde{\mathbf{C}}|}.
\end{align}
The estimated log-determinant of $\mathbf{C}$ can thus be expressed as the sum of two
terms: A deterministic approximation based on the preconditioner and a stochastic estimate of the remainder.
If the preconditioner is sufficiently good, the remainder will be small and the associated stochastic noise
will also be small. This reduces the variance of the overall estimate of $\log{|\mathbf{C}|}$ and the \ac{mll}
\citep{wengerPreconditioningScalableGaussian2022}, which is very helpful when optimizing 
hyperparameters.

It is possible to achieve a similar variance reduction for the trace term by
simply adding and subtracting a preconditioner term:
\begin{align}
  \trace\left( \mathbf{C}^{-1} \pdv{\mathbf{C}}{\theta} \right)
  = \trace\left( \mathbf{P}^{-1} \pdv{\mathbf{P}}{\theta} \right)
  + \left[ 
    \trace\left( \mathbf{C}^{-1} \pdv{\mathbf{C}}{\theta} \right) 
  - \trace\left( \mathbf{P}^{-1} \pdv{\mathbf{P}}{\theta} \right) 
  \right]
\end{align}
The first term on the right-hand side is 
computed deterministically (via a closed-form expression),
while the term in square brackets is estimated by using \eqref{eq:estimate_trace_term}.

\subsection{Sampling the probe vectors} \label{sec:sampling_probe_vectors}

Consider two independent random vectors $\mathbf{v} \in \mathbb{R}^{k}$ and 
$\mathbf{w} \in \mathbb{R}^{N}$ such that 
\begin{subequations} \label{eq:random_vectors_vw}
  \begin{alignat}{5}
    \mathbb{E}[\mathbf{v} \mathbf{v}^{\trans}] &= \mathbf{I}_{k} &&,  & \quad \bm{\mu} &= \mathbb{E}[\mathbf{v}] &&= 0, \\
    \mathbb{E}[\mathbf{w} \mathbf{w}^{\trans}] &= \mathbf{I}_{N} &&,  & \quad \bm{\nu} &= \mathbb{E}[\mathbf{w}] &&= 0.
  \end{alignat}
\end{subequations}
In this paper, $\mathbf{v}$ and $\mathbf{w}$ are Rademacher random vectors.
Since the vectors are independent with mean zero we have
\begin{align} \label{eq:EvwT_zero}
  \mathbf{0} = \mathrm{cov}(\mathbf{v}, \mathbf{w})
  = \mathbb{E}[(\mathbf{v} - \bm{\mu})(\mathbf{w} - \bm{\nu})^{\trans}]
  = \mathbb{E}[\mathbf{v} \mathbf{w}^{\trans}].
\end{align}
Now let
\begin{align}
  \mathbf{b} = \mathbf{Z} \mathbf{v} + \sigma \mathbf{w}, \quad \mathbf{Z} \in \mathbb{R}^{N \times k}.
\end{align}
Expanding the terms and using \eqref{eq:random_vectors_vw} and \eqref{eq:EvwT_zero}
one easily finds that
\begin{align}
  \mathbb{E}[\mathbf{b} \mathbf{b}^{\trans}] 
  = \mathbf{Z} \mathbf{Z}^{\trans} + \sigma^2 \mathbf{I} 
  = \mathbf{P}.
\end{align}

\newpage
\section{Formal definition of subgrids and incomplete grids} \label{appendix:incomplete_grid_formal_def}
\renewcommand{\theequation}{\thesection\arabic{equation}}

\begin{definition} \label{def:subgrid}
  The subgrid $\mathcal{I}^{(\mathbf{m})}$ with displaced dimensions $\mathbf{m} = (m_1, m_2, \ldots, m_k)$
  and dimension $k = |\mathbf{m}|$
  is the set
  of multi-indices of the form
  \begin{align}
    \mathbf{i} = (i_1, i_2, \ldots, i_D) ,\quad
    i_m =
    \begin{dcases}[c]
      0   & \text{if } m \notin \mathbf{m}, \\
      a_m & \text{if } m \in    \mathbf{m},
    \end{dcases} \nonumber
  \end{align}
  where $0 < a_m < n_m$. Recording the non-zero indices in a vector 
  $\mathbf{a} = (a_{m_1}, a_{m_2}, \ldots, a_{m_k})$
  and making the association $\mathbf{i} \cong (\mathbf{m}, \mathbf{a})$,
  we may equivalently define $\mathcal{I}^{(\mathbf{m})}$ as the set of tuples
  \begin{align}
    (\mathbf{m}, \mathbf{a}) ,\quad
    \mathbf{a} \in 
    \mathcal{A}^{(\mathbf{m})} =
    \mathcal{A}^{(m_1)} \times
    \mathcal{A}^{(m_2)} \times
    \cdots \times
    \mathcal{A}^{(m_k)}, \nonumber
  \end{align}
  where $\mathcal{A}^{(m)} = \mathcal{I}^{(m)} \setminus \{0\}$ is the set
  of non-zero indices for dimension $m$. 
  Generally,
  \begin{align}
    |\mathcal{I}^{(\mathbf{m})}| 
    = \prod_{l = 1}^{k} |\mathcal{A}^{(m_l)}|
    = \prod_{l = 1}^{k} \big(|\mathcal{I}^{(m_l)}| - 1\big).
  \end{align}
  For later use, we note that
  $\mathcal{I}^{(\varnothing)}$ is the set containing the reference alone.
\end{definition}
With this in place, the following definition follows naturally:
\begin{definition} \label{def:incomplete_grid}
  A cut-based incomplete grid is a set of points $\hatcal{I} \subseteq \mathcal{I}$
  of the form
  \begin{align}
    \hatcal{I} = \bigcup_{\mathbf{m} \in \gridmcr} \mathcal{I}^{(\mathbf{m})}.
  \end{align}
  The set $\gridmcr$ specifies the cuts or subgrids that are included in $\hatcal{I}$.
\end{definition}
Definition~\ref{def:incomplete_grid} suggests a certain data structure for vectors indexed by $\hatcal{I}$
(e.g. the vector of observed data, $\hatbold{y}$): The reference data point is a scalar,
the 1D data is organized as a set of vectors, the 2D data as a set of matrices and so on.
In short, a low-dimensional \textit{subgrid}, $\mathcal{I}^{(\mathbf{m})}$, translates to a low-dimensional \textit{subtensor}, $\mathbf{y}^{(\mathbf{m})}$
(cf. Example~\ref{example:subgrids}, Figure~\ref{fig:subgrids} and Figure~\ref{fig:3D_tensor}).

\begin{figure}[H]
\centering
\includegraphics[width=\linewidth]{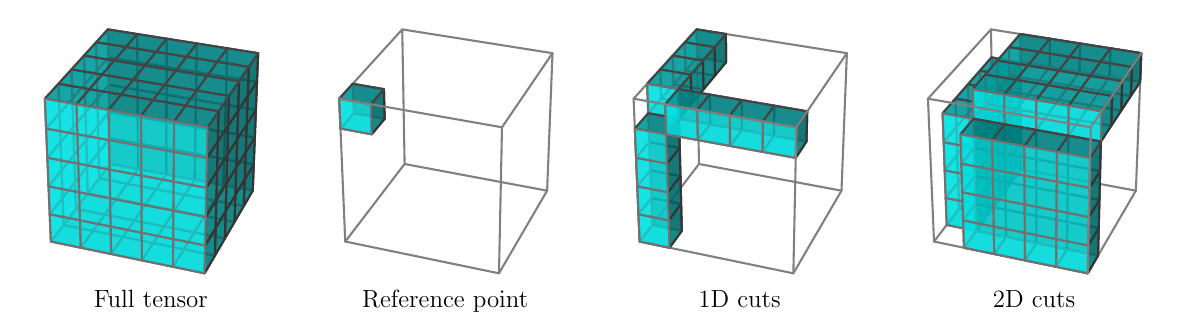}
\caption{Subtensors corresponding to the subgrids in Example~\ref{example:subgrids}.}
\label{fig:3D_tensor}
\end{figure}


\newpage
\section{A generalized notion of CUD} \label{appendix:general_cud_framework}
\renewcommand{\theequation}{\thesection\arabic{equation}}

It turns out that the version of the \ac{cud} property given in the main text (Definition~\ref{def:cud})
can be generalized considerably, as described by \citet{holzmullerFastSparseGrid2021}.
Here, we state a few selected definitions from their work that indicate how
our GPR approach can be extended to other types of grids.
The framework of \citet{holzmullerFastSparseGrid2021} relies heavily on the notion of a preorder:
\begin{definition}
A preorder is a binary relation $\leq$ on a set $\mathcal{S}$ that satisfies (i) and (ii) for all $i,j,k \in \mathcal{S}$: 
\begin{enumerate}[label=(\roman*), nosep]
  \item Reflexivity: $i \leq i$ (every element is related to itself).
  \item Transitivity: $i \leq j$ and $j \leq k$ implies $i \leq k$.
  \item Antisymmetry: $i \leq j$ and $j \leq i$ implies $i = j$ (no two distinct elements precede each other).
  \item Strong connectedness: $i \leq j$ or $j \leq i$ (all pairs of elements are related).
\end{enumerate}
A partial order satisfies (i)--(iii), while a total order (or linear order)
satisfies (i)--(iv).
\end{definition}
Having defined the concept of a preorder, we consider an index set $\mathcal{I}$
and a subset $\hatcal{I} \in \mathcal{I}$ which is \ac{cud} in the following, generalized sense:
\begin{definition} 
The subset $\hatcal{I} \subset \mathcal{I}$ is \ac{cud} 
if $i \leq j$ implies $i \in \hatcal{I}$
for all $i \in \mathcal{I}$ and $j \in \hatcal{I}$.
$\hatcal{I}$ is also said to be downward closed with respect to $\leq$.
\end{definition}
We also need a generalized notion of triangularity \citep[see][Definition 1]{holzmullerFastSparseGrid2021}:
\begin{definition} 
Let $\mathcal{I}$ be an index set equipped with a preorder, $\leq$. 
In addition, let $\mathbf{A}$ be a square matrix of size $|\mathcal{I}|$.
We define $\leq$-triangularity in the following way:
\begin{enumerate}[label=(\roman*), nosep]
  \item $\mathbf{A}$ is called lower $\leq$-triangular if $A_{ij} \neq 0$ implies $j \leq i$ for all $i,j\in \mathcal{I}$.
  \item $\mathbf{A}$ is called upper $\leq$-triangular if $A_{ij} \neq 0$ implies $i \leq j$ for all $i,j\in \mathcal{I}$.
\end{enumerate}
\end{definition}

In the context of grids, it is natural to introduce 1D index sets $\mathcal{I}^{(d)}$ equipped
with preorders $\leq_{d}$ (note that the 1D indices need not be integers).
The corresponding multi-indices are defined as in the main text:
\begin{align}
  \mathcal{I} = \mathcal{I}^{(1)} \times \cdots \times \mathcal{I}^{(D)}.
\end{align}
The 1D preorders in turn induce a preorder $\leq$ on $\mathcal{I}$ \citep[see][Definition 2]{holzmullerFastSparseGrid2021}:
\begin{definition}
  For multi-indices $\mathbf{i}, \mathbf{j} \in \mathcal{I}$, we define $\mathbf{i} \leq \mathbf{j}$
  if and only if $i_d \leq_d j_d$ for all $1 \leq d \leq D$. 
\end{definition}

With these definitions in place, Theorems~\ref{thm:chopping} and \ref{thm:chopped_one_mode_contraction} hold as they stand 
\citep[see][Theorems 1 and 3, for proofs and formal statements]{holzmullerFastSparseGrid2021}.
These theorems allow kernel \acp{mvp} to be computed without reference to the underlying complete grid, thus eliminating
the curse of dimensionality. The concrete computational complexity will, however, depend on the details of the grid.

One grid type of interest is the classical sparse grid \citep{bungartzSparseGrids2004}, whose
hierarchical structure makes it \ac{cud}.
Sparse grids are discussed in detail by \citet{holzmullerFastSparseGrid2021}, so
we will not consider them any further apart from mentioning their 
potential use in a \ac{gpr} context.

In the main text we considered the 1D index sets
\begin{align}
  \mathcal{I}^{(m)} = \{ 0, 1, \ldots, n_m - 1 \},
\end{align}
and we tacitly assumed a particular preorder,
namely the standard order on the integers. A matrix that is triangular 
with respect to the standard order is simply triangular in the ordinary sense.
It is, however, possible to choose a different preorder under which
the \ac{mcr} grid (Definitions~\ref{def:subgrid} and \ref{def:incomplete_grid})
is still \ac{cud}, i.e. a different preorder that is consistent with the \ac{mcr}
framework.
The key observation is that the non-zero indices 
$a_m \in \{1, 2, \ldots, n_m - 1 \} = \mathcal{A}^{(m)}$ are all treated on an equal footing. Only
the reference index is special.
This suggests the following preorder on $\mathcal{I}^{(m)}$:\newpage
\begin{exmp} \label{example:preorder_block_triangular}
  Let $\mathcal{I}^{(m)} = \{ 0, 1, \ldots, n - 1 \}$
  and $\mathcal{A}^{(m)} = \{1, 2, \ldots, n - 1 \}$.
  Let $\leq_m$ be a binary relation on $\mathcal{I}^{(m)}$ defined by:
  \begin{enumerate}[label=(\roman*), nosep]
    \item $0 \leq_m 0$.
    \item $0 \leq_m a$ for all $a    \in \mathcal{A}^{(m)}$.
    \item $a \leq_m b$ for all $a, b \in \mathcal{A}^{(m)}$.
  \end{enumerate}
  $\leq_m$ is easily verified to be a preorder. 
  Let $\bar{n} = n - 1$.
  The following square matrices of size $n$
  are lower and upper triangular, respectively, with respect to $\leq_m$:
  \begin{align}
    \mathbf{L}^{\!(m)} =
    \left[
    \begin{array}{c|cccc}
      l_{0,0}       & 0             & 0             & \cdots & 0             \dhstrut{6pt}{0pt}  \\ \hline
      l_{1,0}       & l_{1,1}       & l_{1,2}       & \cdots & l_{1,\bar{n}} \dhstrut{0pt}{10pt} \\
      l_{2,0}       & l_{2,1}       & l_{2,2}       & \cdots & l_{2,\bar{n}}       \\
      \vdots        & \vdots        & \vdots        & \ddots & \vdots              \\
      l_{\bar{n},0} & l_{\bar{n},1} & l_{\bar{n},2} & \cdots & l_{\bar{n},\bar{n}} \\
    \end{array}
    \right],
    \quad
    \mathbf{U}^{(m)} =
    \left[
    \begin{array}{c|cccc}
      u_{0,0} & u_{0,1}       & u_{0,2}       & \cdots & u_{0,\bar{n}} \dhstrut{6pt}{0pt}  \\ \hline
      0       & u_{1,1}       & u_{1,2}       & \cdots & u_{1,\bar{n}} \dhstrut{0pt}{10pt} \\
      0       & u_{2,1}       & u_{2,2}       & \cdots & u_{2,\bar{n}}       \\
      \vdots  & \vdots        & \vdots        & \ddots & \vdots              \\
      0       & u_{\bar{n},1} & u_{\bar{n},2} & \cdots & u_{\bar{n},\bar{n}} \\
    \end{array}
    \right].
  \end{align}
\end{exmp}


\newpage
\section{Chopped Kronecker products} \label{appendix:chopped_kronecker_products}
\renewcommand{\theequation}{\thesection\arabic{equation}}
\renewcommand{\thethm}{\thesection\arabic{thm}}

The following theorems provide the mathematical underpinning that enables our work:
\begin{thm} \label{thm:chopping}
  Let the index set $\hatcal{I}$ be \ac{cud} and let $\mathbf{A}, \mathbf{B} \in \mathbb{R}^{N \times N}$.
  We then have:
  \begin{enumerate}[label=(\roman*)]
  \item If $\mathbf{A}$ is lower triangular or $\mathbf{B}$ is upper triangular, 
  then $\widehat{\mathbf{A} \mathbf{B}} = \hatbold{A} \hatbold{B}$. \label{thm:chopping_i}
  \item If $\mathbf{A}$ is lower (upper) triangular and invertible, then $\mathbf{A}^{-1}$ is lower (upper) triangular
  and $\hatbold{A}$ is invertible with $\hatbold{A}^{-1} = \widehat{\mathbf{A}^{-1}}$. \label{thm:chopping_ii}
  \end{enumerate}
\end{thm}
\begin{proof}
  See \citet[Theorem 1]{holzmullerFastSparseGrid2021}.
\end{proof}
\begin{thm} \label{thm:chopped_one_mode_contraction}
Define the notation
\begin{align}
  (j_m | \mathbf{i}_{-m})   &= (i_1, \ldots, i_{m-1}, j_m, i_{m+1}, \ldots, i_D), \nonumber \\
  \hatcal{I}(\mathbf{i}, m) &= \{ j_m \in \mathcal{I}^{(m)} \;|\; (j_m | \mathbf{i}_{-m}) \in \hatcal{I} \}. \nonumber
\end{align}
Using this notation, a matrix-vector product of the form $\hatbold{w} = \hatboldsuper{A}{[m]} \hatbold{v}$ can be computed as
\begin{align}
  \hatbold{w}_{\mathbf{i}}
  = \sum_{j_m \in \hatcal{I}(\mathbf{i}, m)} A^{(m)}_{i_m j_m} v_{(j_m | \mathbf{i}_{-m})}, \quad \mathbf{i} \in \hatcal{I}. \nonumber
\end{align}
The complexity is at most $\mathcal{O}(|\mathcal{I}^{(m)}| \cdot |\hatcal{I}|)$.
\end{thm}
\begin{proof}
  See \citet[Theorem 3]{holzmullerFastSparseGrid2021}.
\end{proof}
\begin{cor} \label{cor:chopped_kronecker_product}
  Let $\{ \mathbf{A}^{\!(m)} \}_{m=1}^{D}$ be square matrices
  of size $n_m$ with LU decompositions 
  $\mathbf{A}^{\!(m)} = \mathbf{L}^{(m)} \mathbf{U}^{(m)}$. 
  For a fixed $k \in \{1, 2, \ldots, D \}$, let $\mathbf{M}^{(k)} \in \mathbb{R}^{n_{k} \times n_{k}}$ 
  be a square matrix that does not possess an LU decomposition. We define
  \begin{alignat}{3}
    \mathbf{A} 
    &= \mathbf{A}^{\!(1)} \otimes \cdots \otimes \mathbf{A}^{\!(k)} &&\otimes \cdots \otimes \mathbf{A}^{\!(D)} 
    &&= \bigg( \prod_{m} \mathbf{L}^{[m]} \bigg) \bigg( \prod_{m} \mathbf{U}^{[m]} \bigg), \nonumber \\
    \mathbf{A}^{\!\prime} 
    &= \mathbf{A}^{\!(1)} \otimes \cdots \otimes \mathbf{M}^{(k)}   &&\otimes \cdots \otimes \mathbf{A}^{\!(D)}
    &&= \bigg( \prod_{m \neq k} \mathbf{L}^{[m]} \bigg) \mathbf{M}^{[k]} \bigg( \prod_{m \neq k} \mathbf{U}^{[m]} \bigg). \nonumber
  \end{alignat}
  As a direct consequence of Theorem~\ref{thm:chopping}, we have
  \begin{align}
    \widehat{\mathbf{A}}            = \bigg( \prod_{m}        \hatboldsuper{L}{[m]} \bigg)                       \bigg( \prod_{m}        \hatboldsuper{U}{[m]} \bigg), \quad
    \widehat{\mathbf{A}^{\!\prime}} = \bigg( \prod_{m \neq k} \hatboldsuper{L}{[m]} \bigg) \hatboldsuper{M}{[m]} \bigg( \prod_{m \neq k} \hatboldsuper{U}{[m]} \bigg). \nonumber
  \end{align}
  The chopped Kronecker products $\widehat{\mathbf{A}}$ and $\widehat{\mathbf{A}^{\!\prime}}$ can be multiplied with a vector by repeated application of
  Theorem~\ref{thm:chopped_one_mode_contraction}. The complexity is at most $\mathcal{O}\big( (n_1 + \cdots + n_D) \cdot |\hatcal{I}|\big)$.
\end{cor}
The details of how to implement Theorem~\ref{thm:chopped_one_mode_contraction}
depend on the precise nature of $\hatcal{I}$. 
In our case, an algorithm can be formulated entirely in terms of the subtensors $\mathbf{v}^{(\mathbf{m})}$ and $\mathbf{w}^{(\mathbf{m})}$
rather than the individual elements of $\hatbold{v}$ and $\hatbold{w}$ (see Appendix~\ref{appendix:contractions_mcr_tensor}).
This also means that the scalar multiplication in Theorem~\ref{thm:chopped_one_mode_contraction} 
is replaced by standard operations on the subtensors (scaling, outer products and one-index contractions),
which is a great advantage in terms of performance.

For completeness, we note that Theorem~\ref{thm:chopping}\ref{thm:chopping_ii}
allows direct inversion of a separable kernel defined over an incomplete grid.
To see this, consider the complete-grid kernel (which is a Kronecker product) and introduce the
LU (Cholesky) decomposition of the base kernels. We then have
\begin{align}
  \mathbf{K} 
  = \mathbf{K}^{[1]} \mathbf{K}^{[2]} \cdots \mathbf{K}^{[D]}
  = \big( \mathbf{L}^{[1]} \mathbf{L}^{[2]} \cdots \mathbf{L}^{[D]}  \big)
     \big( \mathbf{U}^{[1]} \mathbf{U}^{[2]} \cdots \mathbf{U}^{[D]}  \big).
\end{align}
By applying Theorem~\ref{thm:chopping}, we find
\begin{align} 
  \hatbold{K} &= 
  \big( \hatboldsuper{L}{[1]} \hatboldsuper{L}{[2]} \cdots \hatboldsuper{L}{[D]} \big)
  \big( \hatboldsuper{U}{[1]} \hatboldsuper{U}{[2]} \cdots \hatboldsuper{U}{[D]} \big), \label{eq:separable_kernel_incomplete_grid} \\
  \hatbold{K}^{-1} &= 
  \big( \hatboldsupertilde{U}{[1]} \hatboldsupertilde{U}{[2]} \cdots \hatboldsupertilde{U}{[D]} \big)
  \big( \hatboldsupertilde{L}{[1]} \hatboldsupertilde{L}{[2]} \cdots \hatboldsupertilde{L}{[D]} \big) \label{eq:separable_kernel_incomplete_grid_inverse}
\end{align}
where we have temporarily used an overbar to denote the inverse to avoid clutter.
For example,
\begin{align}
  \bar{\mathbf{U}}^{[m]}
  = \big( \mathbf{U}^{[m]} \big)^{-1}
  = \mathbf{I}^{(1)} \otimes \cdots \otimes \big( \mathbf{U}^{(m)} \big)^{-1} \otimes \cdots \otimes \mathbf{I}^{(D)}.
\end{align}
Unfortunately, this format in \eqref{eq:separable_kernel_incomplete_grid} does not allow addition of noise, 
in contrast to
the complete-grid method of \citet{ishidaHierarchicalAdditiveInteraction2025},
which is based on eigendecomposition rather than LU decomposition.
Even in applications where the data is essentially noiseless,
noise is usually required to regularize the inverse. 
Although we are not willing to rule out this approach completely, our
initial investigations were not promising due to very ill-conditioned
numerics. In addition, the \ac{mvp} complexity associated with \eqref{eq:separable_kernel_incomplete_grid}
and \eqref{eq:separable_kernel_incomplete_grid_inverse}
is $\mathcal{O}(nD\widehat{N})$ compared to $\mathcal{O}(n \alpha \widehat{N})$,
which is achieved by CUTS-GPR. When $\alpha \ll D$, this difference
is very significant. Finally, we note that the additive kernel used in
CUTS-GPR is often more attractive than a separable kernel, especially
in high dimension.


\newpage
\section{Contractions with an MCR tensor} \label{appendix:contractions_mcr_tensor}
\renewcommand{\theequation}{\thesection\arabic{equation}}

The aim of this appendix is to specialize Theorem~\ref{thm:chopped_one_mode_contraction}
to our \ac{mcr}-based framework.
Theorem~\ref{thm:chopped_one_mode_contraction} suggests an algorithm with an outer loop over the elements of the result vector, $\hatbold{w}$.
It is instructive, however, to start by considering a simplistic algorithm where the outer loop runs over the input vector, $\hatbold{v}$, instead:
\begin{center}
  \begin{minipage}{0.7\linewidth}
    \begin{algorithm}[H]
      \caption{Compute $\hatbold{w} = \hatboldsuper{A}{[m]} \hatbold{v}$ using Theorem~\ref{thm:chopped_one_mode_contraction}.} \label{alg:chopped_one_mode_contraction_appendix}
      \begin{algorithmic}
      \For{$\mathbf{i} \in \hatcal{I}$}
        \For{$j_m \in \hatcal{I}^{(m)}$}
          \If{$(j_m | \mathbf{i}_{-m}) \in \hatcal{I}$}
            \State $\dhstrut{6pt}{0pt}w_{(j_m | \mathbf{i}_{-m})} \mathrel{{+}{=}} A^{(m)}_{j_m i_m} v_{(i_m | \mathbf{i}_{-m})}$
          \EndIf
        \EndFor
      \EndFor
      \end{algorithmic}
    \end{algorithm}
    \end{minipage}
\end{center}

Our goal is not to state all the specifics of our concrete implementation, 
but rather to explain the intuitions of how the algorithm works 
and how it may be implemented using standard tensor operations.
The final algorithm is well-known in vibrational structure theory, where it is
derived using arguments from quantum mechanics \citep{seidlerFastComputationsCorrelated2008,seidlerAutomaticDerivationEvaluation2009}.

It is important to note how the matrix $\mathbf{A}^{\!(m)}$ connects
the input multi-index $\mathbf{i} = (i_m | \mathbf{i}_{-m})$
to the result multi-index $\mathbf{i}' = (j_m | \mathbf{i}_{-m})$.
In particular, $\mathbf{A}^{\!(m)}$ acts to make the substitution $i_m \rightarrow j_m$.
We will examine how this substitution affects 
a multi-index $\mathbf{i} \in \hatcal{I}$ and, in particular, under which circumstances
an out-of-space multi-index $\mathbf{i}' \notin \hatcal{I}$ is produced.

We start by recalling a few definitions.
Let $\gridmcr$ be an \ac{mcr} that is \ac{cuts} (Definition~\ref{def:cuts}) and let 
\begin{align}
  \hatcal{I} = \bigcup_{\mathbf{m} \in \gridmcr} \mathcal{I}^{(\mathbf{m})}
\end{align}
be the corresponding incomplete grid (Definitions~\ref{def:subgrid} and \ref{def:incomplete_grid}). 
For a multi-index $\mathbf{i} \in \mathcal{I}^{(\mathbf{m})}$
we make the association $\mathbf{i} = (i_1, \ldots, i_D) \cong (\mathbf{m}, \mathbf{a})$ (see Definition~\ref{def:subgrid})
such that for all $m = 1, \ldots, D$
\begin{subequations}
  \begin{alignat}{2}
    i_m &= a_m > 0 &&\quad\text{if } m    \in \mathbf{m}, \\
    i_m &= 0       &&\quad\text{if } m \notin \mathbf{m}.
  \end{alignat}
\end{subequations}
In other words, $\mathbf{m}$ indicates the positions of the non-zero indices, which are collected in $\mathbf{a}$.
We use the notation
\begin{align}
  \mathbf{i}_{-m} &= (i_1, \ldots, i_{m-1}, i_{m+1}, \ldots, i_D), \\
  (j_m | \mathbf{i}_{-m}) &= (i_1, \ldots, i_{m-1}, j_m, i_{m+1}, \ldots, i_D)
\end{align}
to mean the multi-indices obtained from $\mathbf{i}$ by removing $i_m$ and by making the replacement $i_m \rightarrow j_m$, respectively.
Now consider a fixed $\mathbf{m} \in \gridmcr$ and a multi-index $\mathbf{i} \in \mathcal{I}^{(\mathbf{m})}$.
We first choose a mode $m \in \mathbf{m}$ and replace the non-zero index $i_m = a_m > 0$ by either zero or by a non-zero
index $b_m \in \mathcal{A}^{(m)}$. The following then holds:
\begin{alignat}{2}
  (0   | \mathbf{i}_{-m}) &\cong (\mathbf{m} \setminus m,        \mathbf{a}_{-m} ) &&\mspace{8mu}\in \mathcal{I}^{(\mathbf{m} \setminus m)}, \\
  (b_m | \mathbf{i}_{-m}) &\cong (\mathbf{m}            , (b_m | \mathbf{a}_{-m})) &&\mspace{8mu}\in \mathcal{I}^{(\mathbf{m}            )}.
\end{alignat}
If instead we choose $m \notin \mathbf{m}$ and replace the index $i_m = 0$, we find that
\begin{alignat}{2}
  (0   | \mathbf{i}_{-m}) &\cong (\mathbf{m}            ,        \mathbf{a}      ) &&\in \mathcal{I}^{(\mathbf{m}            )}, \\
  (b_m | \mathbf{i}_{-m}) &\cong (\mathbf{m} \cup      m, (b_m | \mathbf{a}     )) &&\in \mathcal{I}^{(\mathbf{m} \cup      m)}.
\end{alignat}
This can be summarized as
\begin{align}
  \text{if } m &\in \mathbf{m} \quad:\quad
  \left\{
  \begin{alignedat}{3}
    \mathbf{i} &\in \mathcal{I}^{(\mathbf{m})} \xrightarrow{\mathmakebox[3.7em][l]{\mspace{10mu}a_m \rightarrow 0  }} \mathbf{i}' &&\in \mathcal{I}^{(\mathbf{m} \setminus m)} &&\quad\text{(down)} \\
    \mathbf{i} &\in \mathcal{I}^{(\mathbf{m})} \xrightarrow{\mathmakebox[3.7em][l]{\mspace{10mu}a_m \rightarrow b_m}} \mathbf{i}' &&\in \mathcal{I}^{(\mathbf{m}            )} &&\quad\text{(forward)}
  \end{alignedat} \right. \nonumber
  \\
  \text{if } m &\notin \mathbf{m} \quad:\quad
  \left\{
  \begin{alignedat}{3}
    \mathbf{i} &\in \mathcal{I}^{(\mathbf{m})} \xrightarrow{\mathmakebox[3.7em][l]{\mspace{25.5mu}0  \rightarrow 0    }} \mathbf{i}  &&\in \mathcal{I}^{(\mathbf{m}       )} &&\quad\text{(passive)} \\
    \mathbf{i} &\in \mathcal{I}^{(\mathbf{m})} \xrightarrow{\mathmakebox[3.7em][l]{\mspace{25.5mu}0  \rightarrow b_m  }} \mathbf{i}' &&\in \mathcal{I}^{(\mathbf{m} \cup m)} &&\quad\text{(up)}
  \end{alignedat} \right. \nonumber
\end{align}
Since $\gridmcr$ it \ac{cuts}, we know that ${\mathbf{m} \!\setminus\! m} \in \gridmcr$ and $\mathcal{I}^{(\mathbf{m} \setminus m)} \subset \hatcal{I}$.
It may, however, be the case that $\mathbf{m} \cup m \notin \gridmcr$, so the last substitution may produce an out-of-space multi-index $\mathbf{i}' \notin \hatcal{I}$.
The four classes of substitutions are called \textit{down}, \textit{forward}, \textit{passive} and \textit{up}.

Now let $\hatbold{v}$ be a vector or tensor indexed over $\hatcal{I}$.
With the \ac{mcr} construction in mind, it is natural to view $\hatbold{v}$ as a collection of subtensors $\mathbf{v}^{(\mathbf{m})}$,
just like $\hatcal{I}$ is a collection of subgrids $\mathcal{I}^{(\mathbf{m})}$. The subtensors have elements $\mathbf{v}^{(\mathbf{m})}_{\mathbf{a}}$
for all $\mathbf{a} \in \mathcal{A}^{(\mathbf{m})}$.
Before introducing the subtensor notation into the algorithm, we simply split
the main body of the algorithm according to the four types of index substitutions:
\begin{center}
  \begin{minipage}{0.8\linewidth}
    \begin{algorithm}[H]
      \caption{Compute $\hatbold{w} = \hatboldsuper{A}{[m]} \hatbold{v}$ using Theorem~\ref{thm:chopped_one_mode_contraction}.} \label{alg:chopped_one_mode_contraction_udfp}
      \begin{algorithmic}
      \For{$\mathbf{i} \in \hatcal{I}$}
        \For{$j_m \in \hatcal{I}^{(m)}$}
          \If{$j_m = 0$ \textbf{and} $i_m > 0$} \Comment{down}
            \State $\dhstrut{11pt}{14pt} w_{(0 | \mathbf{i}_{-m})} \mathrel{{+}{=}} A^{(m)}_{0 i_m} v_{\mathbf{i}}$
          \ElsIf{$j_m > 0$ \textbf{and} $i_m > 0$} \Comment{forward}
            \State $\dhstrut{11pt}{14pt} w_{(j_m | \mathbf{i}_{-m})} \mathrel{{+}{=}} A^{(m)}_{j_m i_m} v_{\mathbf{i}}$
          \ElsIf{$j_m = 0$ \textbf{and} $i_m = 0$} \Comment{passive}
            \State $\dhstrut{11pt}{14pt} w_{(0 | \mathbf{i}_{-m})} \mathrel{{+}{=}} A^{(m)}_{0 0} v_{\mathbf{i}}$
          \ElsIf{$j_m > 0$ \textbf{and} $i_m = 0$} \Comment{up}
            \If{$\dhstrut{0pt}{13pt} (j_m | \mathbf{i}_{-m}) \in \hatcal{I}$} \Comment{skip out-of-space}
              \State $\dhstrut{6pt}{12pt} w_{(j_m | \mathbf{i}_{-m})} \mathrel{{+}{=}} A^{(m)}_{j_m 0} \, v_{\mathbf{i}}$
            \EndIf
          \EndIf
        \EndFor
      \EndFor
      \end{algorithmic}
    \end{algorithm}
    \end{minipage}
\end{center}
This rewrite is of course valid for any choice of $\hatcal{I}$, but it is particularly well-suited
for the \ac{mcr}-based construction.
We now replace the loop over $\mathbf{i} \in \hatcal{I}$ with an equivalent nested loop over
$\mathbf{m} \in \gridmcr$ and $\mathbf{a} \in \mathcal{A}^{(\mathbf{m})}$.
In addition, we observe that $i_m = 0$ implies to $m \in \mathbf{m}$,
while $i_m > 0$ implies to $m \notin \mathbf{m}$.
Next, we replace the loop over $j_m \in \mathcal{I}^{(m)}$ with
a separate line covering $j_m = 0$ 
and a loop over $b_m \in \mathcal{A}^{(m)} = \mathcal{I}^{(m)} \setminus \{0\}$.
Finally, we introduce the subtensor notation to obtain the following algorithm:
\begin{center}
  \begin{minipage}{0.8\linewidth}
    \begin{algorithm}[H]
      \caption{Compute $\hatbold{w} = \hatboldsuper{A}{[m]} \hatbold{v}$ using Theorem~\ref{thm:chopped_one_mode_contraction}.} \label{alg:chopped_one_mode_contraction_subtensor}
      \begin{algorithmic}
      \For{$\mathbf{m} \in \gridmcr$}
      \For{$\mathbf{a} \in \mathcal{A}^{(\mathbf{m})}$}
        \If{$m \in \mathbf{m}$}
          \State $\dhstrut{6pt}{14pt} w^{(\mathbf{m}\setminus m)}_{\mathbf{a}_{-m}} \mathrel{{+}{=}} A^{(m)}_{0a^m} v^{(\mathbf{m})}_{\mathbf{a}}$ \Comment{down}
          \For{$b_m \in \mathcal{A}^{(m)}$}
            \State $\dhstrut{8pt}{14pt} w^{(\mathbf{m})}_{(b_m | \mathbf{a}_{-m})} \mathrel{{+}{=}} A^{(m)}_{b_m a_m} v^{(\mathbf{m})}_{\mathbf{a}}$ \Comment{forward}
          \EndFor
          \State
        \Else
          \State $\dhstrut{8pt}{14pt} w^{(\mathbf{m})}_{\mathbf{a}} \mathrel{{+}{=}} A^{(m)}_{00} v^{(\mathbf{m})}_{\mathbf{a}}$ \Comment{passive}
          \If{$(\mathbf{m} \cup m) \in \gridmcr$} \Comment{skip out-of-space}
          \For{$\dhstrut{0pt}{10pt} b_m \in \mathcal{A}^{(m)}$}
            \State $\dhstrut{8pt}{14pt} w^{(\mathbf{m} \cup m)}_{(b_m | \mathbf{a})} \mathrel{{+}{=}} A^{(m)}_{b_m 0} \, v^{(\mathbf{m})}_{\mathbf{a}}$ \Comment{up}
          \EndFor
          \EndIf
        \EndIf
      \EndFor
      \EndFor
      \end{algorithmic}
    \end{algorithm}
    \end{minipage}
\end{center}
Algorithm~\ref{alg:chopped_one_mode_contraction_subtensor} suggests the following partitioning
of the one-mode matrix (see also Example~\ref{example:preorder_block_triangular}):
\begin{align}
  \mathbf{A}^{\!(m)} =
  \left[
  \begin{array}{c|c}
    p^{(m)}          & \mathbf{d}^{(m)} \dhstrut{6pt}{10pt} \\ \hline
    \mathbf{u}^{(m)} & \mathbf{F}^{(m)} \dhstrut{0pt}{12pt}
  \end{array}
  \right].
\end{align}
The vectors $\mathbf{d}^{(m)}$ and $\mathbf{u}^{(m)}$
contain the \textit{down} and \textit{up} elements ($A^{(m)}_{0 b_m}$ and $A^{(m)}_{a_m 0}$), respectively,
while the matrix $\mathbf{F}^{(m)}$ contains the \textit{forward} elements ($A^{(m)}_{a_m b_m}$).
The passive element is $p^{(m)} = A^{(m)}_{00}$.
With this notation we see that
\begin{enumerate}[label=(\roman*), nosep]
  \item The \textit{down} contraction is a one-index contraction between the vector $\mathbf{d}^{(m)}$
  and the tensor $\mathbf{v}^{(\mathbf{m})}$.
  \item The \textit{forward} contraction is a one-index contraction between the matrix $\mathbf{F}^{(m)}$
  and the tensor $\mathbf{v}^{(\mathbf{m})}$.
  \item The \textit{passive} contraction is a scalar multiplication between the scalar $p^{(m)}$ and the tensor $\mathbf{v}^{(\mathbf{m})}$.
  \item The \textit{up} contraction is an outer product between the vector $\mathbf{u}^{(m)}$
  and the tensor $\mathbf{v}^{(\mathbf{m})}$.
\end{enumerate}
A one-index contraction in mode $m$ is also known as an $m$-mode product \citep{koldaTensorDecompositionsApplications2009}.
Assuming $n_1 = \cdots = n_D = n$ for simplicity and letting $k = |\mathbf{m}|$, the complexities are as follows:
\begin{enumerate}[label=(\roman*), nosep]
  \item Down: $\mathcal{O}(n^{k})$
  \item Forward: $\mathcal{O}(n^{k+1})$
  \item Passive: $\mathcal{O}(n^{k})$
  \item Up: $\mathcal{O}(n^{k+1})$
\end{enumerate}
Since the contraction index $m$ is clear from context, 
we simply use the symbol $\otimes$ to denote the vector-tensor outer product
and the symbol $\times$ to denote the $m$-mode product.
With this notation, we can restate Algorithm~\ref{alg:chopped_one_mode_contraction_subtensor} as
Algorithm~\ref{alg:chopped_one_mode_contraction_subtensor_final}, which is generic in the sense
that all elements of $\mathbf{A}^{(m)}$ and $\hatbold{v}$ may be non-zero.
However, it is obvious that the algorithm can be simplified if certain blocks of $\mathbf{A}^{(m)}$
are zero (Appendix~\ref{appendix:complexity_Lv_Uv})
or if certain subtensors in $\hatbold{v}$ are zero (Appendix~\ref{appendix:complexity_Mv}).
\begin{center}
  \begin{minipage}{0.8\linewidth}
    \begin{algorithm}[H]
      \caption{Compute $\hatbold{w} = \hatboldsuper{A}{[m]} \hatbold{v}$ using Theorem~\ref{thm:chopped_one_mode_contraction}.} \label{alg:chopped_one_mode_contraction_subtensor_final}
      \begin{algorithmic}
      \For{$\mathbf{m} \in \gridmcr$}
        \If{$m \in \mathbf{m}$}
          \State $\dhstrut{6pt}{14pt} \mathbf{w}^{(\mathbf{m}\setminus m)} \mathrel{{+}{=}} \mathbf{d}^{(m)} \times \mathbf{v}^{(\mathbf{m})}$ \Comment{down}
          \State $\dhstrut{6pt}{12pt} \mathbf{w}^{(\mathbf{m}           )} \mathrel{{+}{=}} \mathbf{F}^{(m)} \times \mathbf{v}^{(\mathbf{m})}$ \Comment{forward}
        \Else
          \State $\dhstrut{7pt}{13pt} \mathbf{w}^{(\mathbf{m}           )} \mathrel{{+}{=}}        {p}^{(m)}        \mathbf{v}^{(\mathbf{m})}$ \Comment{passive}
          \If{$(\mathbf{m} \cup m) \in \gridmcr$} \Comment{skip out-of-space}
            \State $\dhstrut{6pt}{14pt} \mathbf{w}^{(\mathbf{m} \cup m)} \mathrel{{+}{=}} \mathbf{u}^{(m)} \otimes \mathbf{v}^{(\mathbf{m})}$ \Comment{up}
          \EndIf
        \EndIf
      \EndFor
      \end{algorithmic}
    \end{algorithm}
    \end{minipage}
\end{center}


\newpage
\section{Kernel matrix-vector products} \label{appendix:kernel_mvp}
\renewcommand{\theequation}{\thesection\arabic{equation}}


\subsection{Deriving the kernel MVP} \label{appendix:deriving_kernel_mvp}

Using \eqref{eq:J_decomposition} and the properties of the bracket matrices, we rewrite
\eqref{eq:sop_kernel_complete_grid} as
\begin{align}
  \mathbf{K}
    &= \sum_{\mathbf{m}} \sigma_{|\mathbf{m}|}^2 
    \mathbf{K}^{[\mathbf{m}]} \mathbf{J}^{[\neg \mathbf{m}]} \nonumber \\
    &= \sum_{\mathbf{m}} \sigma_{|\mathbf{m}|}^2 
    \mathbf{K}^{[\mathbf{m}]} \mathbf{L}^{[\neg \mathbf{m}]} \mathbf{R}^{[\neg \mathbf{m}]} \mathbf{U}^{[\neg \mathbf{m}]} \nonumber \\
    &= 
      \sum_{\mathbf{m}} \sigma_{|\mathbf{m}|}^2 
      \mathbf{L}^{[\neg \mathbf{m}]}  \left[ \mathbf{L}^{[\mathbf{m}]} \big(\mathbf{L}^{[\mathbf{m}]} \big)^{-1} \right]
      \mathbf{K}^{[\mathbf{m}]} \mathbf{R}^{[\neg \mathbf{m}]} 
      \left[ \big(\mathbf{U}^{[\mathbf{m}]} \big)^{-1} \mathbf{U}^{[\mathbf{m}]} \right] \mathbf{U}^{[\neg \mathbf{m}]}
    \nonumber \\
    &= \mathbf{L}
    \left[ 
      \sum_{\mathbf{m}} \sigma_{|\mathbf{m}|}^2 
      \big(\mathbf{L}^{[\mathbf{m}]} \big)^{-1} 
      \mathbf{K}^{[\mathbf{m}]}
      \big(\mathbf{U}^{[\mathbf{m}]} \big)^{-1}
      \mathbf{R}^{[\neg \mathbf{m}]} 
    \right]
    \mathbf{U} \nonumber \\
    &=
    \mathbf{L}
    \left[ 
      \sum_{\mathbf{m}} \sigma_{|\mathbf{m}|}^2 
      \mathbf{M}^{[\mathbf{m}]}
      \mathbf{R}^{[\neg \mathbf{m}]} 
    \right]
    \mathbf{U}
    \nonumber \\
    &=
    \mathbf{L}
    \mathbf{M}
    \mathbf{U}.
\end{align}
At the second equality we use the factorization of $\mathbf{J}^{(m)}$ \eqref{eq:J_decomposition},
at the third equality we insert identities (in square brackets), and at the fourth equality
we observe
\begin{alignat}{7}
  \mathbf{L}^{[\neg \mathbf{m}]} \mathbf{L}^{[\mathbf{m}]} 
  &= \mathbf{L}^{[1]} &&\cdots &&\;\mathbf{L}^{[D]}
  &&= \mathbf{L}^{(1)} &&\otimes \cdots \otimes &&\;\mathbf{L}^{(D)}
  &&\equiv \mathbf{L}, \\
  \mathbf{U}^{[\mathbf{m}]} \mathbf{U}^{[\neg \mathbf{m}]} 
  &= \mathbf{U}^{[1]} &&\cdots &&\;\mathbf{U}^{[D]}
  &&= \mathbf{U}^{(1)} &&\otimes \cdots \otimes &&\;\mathbf{U}^{(D)}
  &&\equiv \mathbf{U}.
\end{alignat}
At the last two lines of \eqref{eq:sop_kernel_complete_grid_rewritten}, we simply make the definitions
\begin{align}
  \mathbf{M}^{(m)} &= \big(\mathbf{L}^{(m)} \big)^{-1} \mathbf{K}^{(m)}
  \big(\mathbf{U}^{(m)} \big)^{-1}, \\
  \mathbf{M} &= \left[ 
    \sum_{\mathbf{m}} \sigma_{|\mathbf{m}|}^2 
    \mathbf{M}^{[\mathbf{m}]}
    \mathbf{R}^{[\neg \mathbf{m}]}
  \right].
\end{align}


\subsection{Complexity of MVPs with \texorpdfstring{$\hatbold{L}$}{\textbf{L}} and \texorpdfstring{$\hatbold{U}$}{\textbf{U}}} \label{appendix:complexity_Lv_Uv}

In order to determine the complexity of the \ac{mvp}
\begin{align} \label{eq:mvp_U}
  \hatbold{U} \hatbold{v} = \big( \hatboldsuper{U}{[1]} \hatboldsuper{U}{[1]} \cdots \hatboldsuper{U}{[D]}  \big) \hatbold{v}
\end{align}
we start by considering a single one-mode contraction $\hatboldsuper{U}{[m]} \hatbold{v}$.
The underlying 1D matrix $\mathbf{U}^{m}$ can be decomposed as
\begin{align}
  \mathbf{U}^{m} =
  \begin{bmatrix} 
    1      & 1      & \dots  & 1     \\
    0      & 1      & \dots  & 0     \\
    \vdots & \vdots & \ddots & \vdots\\
    0      & 0      & \dots  & 1
  \end{bmatrix}
  =
  \begin{bmatrix} 
    1      & 0      & \dots  & 0     \\
    0      & 1      & \dots  & 0     \\
    \vdots & \vdots & \ddots & \vdots\\
    0      & 0      & \dots  & 1
  \end{bmatrix}
  +
  \begin{bmatrix} 
    0      & 1      & \dots  & 1     \\
    0      & 0      & \dots  & 0     \\
    \vdots & \vdots & \ddots & \vdots\\
    0      & 0      & \dots  & 0
  \end{bmatrix}
  = \mathbf{I}^{(m)} + \mathbf{S}^{(m)}
\end{align}
which implies
\begin{align} \label{eq:Um_contraction}
  \hatboldsuper{U}{[m]} \hatbold{v} = \big( \hatboldsuper{I}{[m]} + \hatboldsuper{S}{[m]} \big) \hatbold{v} = \hatbold{v} + \hatboldsuper{S}{[m]} \hatbold{v}.
\end{align}
The only non-zero block of the matrix $\mathbf{S}^{(m)}$ is the down-block,
so the second term in \eqref{eq:Um_contraction} only involves down-contractions in mode $m$.
This means that only \acp{mc} containing $m$ are touched.

For definiteness, we consider an \ac{mcr}, $\gridmcr$, defined by a maximum cut level, $\alpha$.
The \ac{mcr} is dominated by the $\alpha$-cuts, which will thus determine the complexity. The number of $\alpha$-cuts
containing the given mode $m$ is
\begin{align}
  \binom{D - 1}{\alpha - 1} \sim \mathcal{O} \! \left( \frac{(D - 1)^{\alpha - 1}}{(\alpha - 1)!} \right) = \mathcal{O} \! \left( \frac{D^{\alpha - 1}}{(\alpha - 1)!} \right).
\end{align}
Assuming $n_1 = \cdots = n_D = n$ for simplicity, each down-contraction has an operation count
of $n^{\alpha}$, so the complexity of $\hatboldsuper{U}{[m]} \hatbold{v}$ is 
\begin{align}
  \mathcal{O} \! \left( \frac{D^{\alpha - 1} n^{\alpha}}{(\alpha - 1)!} \right).
\end{align}
Equation \eqref{eq:mvp_U} involves $D$ \acp{mvp} with this complexity, so the overall
cost of $\hatbold{U} \hatbold{v}$ is
\begin{align}
  \mathcal{O} \! \left( \frac{D^{\alpha} n^{\alpha}}{(\alpha - 1)!} \right) = \mathcal{O} ( \alpha \widehat{N} ).
\end{align}

The complexity of $\hatbold{L} \hatbold{v}$ can be established in much the same way. To start with, we note that
\begin{align}
  \mathbf{L}^{m} =
  \begin{bmatrix} 
    1      & 0      & \dots  & 0     \\
    1      & 1      & \dots  & 0     \\
    \vdots & \vdots & \ddots & \vdots\\
    1      & 0      & \dots  & 1
  \end{bmatrix}
  =
  \begin{bmatrix} 
    1      & 0      & \dots  & 0     \\
    0      & 1      & \dots  & 0     \\
    \vdots & \vdots & \ddots & \vdots\\
    0      & 0      & \dots  & 1
  \end{bmatrix}
  +
  \begin{bmatrix} 
    0      & 0      & \dots  & 0     \\
    1      & 0      & \dots  & 0     \\
    \vdots & \vdots & \ddots & \vdots\\
    1      & 0      & \dots  & 0
  \end{bmatrix}
  = \mathbf{I}^{(m)} + \mathbf{D}^{(m)}
\end{align}
which implies
\begin{align} \label{eq:Lm_contraction}
  \hatboldsuper{L}{[m]} \hatbold{v} = \big( \hatboldsuper{I}{[m]} + \hatboldsuper{D}{[m]} \big) \hatbold{v} = \hatbold{v} + \hatboldsuper{D}{[m]} \hatbold{v}.
\end{align}
This time, the second term involves up-contractions in mode $m$, so only \acp{mc} excluding $m$ are touched.
Due to the chopping we only need
to consider those up-contractions that remain in $\gridmcr$. Doing an up-contraction on
an $\alpha$-cut will for example produce a result of order $\alpha + 1$, which is outside $\gridmcr$.
The dominating up-contractions are thus those that touch \acp{mc} of order $\alpha - 1$ that do not contain the given mode $m$.
The number of such \acp{mc} is
\begin{align}
  \binom{D - 1}{\alpha - 1} \sim \mathcal{O} \! \left( \frac{(D - 1)^{\alpha - 1}}{(\alpha - 1)!} \right) = \mathcal{O} \! \left( \frac{D^{\alpha - 1}}{(\alpha - 1)!} \right)
\end{align}
while the cost of each up-contraction is $n^{\alpha}$. In conclusion, we find that $\hatbold{L} \hatbold{v}$ has the same complexity
as $\hatbold{U} \hatbold{v}$, namely $\mathcal{O} ( \alpha \widehat{N} )$.


\subsection{Complexity of MVPs with \texorpdfstring{$\hatbold{M}$}{\textbf{M}}} \label{appendix:complexity_Mv}

We will now address the complexity of the \ac{mvp} with the matrix
\begin{align}
  \hatbold{M} = 
  \sum_{\mathbf{m}} \sigma_{|\mathbf{m}|}^2 
    \widehat{\mathbf{M}^{[\mathbf{m}]}}
    \widehat{\mathbf{R}^{[\neg \mathbf{m}]}}.
\end{align}
For simplicity, we will start by considering $\hatboldsuper{R}{[\neg \mathbf{m}]} \hatbold{v}$
and then $\hatboldsuper{M}{[\mathbf{m}]} \hatboldsuper{R}{[\neg \mathbf{m}]} \hatbold{v}$
before summing up the terms.
The key component is the matrix $\widehat{\mathbf{R}^{[\neg \mathbf{m}]}}$,
which projects onto the subspace where all modes $m \notin \mathbf{m}$ are in the reference
(equivalently, this is the subspace where only the modes $m \in \mathbf{m}$ are allowed to be
displaced from the reference). This subspace can generally be characterized by a set of multi-indices
$\hatcal{I}(\mathbf{m}) \subseteq \hatcal{I}$, but for our purposes it is more useful
to consider the corresponding \ac{mcr}, $\mathcal{M}(\mathbf{m}) \subseteq \gridmcr$.
Due to the action of the projector (as described above), we must have
\begin{align} \label{eq:subspace_mcr}
  \mathcal{M}(\mathbf{m}) = \{ \mathbf{m}' \,\vert\, \mathbf{m}' \subseteq \mathbf{m}, \,  \mathbf{m}' \in \gridmcr\}.
\end{align}
In our implementation we currently assume $\kernelmcr \subseteq \gridmcr$,
so we always have $\mathbf{m} \in \gridmcr$. But $\gridmcr$ is \ac{cuts},
so in this case \eqref{eq:subspace_mcr} simplifies to
\begin{align} \label{eq:subspace_mcr_simplified}
  \mathcal{M}(\mathbf{m}) = \{ \mathbf{m}' \,\vert\, \mathbf{m}' \subseteq \mathbf{m} \} = \mathrm{powerset}(\mathbf{m}),
\end{align}
which has size $|\mathcal{M}(\mathbf{m})| = 2^{|\mathbf{m}|}$. 

In addition to being a projector, $\widehat{\mathbf{R}^{[\neg \mathbf{m}]}}$ commutes with $\widehat{\mathbf{M}^{[\mathbf{m}]}}$
so
\begin{align}
  \hatboldsuper{M}{[\mathbf{m}]} \hatboldsuper{R}{[\neg \mathbf{m}]}
  = \hatboldsuper{R}{[\neg \mathbf{m}]} \hatboldsuper{M}{[\mathbf{m}]} \hatboldsuper{R}{[\neg \mathbf{m}]},
\end{align}
which shows that the action of $\hatboldsuper{M}{[\mathbf{m}]}$ does not produce any terms outside the subspace
$\mathcal{M}(\mathbf{m})$. We conclude that each term in $\hatbold{M}$ touches only a small part of the vector,
which is in fact the key to the performance of the \ac{mvp}.
The actual multiplication by $\hatboldsuper{M}{[\mathbf{m}]}$ can be carried out using 
the LU-based machinery (Theorem~\ref{thm:chopping} and Corollary~\ref{cor:chopped_kronecker_product}) with the only modification that the code should only
access the sub-vector defined by $\mathcal{M}(\mathbf{m})$, rather than the full vector.
If $\mathbf{m} \in \gridmcr$ (as in our current implementation),
the subspace $\mathcal{M}(\mathbf{m})$ can be viewed as a complete grid over the modes $\mathbf{m}$.
This means that $\hatboldsuper{M}{[\mathbf{m}]}$ can be carried out as a sequence of ordinary
one-mode contractions (no LU decomposition is necessary).

To derive a definite complexity, we consider identical kernel and grid \acp{mcr} ($\kernelmcr = \gridmcr$)
defined by a maximum interaction order or cut level, $\alpha$.
The \ac{mvp} is dominated by the highest-order kernel terms (the $\alpha$-terms), and each such term
is dominated by the forward contractions. 
Consider an $\alpha$-term indexed by $\mathbf{m}$. Every non-empty subspace \ac{mc}, $\mathbf{m}' \in \mathcal{M}(\mathbf{m})$,
can be forward-contracted $k = |\mathbf{m}'|$ times, and each of those forward contraction has an operation count of $(n - 1)^{k + 1}$
(assuming $n_1 = \cdots = n_D = n$). For a single $\alpha$-term, this yields an operation count of
\begin{align}
  \sum_{k=1}^{\alpha} \binom{\alpha}{k} k (n - 1)^{k + 1} 
  = \alpha n^{\alpha - 1} (n - 1)^2
  < \alpha n^{\alpha + 1}.
\end{align}
The first equality can be proved by differentiating the binomial theorem, 
$n^{\alpha} = \sum_{k = 0}^{\alpha} \binom{\alpha}{k} (n - 1)^k$, with respect to $n$
and multiplying the result by $(n - 1)^{2}$ on both sides.
The number of $\alpha$-terms is
\begin{align}
  \binom{D}{\alpha} \sim \mathcal{O} \! \left( \frac{D^{\alpha}}{\alpha !} \right),
\end{align}
so the overall complexity of the \ac{mvp} by $\hatbold{M}$ is
\begin{align}
  \mathcal{O} \! \left( n \alpha \frac{D^{\alpha} n^\alpha}{\alpha !} \right) = \mathcal{O} ( n \alpha \widehat{N} ).
\end{align}


\subsection{Complexity of MVPs with \texorpdfstring{$\partial \hatbold{M} / \partial \ell^{(m)}$}{∂\textbf{M}/∂\textit{l}}} \label{appendix:complexity_dMdl_v}

The length scale derivative,
\begin{align} \label{eq:appendix_dMdl_def}
  \pdv{\hatbold{M}}{ \ell^{(m)} } = \sum_{\substack{ \mathbf{m} \in \kernelmcr,\\ m \in \mathbf{m} }} \sigma^2_{|\mathbf{m}|} \pdv{ \hatboldsuper{M}{[\mathbf{m}]} }{ \ell^{(m)} } \hatboldsuper{R}{[\neg\mathbf{m}]},
\end{align}
has the same form as $\hatbold{M}$ itself, except that the sum is restricted. 
This means that each of the dominating $\alpha$-terms in \eqref{eq:appendix_dMdl_def} has the same cost as an $\alpha$-term in $\hatbold{M}$,
namely $\mathcal{O}(\alpha n^{\alpha + 1})$ (see Appendix~\ref{appendix:complexity_Mv}).
The number of $\alpha$-terms in \eqref{eq:appendix_dMdl_def} is exactly the number $\alpha$-terms containing the given mode $m$:
\begin{align} \label{eq:alpha_cuts_containing_m}
  \binom{D - 1}{\alpha - 1} < \binom{D}{\alpha - 1} \sim \mathcal{O}\!\left( \frac{D^{\alpha - 1}}{(\alpha - 1)!} \right).
\end{align}
We thus find that the complexity of \eqref{eq:appendix_dMdl_def} is
\begin{align}
    \mathcal{O}\!\left( \alpha n^{\alpha + 1} \frac{D^{\alpha - 1}}{(\alpha - 1)!} \right)
  = \mathcal{O}\!\left( \frac{n \alpha^2}{D} \frac{D^{\alpha} n^{\alpha}}{\alpha!} \right)
  = \mathcal{O}\!\left( \frac{n \alpha^2}{D} \widehat{N} \right).
\end{align}
Each term in the derivative \ac{mvp} populates the subspace defined by $\mathcal{M}(\mathbf{m})$ (see Appendix~\ref{appendix:complexity_Mv}).
In total, all $\alpha$-cuts containing $m$ as well as all lower-order cuts are populated. 
Counting the dominating $\alpha$-cuts and using \eqref{eq:alpha_cuts_containing_m}, we find that the number of
non-zero elements of $\hatbold{w} = (\partial \hatbold{M} / \partial \ell^{(m)}) \hatbold{v}$ goes like
\begin{align}
  \mathcal{O}\!\left( n^\alpha \frac{D^{\alpha - 1}}{(\alpha - 1)!} \right) = 
  \mathcal{O}\!\left( \frac{\alpha}{D}  \widehat{N} \right).
\end{align}
Appendix~\ref{appendix:K_column} shows that an \ac{mvp} of the form $\hatbold{M} \hatbold{U} \hatbold{e}_{\mathbf{i}}$
can be computed directly (without forming the intermediate vector $\hatbold{U} \hatbold{e}_{\mathbf{i}}$) at a 
cost of $\mathcal{O}(\alpha \widehat{N})$. This should be compared to 
the generic \ac{mvp} $\hatbold{M} \hatbold{U} \hatbold{v}$ with complexity $\mathcal{O}(n \alpha \widehat{N})$. 
The same reduction by a factor $n$ applies to the derivative \ac{mvp}, so 
$(\partial \hatbold{M} / \partial \ell^{(m)}) \hatbold{U} \hatbold{e}_{\mathbf{i}}$ can be obtained
with complexity
\begin{align}
  \mathcal{O}\!\left( \frac{\alpha^2}{D} \widehat{N} \right).
\end{align}


\subsection{Reducing complexity in the quadratic term}  \label{appendix:quadratic_term}

In order to maintain a $D$-complexity of $\mathcal{O}(D^{\alpha})$, we will use the factorized form of $\hatbold{K}$ 
\eqref{eq:sop_kernel_complete_grid_chopped}:
\begin{align} \label{eq:length scale_quadratic_term}
  \weights^{\trans} \pdv{\hatbold{K}}{\ell^{(m)}} \weights
  = \big( \hatbold{U} \weights \big)^{\trans} \pdv{\hatbold{M}}{\ell^{(m)}}  \big( \hatbold{U} \weights \big).
\end{align}
Appendix~\ref{appendix:complexity_dMdl_v} shows that a derivative \ac{mvp} 
like $\hatbold{w} = (\partial \hatbold{M} / \partial \ell^{(m)}) \hatbold{v}$ 
has a complexity of only $\mathcal{O}(n \alpha^2 \widehat{N} / D)$.
In addition, only a small fraction of $\hatbold{w}$ is non-zero, leading to $\mathcal{O}(\alpha \widehat{N} / D)$ dot products with $\hatbold{w}$.
Computing \eqref{eq:length scale_quadratic_term} for all $D$ length scales can thus be done at a total cost of $\mathcal{O}(n \alpha^2 \widehat{N}) \sim \mathcal{O}(D^{\alpha})$.

The same approach can be used for the bilinear forms that occur in the stochastic estimate of $\trace ( \mathbf{C}^{-1} \partial \mathbf{C} / \partial \theta )$
and the deterministic computation of preconditioner trace terms; see
\eqref{eq:estimate_trace_term} and \eqref{eq:precond_trace_term_via_mvp}, respectively.


\newpage
\section{Computing a single column of \texorpdfstring{$\hatbold{K}$}{\textbf{K}}}  \label{appendix:K_column}
\renewcommand{\theequation}{\thesection\arabic{equation}}

In order to construct the pivoted Cholesky preconditioner (see Appendix~\ref{appendix:cg}), we must be
able to access a single column of the kernel matrix.
For a given multi-index $\mathbf{i} \in \hatcal{I}$, we can compute the corresponding column
of $\hatbold{K}$ by doing an \ac{mvp} with the standard basis vector $\hatbold{e}_{\mathbf{i}} \in \mathbb{R}^{\widehat{N}}$:
\begin{align} \label{eq:single_column_using_mvp}
  \hatbold{k}_{\mathbf{i}} 
  = \hatbold{K} \hatbold{e}_{\mathbf{i}}.
\end{align}
As described in Section~\ref{sec:fast_mvps} this has a complexity 
of $\mathcal{O}(n \alpha \widehat{N})$ (see also Appendices~\ref{appendix:complexity_Lv_Uv}
and \ref{appendix:complexity_Mv}).
In the following we show that this can be improved to $\mathcal{O}(\alpha \widehat{N})$.
We start by applying Theorem~\ref{thm:chopping} to $\hatbold{K}$
after which we introduce explicit chopping matrices:
\begin{align} \label{eq:single_column_using_structure}
  \hatbold{k}_{\mathbf{i}} 
  = \hatbold{L} \reallywidehat{\mathbf{M} \mathbf{U}} \hatbold{e}_{\mathbf{i}}
  = \hatbold{L} \big( \mathbf{\Gamma}^{\trans} \mathbf{M} \mathbf{U} \mathbf{\Gamma} \big) \big( \mathbf{\Gamma}^{\trans} \mathbf{e}_{\mathbf{i}} \big)
  = \hatbold{L} \mathbf{\Gamma}^{\trans} \mathbf{M} \mathbf{U} \mathbf{e}_{\mathbf{i}}.
\end{align}
The vector $\mathbf{e}_{\mathbf{i}}$ is contained in the subspace corresponding to $\hatcal{I}$
so $\mathbf{\Gamma} \mathbf{\Gamma}^{\trans} \mathbf{e}_{\mathbf{i}} = \mathbf{e}_{\mathbf{i}}$
even though $\mathbf{\Gamma} \mathbf{\Gamma}^{\trans}$ is generally not the identity.
The first part of the \ac{mvp} is now expressed without chopping,
which is usually not desirable (indeed, the chopping framework was introduced to avoid
exactly this). It is useful, however, in this particular case since $\mathbf{e}_{\mathbf{i}}$
can be viewed as a rank-one tensor.
To be specific, we make the association $\mathbf{i} \cong (\mathbf{m}, \mathbf{a})$
(see Definition~\ref{def:incomplete_grid}) and write 
\begin{align} \label{eq:ei_def}
  \mathbf{e}_{\mathbf{i}} = \bigotimes_{m=1}^{D} \mathbf{v}^{(m)}, \quad
  \mathbf{v}^{(m)} =
  \begin{dcases}[c]
    \mathbf{r}^{(m)} & \text{if } m \notin \mathbf{m}, \\
    \mathbf{a}^{(m)} & \text{if } m \in    \mathbf{m}.
  \end{dcases}
\end{align}
The vector $\mathbf{v}^{(m)} \in \mathbb{R}^{n_m}$ 
has element $i_m$ equal to one and the remaining elements equal to zero.
The symbols $\mathbf{r}^{(m)}$ and
$\mathbf{a}^{(m)}$ only serve to visually separate the cases $m \notin \mathbf{m}$ ($i_m = 0$)
and $m \in \mathbf{m}$ ($i_m = a_m > 0$).
To further simplify notation we reorder the factors of the Kronecker product, which
corresponds to a permutation of the elements of $\mathbf{e}_{\mathbf{i}}$. To
compensate, we multiply by the inverse permutation, $\mathbf{P}$, and rewrite \eqref{eq:ei_def} as
\begin{align} \label{eq:ei_reordered}
  \mathbf{e}_{\mathbf{i}} = \mathbf{P} \bigg[
  \bigg( \bigotimes_{m    \in \mathbf{m}} \mathbf{a}^{(m)} \bigg)
  \otimes
  \bigg( \bigotimes_{m \notin \mathbf{m}} \mathbf{r}^{(m)} \bigg) \bigg].
\end{align}
We stress that the only purpose of this reordering is to make the equations more readable.
No permutation is necessary in the actual computation.

Before proceeding we observe that 
\begin{subequations}
  \begin{align}
    \mathbf{U}^{(m)} \mathbf{r}^{(m)} &= \mathbf{r}^{(m)}, \label{eq:Ur}\\
    \mathbf{U}^{(m)} \mathbf{a}^{(m)} &= \mathbf{r}^{(m)} + \mathbf{a}^{(m)}, \label{eq:Ua}\\
    \mathbf{R}^{(m)} \mathbf{r}^{(m)} &= \mathbf{r}^{(m)}, \label{eq:Rr}\\
    \mathbf{R}^{(m)} \big( \mathbf{r}^{(m)} + \mathbf{a}^{(m)} \big) &= \mathbf{r}^{(m)}. \label{eq:Rr_plus_a}
  \end{align}
\end{subequations}
Using \eqref{eq:Ur} and \eqref{eq:Ua}, we find
\begin{align}
  \mathbf{U} \mathbf{e}_{\mathbf{i}}
  &=
  \bigg( \bigotimes_{m=1}^{D} \mathbf{U}^{(m)} \bigg)
  \mathbf{P} \bigg[
  \bigg( \bigotimes_{m    \in \mathbf{m}} \mathbf{a}^{(m)} \bigg)
  \otimes
  \bigg( \bigotimes_{m \notin \mathbf{m}} \mathbf{r}^{(m)} \bigg) \bigg] \nonumber \\
  &=
  \mathbf{P} \bigg[
  \bigg( \bigotimes_{m    \in \mathbf{m}} \mathbf{U}^{(m)} \mathbf{a}^{(m)} \bigg)
  \otimes
  \bigg( \bigotimes_{m \notin \mathbf{m}} \mathbf{U}^{(m)} \mathbf{r}^{(m)} \bigg) \bigg] \nonumber \\
  &=
  \mathbf{P} \bigg[
  \bigg( \bigotimes_{m    \in \mathbf{m}} \big( \mathbf{r}^{(m)} + \mathbf{a}^{(m)} \big) \bigg)
  \otimes
  \bigg( \bigotimes_{m \notin \mathbf{m}} \mathbf{r}^{(m)} \bigg) \bigg].
\end{align}
Substituting this into \eqref{eq:single_column_using_structure}
and applying \eqref{eq:Rr} and \eqref{eq:Rr_plus_a} then yields
\begin{align} \label{eq:single_column_expanded}
  \hatbold{k}_{\mathbf{i}} 
  &= \hatbold{L} \mathbf{\Gamma}^{\trans} \mathbf{M} \mathbf{U} \mathbf{e}_{\mathbf{i}} \nonumber \\
  &= \hatbold{L} \mathbf{\Gamma}^{\trans}
  \Bigg(
  \sum_{\mathbf{m}'} \sigma^2_{|\mathbf{m}'|} 
  \mathbf{M}^{[\mathbf{m}']} \mathbf{R}^{[\neg \mathbf{m}']}
  \mathbf{P} \bigg[
  \bigg( \bigotimes_{m    \in \mathbf{m}} \big( \mathbf{r}^{(m)} + \mathbf{a}^{(m)} \big) \bigg)
  \otimes
  \bigg( \bigotimes_{m \notin \mathbf{m}} \mathbf{r}^{(m)} \bigg) \bigg] \Bigg) \nonumber \\
  &= 
  \begin{multlined}[t]
    \hatbold{L} \mathbf{\Gamma}^{\trans}
    \Bigg(
    \sum_{\mathbf{m}'} \sigma^2_{|\mathbf{m}'|} 
    \mathbf{M}^{[\mathbf{m}']} 
    \mathbf{P}' \bigg[
    \bigg( \bigotimes_{m \in \mathbf{m}' \cap \mathbf{m}} \big( \mathbf{r}^{(m)} + \mathbf{a}^{(m)} \big) \bigg) \\
    {}\mspace{200mu} \otimes
    \bigg( \bigotimes_{m \in \mathbf{m}' \setminus \mathbf{m}} \mathbf{r}^{(m)} \bigg)
    \otimes
    \bigg( \bigotimes_{m \notin \mathbf{m}'} \mathbf{r}^{(m)} \bigg) \bigg] 
    \Bigg)
  \end{multlined} \nonumber \\
  &= 
  \begin{multlined}[t]
    \hatbold{L} \mathbf{\Gamma}^{\trans}
    \Bigg(
    \sum_{\mathbf{m}'} \sigma^2_{|\mathbf{m}'|} 
    \mathbf{P}' \bigg[
    \bigg( \bigotimes_{m \in \mathbf{m}' \cap \mathbf{m}} \mathbf{M}^{(m)} \big( \mathbf{r}^{(m)} + \mathbf{a}^{(m)} \big) \bigg) \\
    {}\mspace{200mu} \otimes
    \bigg( \bigotimes_{m \in \mathbf{m}' \setminus \mathbf{m}} \mathbf{M}^{(m)} \mathbf{r}^{(m)} \bigg)
    \otimes
    \bigg( \bigotimes_{m \notin \mathbf{m}'} \mathbf{r}^{(m)} \bigg) \bigg] 
    \Bigg) 
  \end{multlined} \nonumber \\
  &=
  \begin{multlined}[t]
  \hatbold{L} \mathbf{\Gamma}^{\trans}
    \Bigg(
    \sum_{\mathbf{m}'} \sigma^2_{|\mathbf{m}'|} 
    \mathbf{P}' \bigg[
    \bigg( \bigotimes_{m \in \mathbf{m}' \cap \mathbf{m}} \big( \mathbf{M}^{(m)}_{:0} + \mathbf{M}^{(m)}_{:a^m} \big) \bigg) \\
    {}\mspace{200mu} \otimes
    \bigg( \bigotimes_{m \in \mathbf{m}' \setminus \mathbf{m}} \mathbf{M}^{(m)}_{:0} \bigg)
    \otimes
    \bigg( \bigotimes_{m \notin \mathbf{m}'} \mathbf{r}^{(m)} \bigg) \bigg] 
    \Bigg).
  \end{multlined}
\end{align}
The order of the Kronecker factors may change at the third equality, corresponding to
a new permutation $\mathbf{P}'$. The notation $\mathbf{M}^{(m)}_{:j}$ means the $j$th column
of $\mathbf{M}^{(m)}$. In practice, the computation proceeds as follows: For each kernel term
$\mathbf{m}' \in \kernelmcr$, a Kronecker product of 
matrix columns is formed according to \eqref{eq:single_column_expanded}.
The Kronecker product is then partitioned according to \acp{mc} and added
to an intermediary vector. If $\mathbf{m}' \notin \gridmcr$
then the Kronecker product may generate out-of-space \acp{mc}, which should
be skipped during computation due to the presence of the chopping matrix $\mathbf{\Gamma}^{\trans}$. 
Finally, $\hatbold{L}$ is applied to the intermediary vector to obtain the result.

To derive a definite complexity, we consider identical kernel and grid 
\acp{mcr} ($\kernelmcr = \gridmcr$)
defined by a maximum interaction order or cut level, $\alpha$.
We also assume, for simplicity, that $n_1 = \cdots = n_d = n$.
Each of the dominating terms in \eqref{eq:single_column_expanded} consists of 
a Kronecker product of $\alpha$ vectors of size $n$, which has a 
cost of no more than $\alpha n^\alpha$. The number of $\alpha$-terms is
\begin{align}
  \binom{D}{\alpha} \sim \mathcal{O}\!\left( \frac{D^\alpha}{\alpha!} \right)
\end{align}
so the overall complexity of computing the sum is
\begin{align}
  \mathcal{O}\!\left( \alpha \frac{D^\alpha n^\alpha}{\alpha!} \right) =
  \mathcal{O} \big( \alpha \widehat{N}  \big),
\end{align}
which is also the complexity of the subsequent \ac{mvp} by $\hatbold{L}$ 
(see Appendix~\ref{appendix:complexity_Lv_Uv}). We conclude that a single column of $\hatbold{K}$
can be computed at a cost of $\mathcal{O} \big( \alpha \widehat{N} \big)$.


\newpage
\section{Computing the diagonal of \texorpdfstring{$\hatbold{K}$}{\textbf{K}}} \label{appendix:K_diagonal}
\renewcommand{\theequation}{\thesection\arabic{equation}}

The pivoted Cholesky algorithm (see Appendix~\ref{appendix:cg}) requires the main diagonal
of the kernel matrix.
We start by recalling the complete-grid
kernel as written in \eqref{eq:sop_kernel_complete_grid}:
\begin{align}
  \mathbf{K}
    = \sum_{\mathbf{m} \in \kernelmcr} \sigma_{|\mathbf{m}|}^2 \mathbf{K}^{[\mathbf{m}]} \mathbf{J}^{[\neg \mathbf{m}]}.
\end{align}
The diagonal of $\hatbold{K}$ consists of the 
elements $\widehat{K}_{\mathbf{i} \mathbf{i}} = K_{\mathbf{i} \mathbf{i}}$ 
for $\mathbf{i} \in \hatcal{I}$.
It is convenient to rewrite each multi-index as $\mathbf{i} \cong (\mathbf{m}, \mathbf{a})$
(see Definition~\ref{def:incomplete_grid}), in which case
\begin{align} \label{eq:Kii_start}
  K_{\mathbf{i} \mathbf{i}} 
  &= \left[ \sum_{\mathbf{m}'} \sigma_{|\mathbf{m}'|}^2
  \mathbf{K}^{[\mathbf{m}']} \mathbf{J}^{[\neg \mathbf{m}']} \right]_{\mathbf{i} \mathbf{i}} \nonumber \\
  &= \sum_{\mathbf{m}'} \sigma_{|\mathbf{m}'|}^2
  \Bigg( \prod_{m \in    \mathbf{m}' \cap      \mathbf{m}} K^{(m)}_{a^m a^m} \Bigg)
  \Bigg( \prod_{m \in    \mathbf{m}' \setminus \mathbf{m}} K^{(m)}_{00}      \Bigg) \nonumber \\
  &= \sum_{\mathbf{m}'} \sigma_{|\mathbf{m}'|}^2
  \Bigg( \prod_{m \in    \mathbf{m}' \cap      \mathbf{m}} \frac{K^{(m)}_{a^m a^m}}{K^{(m)}_{00}} \Bigg)
  \Bigg( \prod_{m \in    \mathbf{m}'} K^{(m)}_{00} \Bigg).
\end{align}
We have assumed $K^{(m)}_{00} \neq 0$. 
The intersection $\mathbf{s} = \mathbf{m}' \cap \mathbf{m}$ of 
$\mathbf{m}' \in \kernelmcr$ and
$\mathbf{m}  \in \gridmcr$ plays a central role in the following.
The crucial idea is to group the terms in \eqref{eq:Kii_start} according to $\mathbf{s}$.
We start by noting that the possible intersections are exactly the subsets of $\mathbf{m}$,
i.e. $\mathbf{s} \subseteq \mathbf{m}$. Next, we define
the set of kernel \acp{mc} with a prescribed intersection as
\begin{align}
  \mathcal{S(\mathbf{m}, \mathbf{s})} 
  = \{ \mathbf{m}' \in \kernelmcr \;|\; \mathbf{m}' \cap \mathbf{m} = \mathbf{s} \}.
\end{align}
With this notation,
\begin{align} \label{eq:Kii_rewrite}
  K_{\mathbf{i} \mathbf{i}} 
  &=
  \sum_{\mathbf{s} \subseteq \mathbf{m}}
  \sum_{\mathbf{m}' \in \mathcal{S}(\mathbf{m}, \mathbf{s})}
  \sigma_{|\mathbf{m}'|}^2
  \Bigg( \prod_{m \in \mathbf{s}} \frac{K^{(m)}_{a^m a^m}}{K^{(m)}_{00}} \Bigg)
  \Bigg( \prod_{m \in    \mathbf{m}'} K^{(m)}_{00} \Bigg) \nonumber \\ 
  &=
  \sum_{\mathbf{s} \subseteq \mathbf{m}}
  \Bigg( \prod_{m \in \mathbf{s}} \frac{K^{(m)}_{a^m a^m}}{K^{(m)}_{00}} \Bigg)
  \sum_{\mathbf{m}' \in \mathcal{S}(\mathbf{m}, \mathbf{s})}
  \sigma_{|\mathbf{m}'|}^2
  \Bigg( \prod_{m \in    \mathbf{m}'} K^{(m)}_{00} \Bigg).
\end{align}
The outer sum is manageable given that $|\mathbf{m}|$ is small, while the inner sum involves only products of reference base kernel elements, $K^{(m)}_{00}$.
As it stands, however, the inner sum must be computed for every combination of $\mathbf{m}$ and $\mathbf{s}$, which would be costly.

To get started, we define some notation for the inner sum:
\begin{gather}
  X^{(\mathbf{m}, \mathbf{s})} = \sum_{\mathbf{m}' \in \mathcal{S}(\mathbf{m}, \mathbf{s})} P^{(\mathbf{m}')}, \\
  P^{(\mathbf{m}')} = \sigma_{|\mathbf{m}'|}^2
  \Bigg( \prod_{m \in \mathbf{m}'} K^{(m)}_{00} \Bigg).
\end{gather}
The computational problem is the simultaneous dependency on $\mathbf{m}$ and $\mathbf{s}$.
In order to eventually remove this problem, we let
\begin{align}
  \mathcal{S(\mathbf{s})} 
  = \{ \mathbf{m}' \in \kernelmcr \;|\; \mathbf{s} \subseteq \mathbf{m}' \},
\end{align}
which is the set of all kernel \acp{mc} containing $\mathbf{s}$. 
The idea is now to group the elements of $\mathcal{S(\mathbf{s})}$ in a convenient way.
To that end, choose a fixed $\mathbf{m} \in \mathcal{M}_{\textrm{grid}}$ such that $\mathbf{s} \subseteq \mathbf{m}$
and observe that
\begin{align}
  \mathbf{s} \subseteq \mathbf{m}, \, \mathbf{s} \subseteq \mathbf{m}' \;\Rightarrow\; \mathbf{s} \subseteq \mathbf{m}' \cap \mathbf{m}.
\end{align}
On the other hand, $\mathbf{m}' \cap \mathbf{m} \subseteq \mathbf{m}$ so
\begin{align}
  \mathbf{s} \subseteq (\mathbf{m}' \cap \mathbf{m}) \subseteq \mathbf{m}.
\end{align}
We conclude that the intersection $\mathbf{s}' = \mathbf{m}' \cap \mathbf{m}$ is at least $\mathbf{s}$ and at most $\mathbf{m}$.
This implies the following partitioning of $\mathcal{S}(\mathbf{s})$ for $\mathbf{s} \subseteq \mathbf{m}$:
\begin{align}
  \mathcal{S}(\mathbf{s}) 
  = \{ \mathbf{m}' \in \kernelmcr \;|\; \mathbf{s} \subseteq \mathbf{m}' \}
  = \bigcup_{\mathbf{s} \subseteq \mathbf{s}' \subseteq \mathbf{m}}  \{ \mathbf{m}' \in \kernelmcr \;|\; \mathbf{m}' \cap \mathbf{m} = \mathbf{s}' \}
  = \bigcup_{\mathbf{s} \subseteq \mathbf{s}' \subseteq \mathbf{m}} \mathcal{S}(\mathbf{m}, \mathbf{s}').
\end{align}
In turn, we find that
\begin{align} \label{eq:Ys_from_Xms}
  Y^{(\mathbf{s})} 
  \equiv \sum_{\mathbf{m}' \in \mathcal{S}(\mathbf{s})} P^{(\mathbf{m}')}
  = \sum_{\mathbf{s} \subseteq \mathbf{s}' \subseteq \mathbf{m}} \left( \sum_{\mathbf{m}' \in \mathcal{S}(\mathbf{m}, \mathbf{s}')} P^{(\mathbf{m}')} \right)
  = \sum_{\mathbf{s} \subseteq \mathbf{s}' \subseteq \mathbf{m}} X^{(\mathbf{m}, \mathbf{s}')},
\end{align}
which is in fact a triangular system of equations. Importantly, $Y^{(\mathbf{s})}$ does not depend on $\mathbf{m}$
and can thus be precomputed once before computing the entire diagonal.

To see the triangular structure, we examine a concrete example: 
Let $\mathbf{m} = (1,2,3)$ and consider \eqref{eq:Ys_from_Xms} for all $\mathbf{s} \subseteq \mathbf{m}$.
Omitting commas between integers and writing $X^{(\mathbf{s}')}$ for $X^{(\mathbf{m}, \mathbf{s}')}$, we find that
\begin{align}
  \begin{bmatrix}
    Y^{(123)} \\
    Y^{(23)}  \\
    Y^{(13)}  \\
    Y^{(12)}  \\
    Y^{(3)}   \\
    Y^{(2)}   \\
    Y^{(1)}   \\
    Y^{()}
  \end{bmatrix}
  =
  \begin{bmatrix}
    1 & 0 & 0 & 0 & 0 & 0 & 0 & 0 \\
    1 & 1 & 0 & 0 & 0 & 0 & 0 & 0 \\
    1 & 0 & 1 & 0 & 0 & 0 & 0 & 0 \\
    1 & 0 & 0 & 1 & 0 & 0 & 0 & 0 \\
    1 & 1 & 1 & 0 & 1 & 0 & 0 & 0 \\
    1 & 1 & 0 & 1 & 0 & 1 & 0 & 0 \\
    1 & 0 & 1 & 1 & 0 & 0 & 1 & 0 \\
    1 & 1 & 1 & 1 & 1 & 1 & 1 & 1 
  \end{bmatrix}
  \begin{bmatrix}
    X^{(123)} \\
    X^{(23)}  \\
    X^{(13)}  \\
    X^{(12)}  \\
    X^{(3)}   \\
    X^{(2)}   \\
    X^{(1)}   \\
    X^{()}
  \end{bmatrix}.
\end{align}
The dimension of the system is $2^{|\mathbf{m}|}$, so it can be solved using ordinary forward substitution with complexity $\mathcal{O}(4^{|\mathbf{m}|})$.
It is, however, possible to derive an $\mathcal{O}(3^{|\mathbf{m}|})$ algorithm by taking advantage of the special structure. We start by slightly rearranging \eqref{eq:Ys_from_Xms}:
\begin{align} \label{eq:Xms_from_Ys}
  X^{(\mathbf{m}, \mathbf{s})} =  Y^{(\mathbf{s})} - \sum_{\mathbf{s} \subset \mathbf{s}' \subseteq \mathbf{m}} X^{(\mathbf{m}, \mathbf{s}')}.
\end{align}
This suggests the following algorithm:
\begin{center}
  \begin{minipage}{0.7\linewidth}
    \begin{algorithm}[H]
      \caption{Given $\mathbf{m} \in \gridmcr$, compute $X^{(\mathbf{m}, \mathbf{s})}$ for all $\mathbf{s} \subseteq \mathbf{m}$.}\label{alg:Xms}
      \begin{algorithmic}
      \For{$\mathbf{s} \subseteq \mathbf{m}$}
        \State $X^{(\mathbf{m}, \mathbf{s})} = Y^{(\mathbf{s})}$
      \EndFor
      \For{$\mathbf{s} \subseteq \mathbf{m}$} \Comment{Loop in descending order}
        \For{$\mathbf{s}' \subset \mathbf{s}$}
          \State $X^{(\mathbf{m}, \mathbf{s}')} \mathrel{{-}{=}} X^{(\mathbf{s})}$
        \EndFor
      \EndFor
      \end{algorithmic}
    \end{algorithm}
    \end{minipage}
\end{center}
Using the binomial theorem, the number of loop iterations in the nested (dominating) loop is
\begin{align}
  \sum_{k=0}^{|\mathbf{m}|} \binom{|\mathbf{m}|}{k} \sum_{l=0}^{k - 1} \binom{k}{l}
  = \sum_{k=0}^{|\mathbf{m}|} \binom{|\mathbf{m}|}{k} \left( -1 + \sum_{l=0}^{k} \binom{k}{l} \right)
  = 3^{|\mathbf{m}|} - 2^{|\mathbf{m}|},
\end{align}
leading to a complexity of $\mathcal{O}(3^{|\mathbf{m}|})$.

The following algorithm describes how to precompute $Y^{(\mathbf{s})}$ for all $\mathbf{s} \in \gridmcr$:
\begin{center}
  \begin{minipage}{0.7\linewidth}
    \begin{algorithm}[H]
      \caption{Compute $Y^{(\mathbf{s})}$ for all $\mathbf{s} \in \gridmcr$. Assumes $\kernelmcr \subseteq \gridmcr$.} \label{alg:Ys}
      \begin{algorithmic}
      \For{$\mathbf{s} \in \gridmcr$}
        \State $Y^{(\mathbf{s})} = 0$
      \EndFor
      \For{$\mathbf{s} \in \kernelmcr$}
        \State $P = \sigma^2_{|\mathbf{s}|}$
        \For{$m \in \mathbf{s}$}
          \State $P \mathrel{{*}{=}} K^{(m)}_{00}$
        \EndFor
        \For{$\mathbf{s}' \subseteq \mathbf{s}$}
          \State $Y^{(\mathbf{s}')} \mathrel{{+}{=}} P$
        \EndFor
      \EndFor
      \end{algorithmic}
    \end{algorithm}
    \end{minipage}
\end{center}
Assuming simple kernel and grid \acp{mcr} 
($\kernelmcr = \gridmcr$)
with a maximum interaction order or cut level, $\alpha$, we see that
the nested loop over $\mathbf{s}$ and $\mathbf{s}'$ dominates with a
complexity of
\begin{align}
  \mathcal{O}\!\left( \frac{D^\alpha 2^\alpha}{\alpha!} \right).
\end{align}
We now have the tools to compute the entire diagonal:
\begin{center}
  \begin{minipage}{0.7\linewidth}
    \begin{algorithm}[H]
      \caption{Compute $K_{\mathbf{i} \mathbf{i}}$ for all $\mathbf{i} \in \hatcal{I}$.
      Equivalently, compute $K_{(\mathbf{m}, \mathbf{a})(\mathbf{m}, \mathbf{a})}$ for all $\mathbf{m} \in \gridmcr$ 
      and $\mathbf{a} \in \mathcal{A}^{(\mathbf{m})}$ (see Definition~\ref{def:subgrid}).
      Assumes $\kernelmcr \subseteq \gridmcr$.} \label{alg:Kdiag}
      \begin{algorithmic}
      \State Compute $Y^{(\mathbf{s})}$ for all $\mathbf{s} \in \gridmcr$ using Algorithm~\ref{alg:Ys}.
      \For{$\mathbf{m} \in \gridmcr$}
        \State Compute $X^{(\mathbf{m}, \mathbf{s})}$ for all $\mathbf{s} \subseteq \mathbf{m}$ using Algorithm~\ref{alg:Xms}.
        \For{$\mathbf{s} \subseteq \mathbf{m}$}
          \For{$\mathbf{a} \in \mathcal{A}^{(\mathbf{m})}$}
            \State $K_{(\mathbf{m}, \mathbf{a})(\mathbf{m}, \mathbf{a})} \mathrel{{+}{=}} \left( \displaystyle \prod_{m \in \mathbf{s}} \frac{K^{m}_{a^m a^m}}{K^{(m)}_{00}} \right) X^{(\mathbf{m}, \mathbf{s})}$
          \EndFor
        \EndFor
      \EndFor
      \end{algorithmic}
    \end{algorithm}
    \end{minipage}
\end{center}
Taking each loop into account (and assuming $n_1 = \cdots = n_d = n$ for simplicity),
we find a complexity of
\begin{align}
  \mathcal{O}\!\left( \alpha 2^\alpha \frac{D^\alpha n^\alpha}{\alpha!} \right)
  = \mathcal{O}\big( \alpha 2^\alpha \widehat{N} \big).
\end{align}
This should be compared to the \ac{mvp} by $\hatbold{K}$, which has a cost of $\mathcal{O}\big( n \alpha \widehat{N} \big)$.
The $\alpha$-scaling related to the diagonal is clearly more severe, but at the same time we expect $\alpha$ to be
relatively small. In practice, we thus anticipate that $2^\alpha$ is comparable to $n$, leading to roughly similar costs for the kernel diagonal and the kernel \ac{mvp}.
It should also be recalled that the diagonal is computed only once in each optimization cycle (when the preconditioner is constructed),
whereas the kernel \ac{mvp} is computed a large number times.


\newpage
\section{The preconditioner trace term} \label{appendix:precond_trace_term}
\renewcommand{\theequation}{\thesection\arabic{equation}}

In order to reduce the variance of the trace term estimate \citep{wengerPreconditioningScalableGaussian2022}
we must be able to compute the preconditioner trace term, $\trace ( \mathbf{P}^{-1} \partial \mathbf{P} / \partial \theta )$,
given that $\mathbf{P}$ is a pivoted Cholesky preconditioner of rank $k$:
\begin{align}
  \mathbf{P} = 
  \mathbf{Z} \mathbf{Z}^{\trans} + \sigma^2 \mathbf{I}, \quad \mathbf{Z} \in \mathbb{R}^{\widehat{N} \times k}.
\end{align}
As a working example we will consider the length scale trace terms, i.e. $\trace ( \mathbf{P}^{-1} \partial \mathbf{P} / \partial \ell^{(m)} )$
for $m = 1, \ldots, D$. In practice, we may not use an independent length scale for each dimension, so this is a worst-case example
that we should nevertheless be able to handle.

The simplest option requires an explicit representation of the derivative of the pivoted Cholesky decomposition.
Considering any (non-noise) hyperparameter $\theta$, we have
\begin{align}
  \frac{\partial \mathbf{P}}{\partial \theta} 
  = \frac{\partial \mathbf{Z}}{\partial \theta} \mathbf{Z}^{\trans} + \mathbf{Z} \frac{\partial \mathbf{Z}^{\trans}}{\partial \theta}
  = \mathbf{X} \mathbf{Z}^{\trans} + \mathbf{Z} \mathbf{X}^{\trans} , \quad \mathbf{X} \in \mathbb{R}^{\widehat{N} \times k}.
\end{align}
The derivative Cholesky vectors $\mathbf{X}$ can be obtained by manual \citep{aquilanteAnalyticDerivativesCholesky2008, fengImplementationAnalyticGradients2019} 
or automatic \citep{smithDifferentiationCholeskyAlgorithm1995} differentiation of the pivoted Cholesky algorithm, but storage alone costs $\mathcal{O}( k p \widehat{N} )$ where $p$
is the number of hyperparameters. For the length scales in our working example, this translates to $\mathcal{O}( k D \widehat{N} )$ or 
$\mathcal{O}( k D^{\alpha + 1} n^\alpha / \alpha! )$ if the kernel and grid \acp{mcr} are simple.
To avoid the $D$-scaling of $\mathcal{O}(D^{\alpha + 1})$, which does not occur elsewhere in our code,
we will use the connection between the pivoted Cholesky factorization and the column Nyström approximation \citep[see, e.g.,][Property 2.1]{chenRandomlyPivotedCholesky2025}.
Using MATLAB notation we have
\begin{align} \label{eq:nystrom}
  \mathbf{Z} \mathbf{Z}^{\trans} = \hatbold{K}(:,\mathcal{P}) \hatbold{K}(\mathcal{P},\mathcal{P})^{-1} \hatbold{K}(:,\mathcal{P})^{\trans}.
\end{align}
where $\mathcal{P}$ denotes the set of pivot indices as determined by the pivoted Cholesky algorithm.
The pivot columns and the pivot submatrix can be written as
\begin{alignat}{5}
  \mathbf{A} &= \hatbold{K}(\;:        &&,\mathcal{P}) &&= &                              \hatbold{K} \mathbf{\Delta} &\in \mathbb{R}^{\widehat{N} \times k}, \\
  \mathbf{B} &= \hatbold{K}(\mathcal{P}&&,\mathcal{P}) &&= &\; \mathbf{\Delta}^{\!\trans} \hatbold{K} \mathbf{\Delta} &\in \mathbb{R}^{k           \times k},
\end{alignat}
where $\mathbf{\Delta} \in \mathbb{R}^{\widehat{N} \times k}$ is a chopping matrix that selects the pivot indices,
i.e. the columns of $\mathbf{\Delta}$ are simply standard basis vectors $\hatbold{e}_{\mathbf{i}}$ with $\mathbf{i} \in \mathcal{P} \subseteq \hatcal{I}$.
The expressions involving chopping will be useful later on, since they allow the derivatives of 
$\mathbf{A}$ and $\mathbf{B}$ to be computed via fast \acp{mvp}.
The pivot columns, $\mathbf{A}$, are already computed by the pivoted Cholesky algorithm, and the pivot submatrix
$\mathbf{B}$ is easily obtained from $\mathbf{A}$ by selecting the appropriate rows. 
Taking the derivative of \eqref{eq:nystrom} with respect to a non-noise hyperparameter $\theta$, we obtain
\begin{align}
  \pdv{\mathbf{P}}{\theta}
  = \pdv{\mathbf{Z} \mathbf{Z}^{\trans}}{\theta}
  = \pdv{\mathbf{A} \mathbf{B}^{-1} \mathbf{A}^{\trans}}{\theta}
  &= \pdv{\mathbf{A}}{\theta} \mathbf{B}^{-1} \mathbf{A}^{\trans}
  +  \mathbf{A} \pdv{\mathbf{B}^{-1}}{\theta} \mathbf{A}^{\trans}
  +  \mathbf{A} \mathbf{B}^{-1} \pdv{\mathbf{A}^{\trans}}{\theta} \nonumber \\
  &= \pdv{\mathbf{A}}{\theta} \mathbf{B}^{-1} \mathbf{A}^{\trans}
  -  \mathbf{A} \mathbf{B}^{-1} \pdv{\mathbf{B}}{\theta} \mathbf{B}^{-1} \mathbf{A}^{\trans}
  +  \mathbf{A} \mathbf{B}^{-1} \pdv{\mathbf{A}^{\trans}}{\theta}. \label{eq:dPdt_nystrom}
\end{align}
The Nyström format \eqref{eq:nystrom} can also be used to express $\mathbf{P}^{-1}$ via the matrix inverse lemma,
which would remove the need to store the Cholesky vectors, $\mathbf{Z}$.
We have found, however, that the Cholesky format produces more accurate solves, and
we will thus use Cholesky for $\mathbf{P}^{-1}$ and Nyström for $\partial \mathbf{P} / \partial \theta$
when computing the trace terms.
Before proceeding we recall \eqref{eq:Pinv_woodbury}:
\begin{align}
  \mathbf{P}^{-1} = \frac{1}{\sigma^2}
  \left( \mathbf{I}_{\widehat{N}} - \mathbf{Z} \mathbf{Q}^{-1} \mathbf{Z}^{\trans} \right). \label{eq:Pinv_cholesky}
\end{align}
It is now a simple matter to combine \eqref{eq:dPdt_nystrom} and \eqref{eq:Pinv_cholesky}.
Using the invariance of the trace under transposition and cyclic shifts, one finds the expression
\begin{align} \label{eq:precond_trace_term_rewritten}
  \sigma^2 \trace\left( \mathbf{P}^{-1} \pdv{\mathbf{P}}{\theta} \right) 
  = 2 \trace\left( \mathbf{X}^{\trans} \pdv{\mathbf{A}}{\theta} \right)
  +   \trace\left( \mathbf{Y}^{\trans} \pdv{\mathbf{B}}{\theta} \right)
\end{align}
with the definitions
\begin{subequations}
  \begin{gather}
    \mathbf{X}^{\trans} = \big( \mathbf{B}^{-1} \mathbf{A}^{\trans} \big) 
    - \Big( \big(\mathbf{B}^{-1} \mathbf{A}^{\trans}\big) \mathbf{Z} \Big) \big(\mathbf{Q}^{-1} \mathbf{Z}^{\trans}\big)
    \in \mathbf{R}^{k \times \widehat{N}}, \\
    \mathbf{Y}^{\trans} = \Big( \big(\mathbf{B}^{-1} \mathbf{A}^{\trans}\big) \mathbf{Z} \Big) \Big( \big(\mathbf{Q}^{-1} \mathbf{Z}^{\trans}\big) \big(\mathbf{B}^{-1} \mathbf{A}^{\trans}\big)^{\!\trans} \Big) 
    - \big(\mathbf{B}^{-1} \mathbf{A}^{\trans}\big) \big(\mathbf{B}^{-1} \mathbf{A}^{\trans}\big)^{\!\trans}
    \in \mathbf{R}^{k \times k}.
  \end{gather}
\end{subequations}
Parentheses indicate the order of evaluation used in our code (expressions that occur multiple times should of course be computed once and then reused).
Computing $\mathbf{X}$ and $\mathbf{Y}$ from $\mathbf{B}$, $\mathbf{Q}$, $\mathbf{A}$ and $\mathbf{Z}$
has a complexity of $\mathcal{O}(k^2 \widehat{N})$.

We still need to handle the derivatives. 
Using $\mathbf{A} = \hatbold{K} \mathbf{\Delta} = \hatbold{L} \hatbold{M} \hatbold{U} \mathbf{\Delta}$
and $\mathbf{B} = \mathbf{\Delta}^{\!\trans} \hatbold{K} \mathbf{\Delta} = \mathbf{\Delta}^{\!\trans} \hatbold{L} \hatbold{M} \hatbold{U} \mathbf{\Delta}$,
we can write the matrices on the right-hand side of \eqref{eq:precond_trace_term_rewritten} as
\begin{subequations} \label{eq:precond_trace_term_via_mvp}
  \begin{gather}
    \mathbf{X}^{\trans} \pdv{\mathbf{A}}{\theta}
     = \mathbf{X}^{\trans} \hatbold{L} \pdv{\hatbold{M}}{\theta} \hatbold{U} \mathbf{\Delta} 
     = \big( \hatbold{U} \mathbf{X} \big)^{\!\trans} \bigg( \pdv{ \hatbold{M}}{\theta} \hatbold{U} \mathbf{\Delta} \bigg), \\
    \mathbf{Y}^{\trans} \pdv{\mathbf{B}}{\theta}
     = \mathbf{Y}^{\trans} \mathbf{\Delta}^{\!\trans} \hatbold{L} \pdv{\hatbold{M}}{\theta} \hatbold{U} \mathbf{\Delta} 
     = \big( \hatbold{U} (\mathbf{\Delta} \mathbf{Y}) \big)^{\!\trans} \bigg( \pdv{ \hatbold{M}}{\theta} \hatbold{U} \mathbf{\Delta} \bigg).
  \end{gather}
\end{subequations}
The cost of $\mathbf{\Delta} \mathbf{Y}$ is only $\mathcal{O}(k\widehat{N})$ due to the sparsity of $\mathbf{\Delta}$.
The matrices $\hatbold{U} \mathbf{X}$ and $\hatbold{U} (\mathbf{\Delta} \mathbf{Y})$ can each be obtained
at a cost of $\mathcal{O}(k \alpha \widehat{N})$ by computing fast \acp{mvp} with $\hatbold{U}$ (see Appendix~\ref{appendix:complexity_Lv_Uv}).
The columns of $\mathbf{\Delta}$ are simply standard basis vectors,
so we can use the techniques of Appendices~\ref{appendix:K_column} and \ref{appendix:complexity_dMdl_v} to compute the sparse columns of 
$(\partial \hatbold{M} / \partial \theta) \hatbold{U} \mathbf{\Delta} \in \mathbb{R}^{\widehat{N} \times k}$ directly.
The diagonals of $\mathbf{X}^{\trans} (\partial \mathbf{A} / \partial \theta)$ and $\mathbf{Y}^{\trans} (\partial \mathbf{B} / \partial \theta)$,
which are sufficient for computing the traces, require only $k$ sparse dot products each.

We have now established a precomputation phase where $\mathbf{X}$, $\mathbf{Y}$, $\hatbold{U} \mathbf{X}$ and $\hatbold{U} (\mathbf{\Delta} \mathbf{Y})$
are computed with leading-order complexity $\mathcal{O}(k^2 \widehat{N})$ (assuming $k > \alpha$, which is almost always the case in practice).
This phase is followed by computing, for each $\theta$, the $k$ columns of $(\partial \hatbold{M} / \partial \theta) \hatbold{U} \mathbf{\Delta}$
as well as $k$ sparse dot products.

Returning now to our working example, we use the fact that each column of $(\partial \hatbold{M} / \partial \ell^{(m)}) \hatbold{U} \mathbf{\Delta}$
has a cost of $\mathcal{O}(\alpha^2 \widehat{N} / D)$, while each sparse dot product scales as $\mathcal{O}(\alpha \widehat{N} / D)$ (see Appendix~\ref{appendix:complexity_dMdl_v}).
Computing the traces for all $m = 1, \ldots, D$ thus has a total cost of $\mathcal{O}(k \alpha^2 \widehat{N})$.


\newpage
\section{Overall structure of the code} \label{appendix:overal_structure}
\renewcommand{\theequation}{\thesection\arabic{equation}}

Having covered all components necessary to do \ac{mvp}-based \ac{gpr},
a summary is in place.
The computation proceeds in two
main phases: Hyperparameter optimization ($N_{\mathrm{opt}}$ cycles) and prediction ($\widehat{N}_{*}$ points).
At the beginning of each optimization cycle, the preconditioner is constructed and a number of
preconditioner-related quantities are precomputed. This is followed by estimation of the \ac{mll},
\begin{align} \label{gridgpr:eq:mll}
  \mathcal{L} 
  = - \frac{1}{2} \left( \hatbold{y}^{\trans} \widehat{\weights} + \log{|\hatbold{C}|} + \widehat{N} \log{2 \pi} \right),
\end{align}
and the \ac{mll} gradient,
\begin{align} \label{gridgpr:eq:mll_gradient}
  \pdv{\mathcal{L}}{\theta} 
  = \frac{1}{2} \widehat{\weights}^{\trans} \pdv{\hatbold{C}}{\theta} \widehat{\weights}
  - \frac{1}{2} \trace\left( \hatbold{C}^{-1} \pdv{\hatbold{C}}{\theta} \right). 
\end{align}
The dominating part of Eqs.~\eqref{gridgpr:eq:mll} and \eqref{gridgpr:eq:mll_gradient}
is the modified \ac{cg} algorithm needed for estimating the log-determinant and the trace terms.
Each \ac{cg} iteration spends a kernel \ac{mvp}, an inverse preconditioner \ac{mvp} 
and a few dot products (which we will ignore when analyzing the cost).

The predictive mean and variance for a single test point is
\begin{gather}
  \mu = \hatbold{k}_{*}^{\trans} \widehat{\weights} \label{gridgpr:eq:mean} \\
  \Sigma = K_{**} - \hatbold{k}_{*}^{\trans} \hatbold{C}^{-1} \hatbold{k}_{*} \label{gridgpr:eq:variance}
\end{gather}
where $\hatbold{k}_{*} \in \mathbb{R}^{\widehat{N}}$ is a column 
of the training--test cross-covariance matrix
and the scalar $K_{**}$ is the test--test covariance. 
Again, the dominating
part is the \ac{cg} algorithm required for the matrix inverse.
The training--test column can be computed by a slightly modified version of
the algorithm in Appendix~\ref{appendix:K_column} at a cost of $\mathcal{O}(\alpha \widehat{N})$.

We use the symbols $N_{\mathrm{opt}}$, $N_{\mathrm{probe}}$ and $N_{\mathrm{CG}}$
for the number of optimization cycles, probe vectors and \ac{cg} iterations.
The preconditioner
rank is $k$, and it is assumed that $k < \widehat{N}$, $\alpha < k$ and $N_{\mathrm{CG}} < \widehat{N}$.
In addition, we only consider the \ac{mll} gradient with respect to the length scales (the order variance gradients are not
dominating).
With these assumptions in place, we are ready to summarize the various steps in more detail. 
The summary is not fully exhaustive -- it is only intended to cover the parts that
are conceptually or computationally significant.
\begin{itemize}[label=$\bullet$]
  \item Phase 1: Hyperparameter optimization ($N_{\mathrm{opt}}$ cycles)
  \begin{itemize}[label=$\circ$]
    \item Preconditioner
    \begin{itemize}
      \item Construct $\mathbf{P}$ and precompute preconditioner-related quantities: $\mathcal{O}(k^2 \widehat{N})$.
      \item Compute the log-determinant of $\mathbf{P}$: $\mathcal{O}(k)$.
      \item Prepare the preconditioner trace terms: $\mathcal{O}(k^2 \widehat{N})$
      \item Compute the preconditioner trace terms: $\mathcal{O}(k \alpha^2 \widehat{N})$
    \end{itemize}
    \item Estimates of \ac{mll} and \ac{mll} gradient
    \begin{itemize}
      \item Quadratic terms: $\mathcal{O}(n \alpha^2 \widehat{N})$.
      \item Stochastic trace estimation: $\mathcal{O}\big(N_{\mathrm{probe}} N_{\mathrm{CG}} (n \alpha \widehat{N} + k \widehat{N} ) \big)$
    \end{itemize}
  \end{itemize}
  \item Phase 2: Prediction ($\widehat{N}_{*}$ points)
  \begin{itemize}[label=$\circ$]
    \item Compute training--test column: $\mathcal{O}(\alpha \widehat{N})$.
    \item Compute predictive mean: $\mathcal{O}(\widehat{N})$.
    \item Compute predictive variance: $\mathcal{O}\big( N_{\mathrm{CG}} (n \alpha \widehat{N} + k \widehat{N} ) \big)$.
  \end{itemize}
\end{itemize}
Looking at the summary, we see that the most expensive parts are
the stochastic trace estimation and the predictive variances
with complexities
\begin{align}
  \mathcal{O}\big(N_{\mathrm{opt}} N_{\mathrm{probe}} N_{\mathrm{CG}} (n \alpha \widehat{N} + k \widehat{N} ) \big) \quad\text{and}\quad
  \mathcal{O}\big( \widehat{N}_{*} N_{\mathrm{CG}} (n \alpha \widehat{N} + k \widehat{N} ) \big), \nonumber
\end{align}
respectively.
Ideally, the parameters $N_{\mathrm{opt}}$, $N_{\mathrm{probe}}$, $N_{\mathrm{CG}}$ and $k$ do not depend on $\widehat{N}$,
in which case the complexities just mentioned essentially reduce to the complexity of the kernel \ac{mvp}, i.e.
$\mathcal{O}(n \alpha \widehat{N})$. Whether this happens in practice depends chiefly 
on the conditioning of the kernel, which does not depend on $\widehat{N}$ as such but rather
on the nature of the dataset and the value of the hyperparameters. Noise plays a particularly
important role in this regard, since larger noise improves the conditioning of the kernel.

Irrespective of the exact scaling, the \ac{mvp}-based framework has the great advantage
that it is easily parallelized: Trace estimation can be parallelized over probe vectors
and predictions over test points.


\newpage
\section{Complexity of the kernel MVP} \label{appendix:mvp_scaling}
\renewcommand{\theequation}{\thesection\arabic{equation}}
\renewcommand{\thefigure}{\thesection\arabic{figure}}
\renewcommand{\thetable}{\thesection\arabic{table}}

\subsection{Setup}

All benchmarks were run using standard double precision on a single core of an 
Intel Xeon Platinum 8358 processor. 
The benchmarks were run using Google Benchmark,\footnote{Version 1.9.4, Apache License 2.0, available at \url{https://github.com/google/benchmark}.}
which dynamically determines an appropriate number of iterations to ensure stable timings.
The minimum warm-up time and minimum running time were set to \SI{2}{\second}.
The fitted timings are shown as a function of $D$ in Figures~\ref{fig:gridgpr_benchmark_alpha2}--\ref{fig:gridgpr_benchmark_alpha4}
and as a function of $\widehat{N}$ in Figures~\ref{fig:gridgpr_scaling_alpha2_n5}--\ref{fig:gridgpr_scaling_alpha4_n20}.
The raw timings are available in Tables~\ref{tab:gridgpr_benchmark_alpha2}--\ref{tab:gridgpr_benchmark_alpha4}.

\subsection{Empirical $D$-scaling}

We mention in Section~\ref{sec:computational_complexity} that the
empirical $D$-complexities are slightly too large, in particular for $\alpha = 4$,
which we take as a sign that the 
asymptotic region is not yet fully reached.
This hypothesis has been investigated for each combination of $\alpha$ and $n$ by fitting 
either (i) the lower half of the data, (ii) all data, or (iii) the upper half of the data (when the number of data points
is odd, we include the middle point in the lower half). The exponents from the three fits are shown in Table~\ref{tab:gridgpr_benchmark_refitted_exponents}.
For $\alpha = 3,4$ we see a steady decrease of the exponents toward the theoretical values, thus supporting the hypothesis.
The same behavior is seen for $\alpha = 2$ and $n = 20$.
For $\alpha = 2$ and $n = 10, 5$ there is no clear trend, but the deviations from the theoretical exponents are, in any case, small.
We conclude that the fits corroborate the theoretical complexities of the \ac{mvp} algorithm.

\begin{table}[H]
  \centering
  \caption{Empirical $D$-complexities for $\alpha = 2, 3, 4$ and $n = 5,10,20$.}
  \begin{tabular}{
  S[table-format=1.0  , table-text-alignment=right] 
  S[table-format=2.0  , table-text-alignment=right] 
  S[table-format=1.2  , table-text-alignment=center, round-mode=places, round-precision=2] 
  S[table-format=1.2  , table-text-alignment=center, round-mode=places, round-precision=2] 
  S[table-format=1.2  , table-text-alignment=center, round-mode=places, round-precision=2] 
  }
\toprule
{$\alpha$} & {$n$}  &   {Lower half} &  {All points} & {Upper half} \\
\midrule
2          & 20     &   2.1058       &  2.05974      & 1.99578 \\
           & 10     &   2.0671       &  2.0766       & 2.01532 \\
           &  5     &   2.02946      &  2.05663      & 2.05141 \\
\midrule
3          & 20     &   3.16226      &  3.09385      & 3.04931 \\
           & 10     &   3.2951       &  3.15099      & 3.02858 \\
           &  5     &   3.16562      &  3.13234      & 3.09668 \\
\midrule
4          & 20     &   4.34845      &  4.29005      & 4.24685 \\
           & 10     &   4.31299      &  4.21393      & 4.13051 \\
           &  5     &   4.4217       &  4.23507      & 4.07809 \\
\bottomrule
\end{tabular}
\label{tab:gridgpr_benchmark_refitted_exponents}%
\end{table}%

\afterpage{\clearpage}
\begin{figure}[p]
  \centering
  \includegraphics[width=0.5\linewidth]{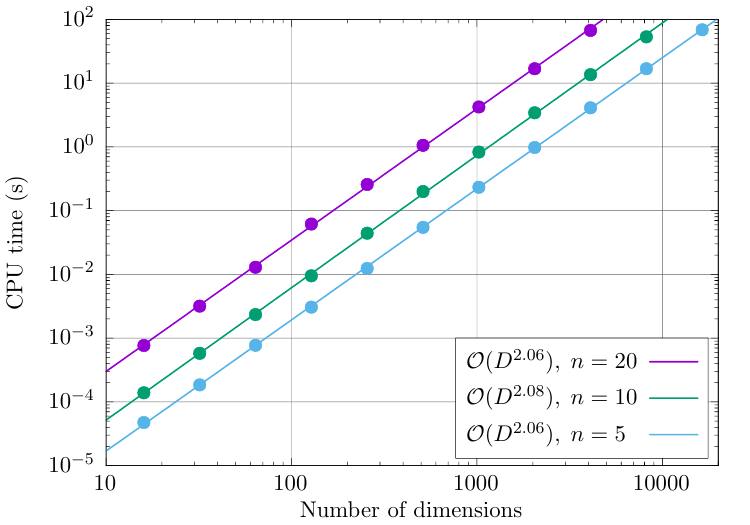}
  \caption{CPU times for the kernel \ac{mvp} with $\alpha = 2$ and $n = 5,10,20$.}
  \label{fig:gridgpr_benchmark_alpha2}
\end{figure}

\begin{figure}[p]
  \centering
  \includegraphics[width=0.5\linewidth]{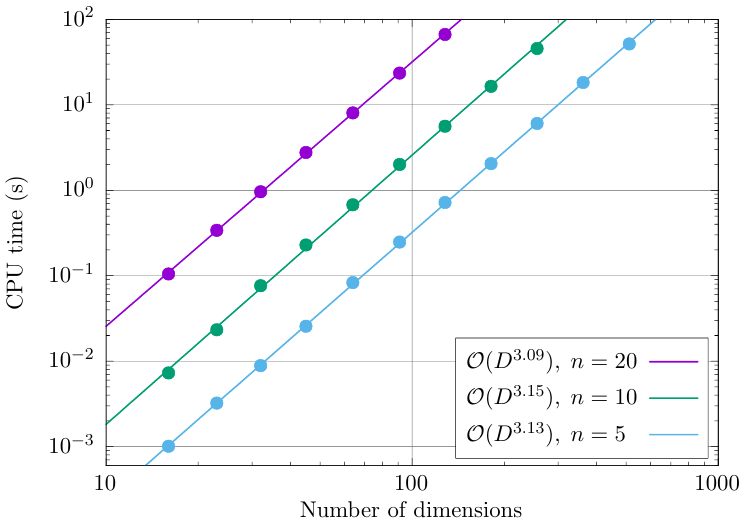}
  \caption{CPU times for the kernel \ac{mvp} with $\alpha = 3$ and $n = 5,10,20$.}
  \label{fig:gridgpr_benchmark_alpha3}
\end{figure}

\begin{figure}[p]
  \centering
  \includegraphics[width=0.5\linewidth]{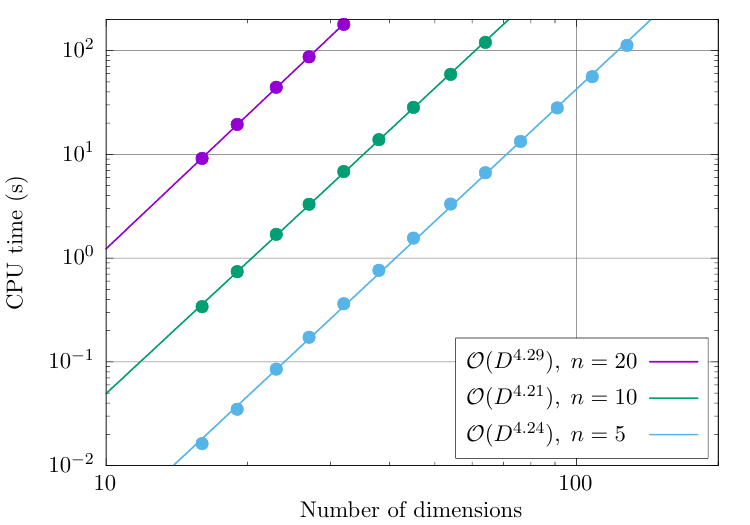}
  \caption{CPU times for the kernel \ac{mvp} with $\alpha = 4$ and $n = 5,10,20$.}
  \label{fig:gridgpr_benchmark_alpha4}
\end{figure}

\newpage
\subsection{Empirical $N$-scaling ($\alpha = 2$)}

\begin{figure}[H]
  \centering
  \includegraphics[width=0.5\linewidth]{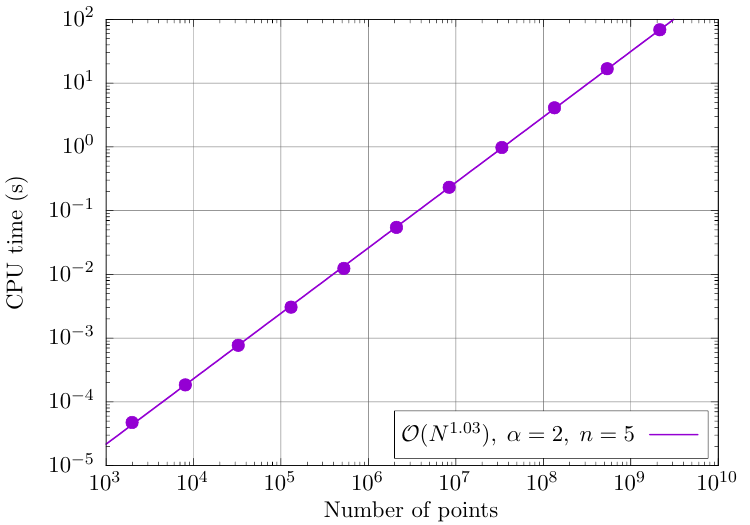}
  \caption{CPU times for the kernel \ac{mvp} with $\alpha = 2$ and $n = 5$.}
  \label{fig:gridgpr_scaling_alpha2_n5}
\end{figure}

\begin{figure}[H]
  \centering
  \includegraphics[width=0.5\linewidth]{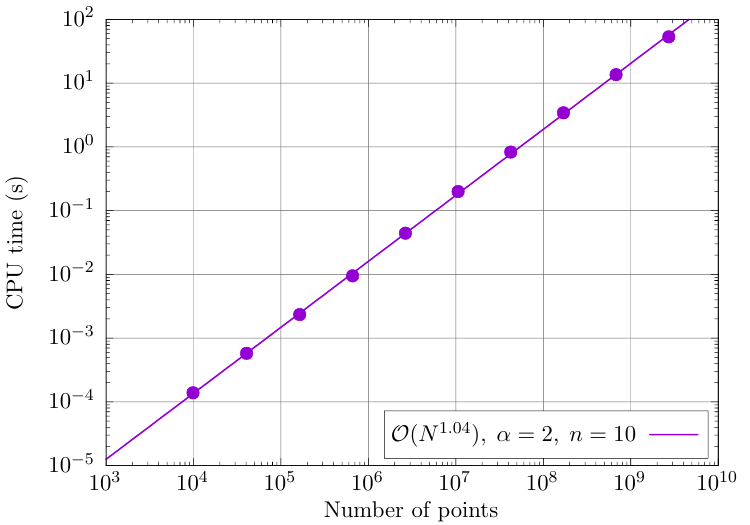}
  \caption{CPU times for the kernel \ac{mvp} with $\alpha = 2$ and $n = 10$.}
  \label{fig:gridgpr_scaling_alpha2_n10}
\end{figure}

\begin{figure}[H]
  \centering
  \includegraphics[width=0.5\linewidth]{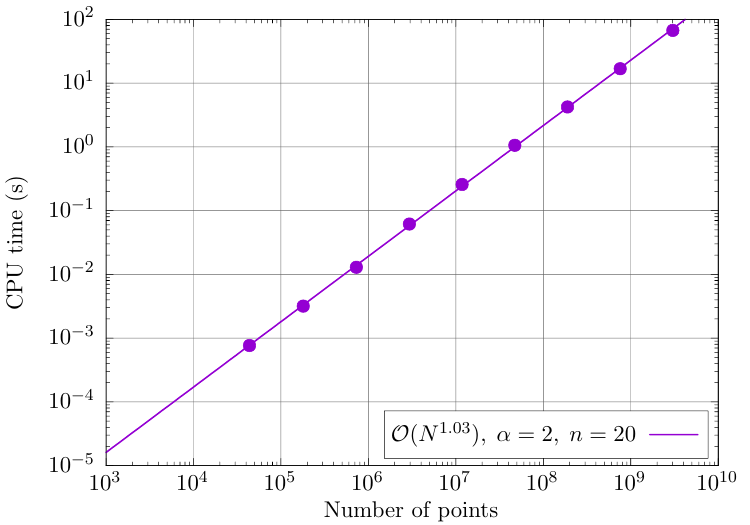}
  \caption{CPU times for the kernel \ac{mvp} with $\alpha = 2$ and $n = 20$.}
  \label{fig:gridgpr_scaling_alpha2_n20}
\end{figure}

\newpage
\subsection{Empirical $N$-scaling ($\alpha = 3$)}

\begin{figure}[H]
  \centering
  \includegraphics[width=0.5\linewidth]{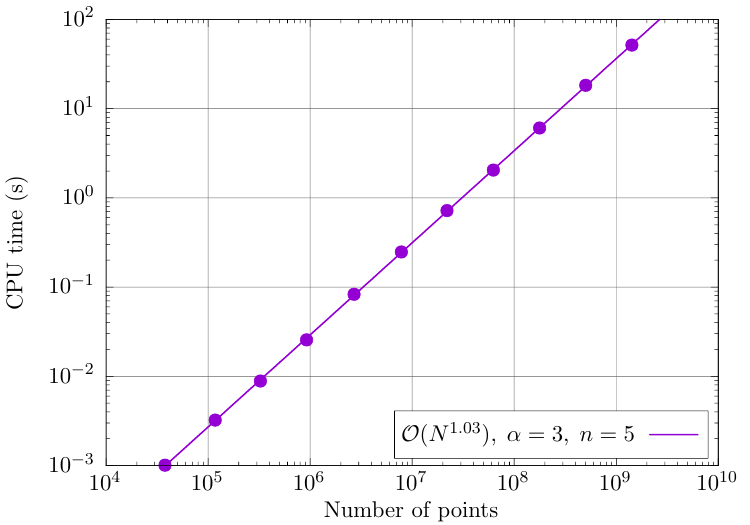}
  \caption{CPU times for the kernel \ac{mvp} with $\alpha = 3$ and $n = 5$.}
  \label{fig:gridgpr_scaling_alpha3_n5}
\end{figure}

\begin{figure}[H]
  \centering
  \includegraphics[width=0.5\linewidth]{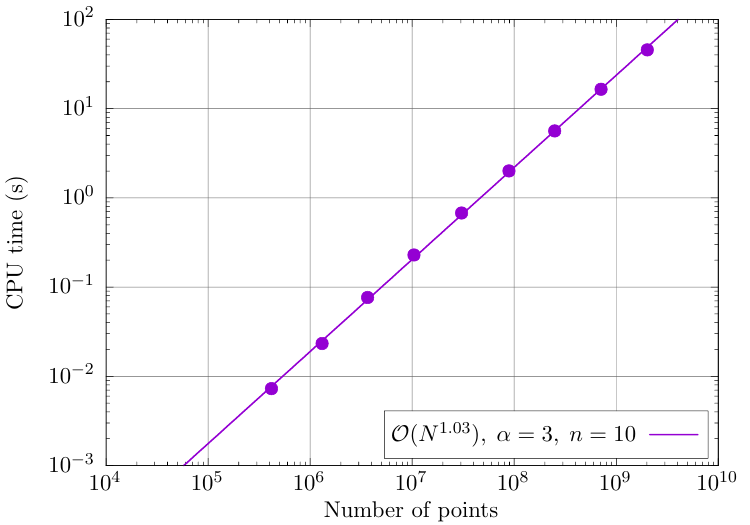}
  \caption{CPU times for the kernel \ac{mvp} with $\alpha = 3$ and $n = 10$.}
  \label{fig:gridgpr_scaling_alpha3_n10}
\end{figure}

\begin{figure}[H]
  \centering
  \includegraphics[width=0.5\linewidth]{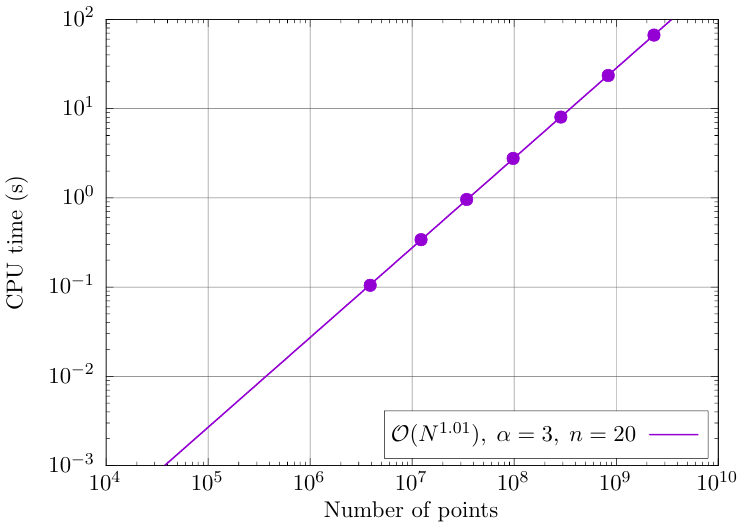}
  \caption{CPU times for the kernel \ac{mvp} with $\alpha = 3$ and $n = 20$.}
  \label{fig:gridgpr_scaling_alpha3_n20}
\end{figure}

\newpage
\subsection{Empirical $N$-scaling ($\alpha = 3$)}

\begin{figure}[H]
  \centering
  \includegraphics[width=0.5\linewidth]{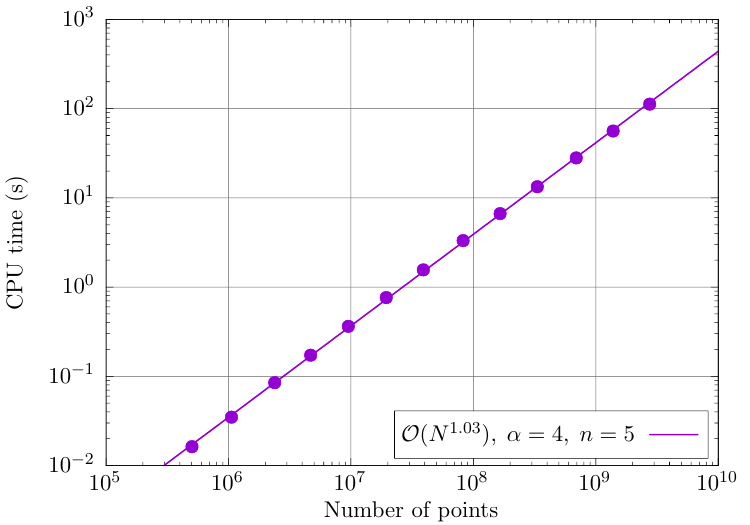}
  \caption{CPU times for the kernel \ac{mvp} with $\alpha = 4$ and $n = 5$.}
  \label{fig:gridgpr_scaling_alpha4_n5}
\end{figure}

\begin{figure}[H]
  \centering
  \includegraphics[width=0.5\linewidth]{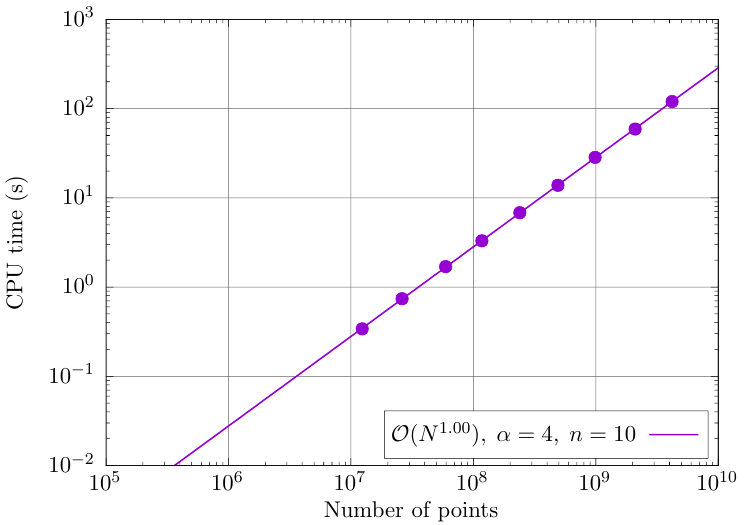}
  \caption{CPU times for the kernel \ac{mvp} with $\alpha = 4$ and $n = 10$.}
  \label{fig:gridgpr_scaling_alpha4_n10}
\end{figure}

\begin{figure}[H]
  \centering
  \includegraphics[width=0.5\linewidth]{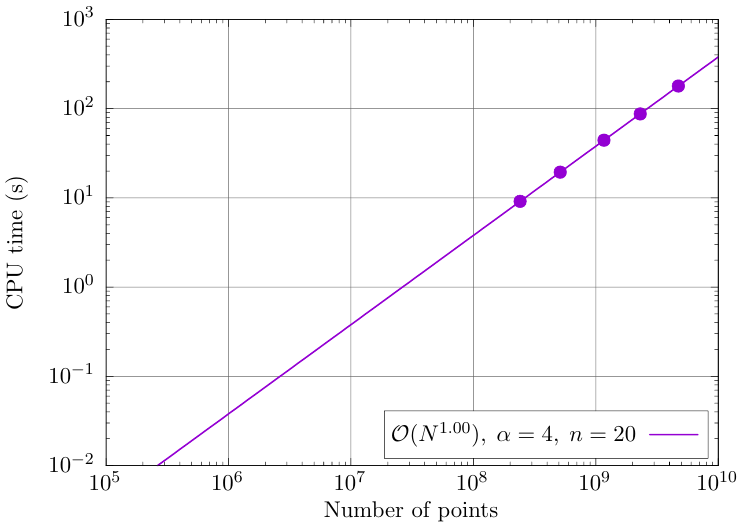}
  \caption{CPU times for the kernel \ac{mvp} with $\alpha = 4$ and $n = 20$.}
  \label{fig:gridgpr_scaling_alpha4_n20}
\end{figure}

\newpage
\subsection{Raw timings}

\begin{table}[H]
  \centering
  \caption{Benchmark data for $\alpha = 2$ and $n = 5,10,20$.
  Column five indicates the memory consumption corresponding to a single vector of size $\widehat{N}$ (assuming standard double precision).}
  \begin{tabular}{
  S[table-format=1.0  , table-text-alignment=right] 
  S[table-format=2.0  , table-text-alignment=right] 
  S[table-format=5.0  , table-text-alignment=right] 
  S[table-format=1.2e2, table-text-alignment=center, round-mode=places, round-precision=2] 
  r 
  S[table-format=1.2e1, table-text-alignment=center, round-mode=places, round-precision=2] 
  }
\toprule
{$\alpha$} & {$n$} & {$D$} & {$\widehat{N}$} & {Memory} & {CPU time (s)}   \\
\midrule
2        & 5     &    16  &  1.985000e+03  &  {\printmem{1}{1.588000e+04}} & 4.7439000000e-05 \\
         &       &    32  &  8.065000e+03  &  {\printmem{1}{6.452000e+04}} & 1.8467500000e-04 \\
         &       &    64  &  3.251300e+04  &  {\printmem{1}{2.601040e+05}} & 7.6993000000e-04 \\
         &       &   128  &  1.305610e+05  &  {\printmem{1}{1.044488e+06}} & 3.0601300000e-03 \\
         &       &   256  &  5.232650e+05  &  {\printmem{1}{4.186120e+06}} & 1.2411244000e-02 \\
         &       &   512  &  2.095105e+06  &  {\printmem{1}{1.676084e+07}} & 5.4491284000e-02 \\
         &       &  1024  &  8.384513e+06  &  {\printmem{1}{6.707610e+07}} & 2.3268220500e-01 \\
         &       &  2048  &  3.354624e+07  &  {\printmem{1}{2.683699e+08}} & 9.7750235100e-01 \\
         &       &  4096  &  1.342013e+08  &  {\printmem{1}{1.073611e+09}} & 4.1058150380e+00 \\
         &       &  8192  &  5.368381e+08  &  {\printmem{1}{4.294705e+09}} & 1.6802000000e+01 \\
         &       & 16384  &  2.147418e+09  &  {\printmem{1}{1.717934e+10}} & 6.8679000000e+01 \\
\midrule
2        & 10    &    16  &  9.865000e+03  &  {\printmem{1}{7.892000e+04}} & 1.3876600000e-04 \\
         &       &    32  &  4.046500e+04  &  {\printmem{1}{3.237200e+05}} & 5.7681600000e-04 \\
         &       &    64  &  1.638730e+05  &  {\printmem{1}{1.310984e+06}} & 2.3433010000e-03 \\
         &       &   128  &  6.595210e+05  &  {\printmem{1}{5.276168e+06}} & 9.4915620000e-03 \\
         &       &   256  &  2.646145e+06  &  {\printmem{1}{2.116916e+07}} & 4.4200680000e-02 \\
         &       &   512  &  1.060070e+07  &  {\printmem{1}{8.480564e+07}} & 1.9902654400e-01 \\
         &       &  1024  &  4.243507e+07  &  {\printmem{1}{3.394806e+08}} & 8.2786242500e-01 \\
         &       &  2048  &  1.698048e+08  &  {\printmem{1}{1.358438e+09}} & 3.4249155280e+00 \\
         &       &  4096  &  6.793482e+08  &  {\printmem{1}{5.434786e+09}} & 1.3572000000e+01 \\
         &       &  8192  &  2.717651e+09  &  {\printmem{1}{2.174121e+10}} & 5.3079000000e+01 \\
\midrule
2        & 20    &    16  &  4.362500e+04  &  {\printmem{1}{3.490000e+05}} & 7.6686500000e-04 \\
         &       &    32  &  1.796650e+05  &  {\printmem{1}{1.437320e+06}} & 3.1765050000e-03 \\
         &       &    64  &  7.289930e+05  &  {\printmem{1}{5.831944e+06}} & 1.2933457000e-02 \\
         &       &   128  &  2.936641e+06  &  {\printmem{1}{2.349313e+07}} & 6.1556135000e-02 \\
         &       &   256  &  1.178790e+07  &  {\printmem{1}{9.430324e+07}} & 2.5739680500e-01 \\
         &       &   512  &  4.723430e+07  &  {\printmem{1}{3.778744e+08}} & 1.0564553280e+00 \\
         &       &  1024  &  1.891026e+08  &  {\printmem{1}{1.512821e+09}} & 4.2271102590e+00 \\
         &       &  2048  &  7.567411e+08  &  {\printmem{1}{6.053929e+09}} & 1.6829000000e+01 \\
         &       &  4096  &  3.027626e+09  &  {\printmem{1}{2.422101e+10}} & 6.7062000000e+01 \\
\bottomrule
\end{tabular}
\label{tab:gridgpr_benchmark_alpha2}%
\end{table}%

\afterpage{\clearpage}
\begin{table}[p]
\centering
\caption{Benchmark data for $\alpha = 3$ and $n = 5,10,20$.
Column five indicates the memory consumption corresponding to a single vector of size $\widehat{N}$ (assuming standard double precision).}
\begin{tabular}{
S[table-format=1.0  , table-text-alignment=right] 
S[table-format=2.0  , table-text-alignment=right] 
S[table-format=4.0  , table-text-alignment=right] 
S[table-format=1.2e2, table-text-alignment=center, round-mode=places, round-precision=2] 
r
S[table-format=1.2e1, table-text-alignment=center, round-mode=places, round-precision=2] 
}
\toprule
{$\alpha$} & {$n$} & {$D$} & {$\widehat{N}$} & {Memory} & {CPU time (s)}   \\
\midrule
3      & 5     &   16  &  3.782500e+04  &  {\printmem{1}{3.026000e+05}} & 1.0114260000e-03 \\
       &       &   23  &  1.174850e+05  &  {\printmem{1}{9.398800e+05}} & 3.2355280000e-03 \\
       &       &   32  &  3.255050e+05  &  {\printmem{1}{2.604040e+06}} & 8.8771580000e-03 \\
       &       &   45  &  9.241810e+05  &  {\printmem{1}{7.393448e+06}} & 2.5648118000e-02 \\
       &       &   64  &  2.699009e+06  &  {\printmem{1}{2.159207e+07}} & 8.3170692000e-02 \\
       &       &   91  &  7.840925e+06  &  {\printmem{1}{6.272740e+07}} & 2.4786300400e-01 \\
       &       &  128  &  2.197862e+07  &  {\printmem{1}{1.758290e+08}} & 7.1886346600e-01 \\
       &       &  181  &  6.246744e+07  &  {\printmem{1}{4.997396e+08}} & 2.0480175660e+00 \\
       &       &  256  &  1.773885e+08  &  {\printmem{1}{1.419108e+09}} & 6.0638160120e+00 \\
       &       &  362  &  5.028658e+08  &  {\printmem{1}{4.022926e+09}} & 1.8252000000e+01 \\
       &       &  512  &  1.425373e+09  &  {\printmem{1}{1.140299e+10}} & 5.1540000000e+01 \\
\midrule
3      & 10    &   16  &  4.181050e+05  &  {\printmem{1}{3.344840e+06}} & 7.2877260000e-03 \\
       &       &   23  &  1.311760e+06  &  {\printmem{1}{1.049408e+07}} & 2.3299029000e-02 \\
       &       &   32  &  3.656305e+06  &  {\printmem{1}{2.925044e+07}} & 7.6501638000e-02 \\
       &       &   45  &  1.042511e+07  &  {\printmem{1}{8.340085e+07}} & 2.2928475800e-01 \\
       &       &   64  &  3.053693e+07  &  {\printmem{1}{2.442954e+08}} & 6.7717552900e-01 \\
       &       &   91  &  8.889508e+07  &  {\printmem{1}{7.111606e+08}} & 2.0032991450e+00 \\
       &       &  128  &  2.495226e+08  &  {\printmem{1}{1.996181e+09}} & 5.6226114350e+00 \\
       &       &  181  &  7.098872e+08  &  {\printmem{1}{5.679098e+09}} & 1.6474000000e+01 \\
       &       &  256  &  2.017252e+09  &  {\printmem{1}{1.613802e+10}} & 4.5549000000e+01 \\
\midrule
3      & 20    &   16  &  3.884665e+06  &  {\printmem{1}{3.107732e+07}} & 1.0492140900e-01 \\
       &       &   23  &  1.223906e+07  &  {\printmem{1}{9.791248e+07}} & 3.4037239500e-01 \\
       &       &   32  &  3.420030e+07  &  {\printmem{1}{2.736024e+08}} & 9.5973221600e-01 \\
       &       &   45  &  9.768746e+07  &  {\printmem{1}{7.814996e+08}} & 2.7660210560e+00 \\
       &       &   64  &  2.865024e+08  &  {\printmem{1}{2.292019e+09}} & 8.0436067910e+00 \\
       &       &   91  &  8.347456e+08  &  {\printmem{1}{6.677965e+09}} & 2.3509000000e+01 \\
       &       &  128  &  2.344435e+09  &  {\printmem{1}{1.875548e+10}} & 6.6587000000e+01 \\
\bottomrule
\end{tabular}
\label{tab:gridgpr_benchmark_alpha3}%
\end{table}%

\afterpage{\clearpage}
\begin{table}[p]
\centering
\caption{Benchmark data for $\alpha = 4$ and $n = 5,10,20$.
Column five indicates the memory consumption corresponding to a single vector of size $\widehat{N}$ (assuming standard double precision).}
\begin{tabular}{
S[table-format=1.0  , table-text-alignment=right] 
S[table-format=2.0  , table-text-alignment=right] 
S[table-format=4.0  , table-text-alignment=right] 
S[table-format=1.2e2, table-text-alignment=center, round-mode=places, round-precision=2] 
r
S[table-format=1.2e1, table-text-alignment=center, round-mode=places, round-precision=2] 
}
\toprule
{$\alpha$} & {$n$} & {$D$} & {$\widehat{N}$} & {Memory} & {CPU time (s)}   \\
\midrule
4      & 5     &  16  &  5.037450e+05  &  {\printmem{1}{4.029960e+06}} & 1.6304290000e-02 \\
       &       &  19  &  1.057085e+06  &  {\printmem{1}{8.456680e+06}} & 3.4892387000e-02 \\
       &       &  23  &  2.384365e+06  &  {\printmem{1}{1.907492e+07}} & 8.4956992000e-02 \\
       &       &  27  &  4.685725e+06  &  {\printmem{1}{3.748580e+07}} & 1.7273184200e-01 \\
       &       &  32  &  9.531265e+06  &  {\printmem{1}{7.625012e+07}} & 3.6317884500e-01 \\
       &       &  38  &  1.944794e+07  &  {\printmem{1}{1.555836e+08}} & 7.6284041600e-01 \\
       &       &  45  &  3.906690e+07  &  {\printmem{1}{3.125352e+08}} & 1.5587475620e+00 \\
       &       &  54  &  8.257082e+07  &  {\printmem{1}{6.605666e+08}} & 3.3214431540e+00 \\
       &       &  64  &  1.653553e+08  &  {\printmem{1}{1.322842e+09}} & 6.6624243250e+00 \\
       &       &  76  &  3.329867e+08  &  {\printmem{1}{2.663894e+09}} & 1.3305000000e+01 \\
       &       &  91  &  6.920444e+08  &  {\printmem{1}{5.536356e+09}} & 2.8020000000e+01 \\
       &       & 108  &  1.385087e+09  &  {\printmem{1}{1.108070e+10}} & 5.6021000000e+01 \\
       &       & 128  &  2.752987e+09  &  {\printmem{1}{2.202389e+10}} & 1.1224000000e+02 \\
\midrule
4      & 10    &  16  &  1.235912e+07  &  {\printmem{1}{9.887300e+07}} & 3.4041760600e-01 \\
       &       &  19  &  2.615086e+07  &  {\printmem{1}{2.092069e+08}} & 7.4013787200e-01 \\
       &       &  23  &  5.940942e+07  &  {\printmem{1}{4.752753e+08}} & 1.6954236990e+00 \\
       &       &  27  &  1.173066e+08  &  {\printmem{1}{9.384524e+08}} & 3.3044752560e+00 \\
       &       &  32  &  2.395899e+08  &  {\printmem{1}{1.916719e+09}} & 6.8263947770e+00 \\
       &       &  38  &  4.905073e+08  &  {\printmem{1}{3.924059e+09}} & 1.3839000000e+01 \\
       &       &  45  &  9.879813e+08  &  {\printmem{1}{7.903850e+09}} & 2.8385000000e+01 \\
       &       &  54  &  2.093121e+09  &  {\printmem{1}{1.674497e+10}} & 5.9008000000e+01 \\
       &       &  64  &  4.199239e+09  &  {\printmem{1}{3.359391e+10}} & 1.2008000000e+02 \\
\midrule
4      & 20    &  16  &  2.410689e+08  &  {\printmem{1}{1.928551e+09}} & 9.1353680980e+00 \\
       &       &  19  &  5.118327e+08  &  {\printmem{1}{4.094661e+09}} & 1.9393000000e+01 \\
       &       &  23  &  1.166232e+09  &  {\printmem{1}{9.329852e+09}} & 4.4275000000e+01 \\
       &       &  27  &  2.307323e+09  &  {\printmem{1}{1.845859e+10}} & 8.7063000000e+01 \\
       &       &  32  &  4.720543e+09  &  {\printmem{1}{3.776435e+10}} & 1.7914000000e+02 \\
\bottomrule
\end{tabular}
\label{tab:gridgpr_benchmark_alpha4}%
\end{table}%


\newpage
\section{Exact baseline comparison} \label{appendix:exact_baselines}

\setcounter{equation}{0}
\renewcommand{\theequation}{\thesection\arabic{equation}}

\setcounter{figure}{0}
\renewcommand{\thefigure}{\thesection\arabic{figure}}

\setcounter{table}{0}
\renewcommand{\thetable}{\thesection\arabic{table}}

\subsection{Datasets} \label{appendix:generating_datasets}

In order to demonstrate the correctness of our code and to shed some light on the end-to-end scaling,
we have performed a series of small-scale baseline calculations where a standard/exact \ac{gpr} reference
is easily available. Specifically, we consider the \acp{pes} of the eight molecules in Table~\ref{tab:small_baseline_molecules}
using an incomplete grid of order $\alpha = 2$ and an \ac{oak} of order $\omega = 2$. Details on kernel centering
are provided in Appendix~\ref{appendix:kernel_centering}.

The \acp{pes} were sampled according to the 1D grids in Table~\ref{tab:1D_grids} with $0.0$ being the reference point for each dimension.
A coarse grid ($n = 7$) was used for training, 
while the test set consists of a fine grid ($n = 11$) with the training points removed.\footnote{The datasets are available at \url{https://anonymous.4open.science/r/DONT_GET_YOUR_KRONECKERS_IN_A_TWIST_DATASETS-F084}.}

All \ac{pes} values (electronic energies)
have been computed with the semi-empirical 
GFN2-xTB method \citep{bannwarthGFN2xTBAnAccurateBroadly2019}
as implemented the xTB program\footnote{Version 6.3.3, license CC BY-SA 4.0, available at \url{https://github.com/grimme-lab/xtb}.}  \citep{bannwarthExtendedTightbindingQuantum2021}.
First, the molecular geometries were optimized using GFN2-xTB
with the `extreme' convergence threshold (the energy is converged to $\SI{5e-8}{\hartree}$ and the gradient norm to $\SI{5e-5}{\hartree/a_0}$).
Then, the Hessian was computed at the same level of theory.
The Hessian was subsequently diagonalized by the MidasCpp program\footnote{Version 2025.10.0, license LGPL 2.1, available at \url{https://source.coderefinery.org/midascpp/midascpp}.} \citep{christiansenMidasCppMolecularInteractions2025}
to obtain mass-scaled normal coordinates. Finally, MidasCpp was used to generate the set of grid 
points and gather the electronic energies via an interface to the xTB program.

\begin{table}[H]
  \centering
  \caption{Eight small molecules of increasing size.}
  \begin{tabular}{l l S[table-format=2, table-text-alignment=right] S[table-format=4, table-text-alignment=right]}
  \toprule
  Name                            & Chemical formula          & {$D$} & {$\widehat{N}$}  \\
  \midrule
  Water                           & $\mathrm{H_2 O}$          &  3 &  127\\
  Hydrogen cyanide                & $\mathrm{HCN}$            &  4 &  241\\
  Formaldehyde                    & $\mathrm{H_2 CO}$         &  6 &  577\\
  Acetylene                       & $\mathrm{C_2 H_2}$        &  7 &  799\\
  Difluoromethane                 & $\mathrm{CH_2 F_2}$       &  9 & 1351\\
  Cyanoacetylene                  & $\mathrm{C_3 NH}$         & 10 & 1681\\
  \textit{trans}-Difluoroethylene & $\mathrm{C_2 H_2 F_2}$    & 12 & 2449\\
  Diacetylene                     & $\mathrm{C_4 H_2}$        & 13 & 2887\\
  \bottomrule
  \end{tabular}
  \label{tab:small_baseline_molecules}%
  \end{table}%

\begin{table}[H]
  \centering
  \caption{1D grids used for sampling the \acp{pes}. The reference point is $0.0$.
  The grid points are given as fractions of the 1D grid bounds, which
  are chosen as the classical turning points of the 1D harmonic oscillator state with $v = 10$.
  In mass-scaled coordinated, the turnings points are given by $x_{\mathrm{TP}}(v) = \pm \sqrt{2 \hbar (v + 1/2) / \omega}$.}
  \begin{tabular}{llllllllllll}
\toprule
\multicolumn{12}{c}{1D grid points} \\
\midrule
Fine     & $-1.0$ & $-0.8$ & $-0.6$ & $-0.4$ & $-0.2$ & $0.0$ & $+0.2$ & $+0.4$ & $+0.6$ & $+0.8$ & $+1.0$ \\
Coarse   & $-1.0$ &        & $-0.6$ &        & $-0.2$ & $0.0$ & $+0.2$ &        & $+0.6$ &        & $+1.0$ \\
\bottomrule
\end{tabular}
\label{tab:1D_grids}%
\end{table}%

\newpage
\subsection{Kernel centering} \label{appendix:kernel_centering}

Generally speaking, the base kernel centering depends on the distribution of input locations in
the training set.\autocite{luAdditiveGaussianProcesses2022,ishidaHierarchicalAdditiveInteraction2025}
Using the 
empirical distribution, which is discrete, one finds 
that centering amounts to the following inexpensive modification of the training--training,
training--test and test--test base matrices:
\begin{subequations}
  \begin{align}
    \mathbf{K}^{(m)} &\leftarrow \mathbf{K}^{(m)} 
    - \frac{\mathbf{K}^{(m)} \mathbf{w}^{(m)} (\mathbf{w}^{(m)})^{\trans} \mathbf{K}^{(m)} }{(\mathbf{w}^{(m)})^{\trans} \mathbf{K}^{(m)} \mathbf{w}^{(m)}}, \\
    \mathbf{K}^{(m)}_{*} &\leftarrow \mathbf{K}^{(m)}_{*} 
    - \frac{\mathbf{K}^{(m)} \mathbf{w}^{(m)} (\mathbf{w}^{(m)})^{\trans} \mathbf{K}^{(m)}_{*} }{(\mathbf{w}^{(m)})^{\trans} \mathbf{K}^{(m)} \mathbf{w}^{(m)}}, \\
    \mathbf{K}^{(m)}_{**} &\leftarrow \mathbf{K}^{(m)}_{**}
    - \frac{(\mathbf{K}^{(m)}_{*})^{\trans} \mathbf{w}^{(m)} (\mathbf{w}^{(m)})^{\trans} \mathbf{K}^{(m)}_{*} }{(\mathbf{w}^{(m)})^{\trans} \mathbf{K}^{(m)} \mathbf{w}^{(m)}}.
  \end{align}
\end{subequations}
The vector $\mathbf{w}^{(m)}$ contains the frequencies with which 
the 1D grid points for mode $m$ occur in the training set.
The first element of $\mathbf{w}^{(m)}$ is the largest,
since the reference is the most frequent point for every mode.

\subsection{Settings} \label{appendix:exact_baselines_settings}

By construction, the inputs span the range $[-1, +1]$. 
For each dimension, the initial length scale was computed as twice the 
average distance between neighbouring 1D grid points. In all cases, this amounts
to $\ell^{(m)} = 0.667$. The initial order variances were chosen as
$\sigma_0^2 = 10^{-4}$, $\sigma_1^2 = 1$ and $\sigma_2^2 = 1$. The training outputs were standardized
to have zero mean and unit variance.

All GPR calculations (CUTS-GPR and conventional GPR) were run with a fixed noise of $\sigma^2 = 10^{-3}$, 
while the remaining hyperparameters were optimized using 100 cycles of the Adam algorithm \citep{kingmaAdamMethodStochastic2017} with learning rate 0.1.
For CUTS-GPR, the preconditioner rank was $k = 10$ and the \ac{cg} residuals were converged to a relative norm of
$\varepsilon_{\mathrm{CG}} = 10^{-3}$. 35 reparameterized Rademacher probe vectors were used for stochastic trace estimation (see Appendices~\ref{sec:stochastic_trace_estimation} and \ref{sec:sampling_probe_vectors}).
The CUTS-GPR calculations were run on 36 cores of a twin Intel Xeon Platinum 8358 CPU using standard double precision.
The standard GPR baseline was run on a single core of an Intel Xeon Platinum 8358 CPU. 

\newpage
\subsection{Results}

Table~\ref{tab:exact_baseline_rn_rmse} compares the range normalized RMSE for the eight molecules.
When hyperparameters are independently optimized using 100 cycles of Adam, CUTS-GPR produces an RMSE
that is very similar to the exact baseline. Truly identical results are not obtained, which is not
unexpected since the CUTS-GPR hyperparameter gradients are merely stochastic estimates of the true
gradients. The CUTS-GPR optimization will therefore follow a slightly different path
compared to the baseline.

Exact agreement (to at least three digits) is recovered if the CUTS-GPR predictions 
are recomputed using the baseline hyperparameters and a tight
\ac{cg} threshold of $\varepsilon_{\mathrm{CG}} = 10^{-6}$.

Table~\ref{tab:exact_baseline_cpu_time} shows the total CPU times for CUTS-GPR and the exact baseline. The timings
have been fitted in Figures~\ref{fig:baseline_scaling_D} and \ref{fig:baseline_scaling_N} to determine
the empirical end-to-end scaling. For CUTS-GPR the empirical complexities [$\mathcal{O}(D^{3.22})$ and $\mathcal{O}(\widehat{N}^{1.51})$]
are larger than the MVP-complexities [$\mathcal{O}(D^{2})$ and $\mathcal{O}(\widehat{N})$], showing that the number of \ac{cg} iterations
tends to grow with $D$ and $\widehat{N}$. The overall scaling is nevertheless small.

\begin{table}[H]
  \centering
  \caption{Range normalized RMSE (divided by $10^{-3}$) for the eight molecules in Table~\ref{tab:small_baseline_molecules}.
  The rows `CUTS-GPR optimized' and `Standard GPR optimized' show the results obtained with independent hyperparameter optimization.
  The row `CUTS-GPR fixed' shows the results obtained when computing CUTS-GPR predictions using the hyperparameters from `Standard GPR optimized'
  and $\varepsilon_{\mathrm{CG}} = 10^{-6}$.}
  \begin{tabular}{
    r
    l
    S[table-format=2.3, table-text-alignment=right] 
    S[table-format=2.3, table-text-alignment=right] 
    S[table-format=2.3, table-text-alignment=right] 
    S[table-format=2.3, table-text-alignment=right] 
    S[table-format=2.3, table-text-alignment=right] 
    S[table-format=2.3, table-text-alignment=right] 
    S[table-format=2.3, table-text-alignment=right] 
    S[table-format=2.3, table-text-alignment=right] 
    }
\toprule
$D$          &            & {  3} &  {  4} & {  6} & {  7} & {   9} & {  10} & {  12} & {  13} \\
$N$          &            & {127} &  {241} & {577} & {799} & {1351} & {1681} & {2449} & {2887} \\ \midrule
CUTS-GPR     & optimized  & 9.915 & 13.754 & 7.640 & 5.065 &  5.802 & 10.316 &  4.865 &  4.155 \\
Standard GPR & optimized  & 9.872 & 11.931 & 7.633 & 5.799 &  5.990 & 10.169 &  5.319 &  4.358 \\
CUTS-GPR     & fixed      & 9.872 & 11.931 & 7.633 & 5.799 &  5.990 & 10.169 &  5.319 &  4.358 \\
\bottomrule
\end{tabular}
\label{tab:exact_baseline_rn_rmse}%
\end{table}%

\begin{table}[H]
  \centering
  \caption{Total CPU time (in seconds) for optimizing hyperparameters and computing the predictive mean with
  CUTS-GPR and standard/exact GPR.}
  \begin{tabular}{
    r
    r
    S[table-format=1.2e1, table-text-alignment=center, round-mode=places, round-precision=2, print-zero-exponent=true] 
    S[table-format=1.2e1, table-text-alignment=center, round-mode=places, round-precision=2, print-zero-exponent=true] 
    }
\toprule
$D$ & $N$ & {CUTS-GPR}      & {Standard GPR}    \\ \midrule
 3 &  127	& 7.78319200e+00	& 1.23563400e+01 \\
 4 &  241	& 3.04562890e+01	& 7.23725720e+01 \\
 6 &  577	& 1.09855785e+02	& 4.81860201e+02 \\
 7 &  799	& 2.08639133e+02	& 1.17241979e+03 \\
 9 & 1351	& 3.04534983e+02	& 5.44395374e+03 \\
10 & 1681	& 4.34073701e+02	& 1.03524478e+04 \\
12 & 2449	& 7.44626800e+02	& 3.21788326e+04 \\
13 & 2887	& 1.24667585e+03	& 5.38131042e+04 \\
\bottomrule
\end{tabular}
\label{tab:exact_baseline_cpu_time}%
\end{table}%

\afterpage{\clearpage}

\begin{figure}[p]
\centering
\includegraphics[width=0.7\linewidth]{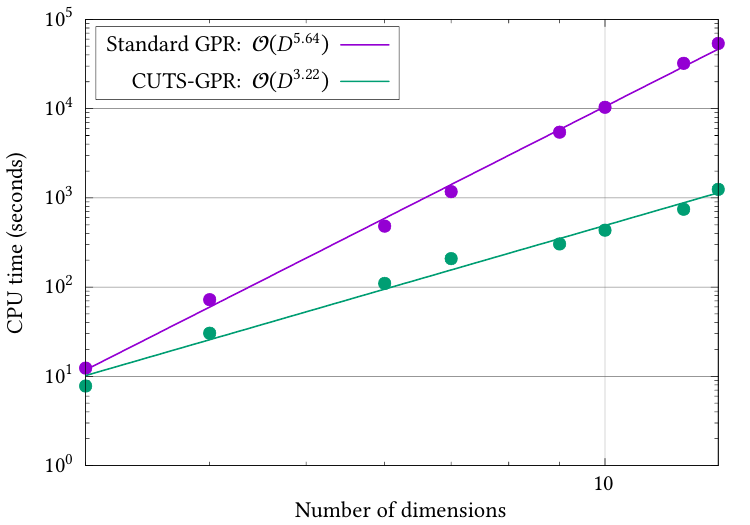}
\caption{Empirical $D$-scaling for CUTS-GPR and standard/exact GPR.}
\label{fig:baseline_scaling_D}
\end{figure}

\begin{figure}[p]
\centering
\includegraphics[width=0.7\linewidth]{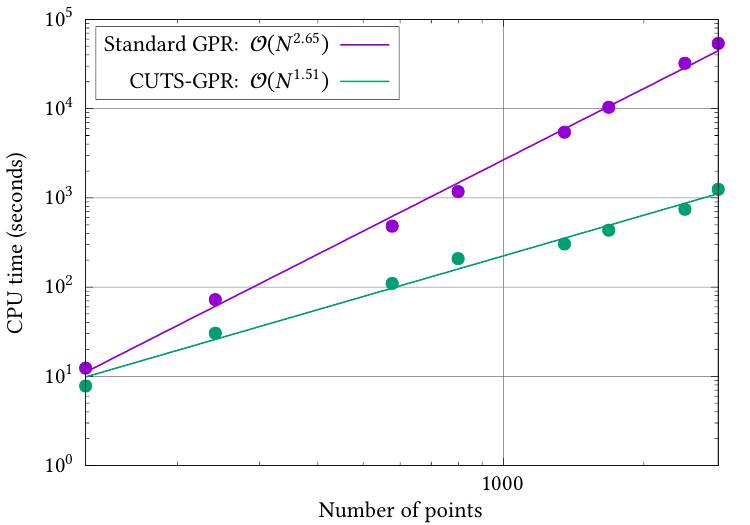}
\caption{Empirical $N$-scaling for CUTS-GPR and standard/exact GPR.}
\label{fig:baseline_scaling_N}
\end{figure}


\newpage
\section{The effect of truncating kernel and grid} \label{appendix:truncating_kernel_and_grid}

\setcounter{equation}{0}
\renewcommand{\theequation}{\thesection\arabic{equation}}

\setcounter{figure}{0}
\renewcommand{\thefigure}{\thesection\arabic{figure}}

\setcounter{table}{0}
\renewcommand{\thetable}{\thesection\arabic{table}}

In order to isolate the effect of truncating the grid, we consider the PES of carbon dioxide ($D = 4$).
We use training grids of order $\alpha = 2, 3, 4$ and \ac{oak} kernels of order $\omega = 2, 3, 4$. 
For training, we use the 1D grids defined by Table~\ref{tab:1D_grids} ($n = 7$). The resulting number of training
points is stated in Table~\ref{tab:truncating_kernel_and_grid_num_training_points}.
The training outputs were standardized to have zero mean and unit variance.

The test grid, which is kept fixed, is a complete grid with $n = 11$ from which the training points are removed.
For the sake of presentation, the test points are partitioned according to displacement order as in Table~\ref{tab:truncating_kernel_and_grid_num_test_points}.
All calculations have been run using a conventional, exact GP code with the same settings as in Section~\ref{appendix:exact_baselines_settings}.

Table~\ref{tab:truncating_kernel_and_grid} show the RMSE for all nine combinations of $\omega$ (kernel order)
and $\alpha$ (grid order).
We notice several things.
For all $\omega$, the RMSE improves when $\alpha$ is increased.
The 1D and 2D cuts are well-modelled with $\alpha = 2$, while the 4D cuts show a larger error.
Choosing $\omega > \alpha$ does not seem to improve the model (there is, apparently, no benefit in having a model more complex than the data).
The most accurate model is $\omega = \alpha =4$. However, this is also by far the most expensive model.
Depending on the intended application and required accuracy, one can debate how elaborate the model should be. 
However, we note that even the simplest/cheapest model ($\omega = \alpha = 2$) yields a very reasonable (and very affordable) description. 
This is exactly what the low-order model achieves in many cases: 
An affordable model with reasonable error. 
Ultimately, we believe that the best argument in favour of the low-order model is the vast 
number of successful applications in the literature (see, e.g., the references in Section~\ref{sec:introduction}).

\begin{table}[H]
  \centering
  \caption{Number of training points for the carbon dioxide \acp{pes} ($D = 4$).}
  \begin{tabular}{
    c
    r
    }
\toprule
$\alpha$ &  $\widehat{N}$  \\ \midrule
       2 &            241  \\
       3 &           1105  \\
       4 &           2401  \\
\bottomrule
\end{tabular}
\label{tab:truncating_kernel_and_grid_num_training_points}%
\end{table}%

\begin{table}[H]
  \centering
  \caption{Number of test points for the carbon dioxide \acp{pes} ($D = 4$). The test points have been partitioned
  according to displacement order so that the subset labelled `1D' contains the test points that are displaced along
  one dimension etc.}
  \begin{tabular}{
    c
    r
    }
\toprule
 Order &  $\widehat{N}_{*}$   \\ \midrule
    1D &                 16   \\
    2D &                384   \\
    3D &               3136   \\
    4D &               8704   \\
   All &              12240   \\
\bottomrule
\end{tabular}
\label{tab:truncating_kernel_and_grid_num_test_points}%
\end{table}%

\begin{table}[H]
  \centering
  \caption{RMSE for all combinations of kernel order ($\omega$) and grid order ($\alpha$). The RMSE has been
  computed }
  \begin{tabular}{
    c
    c
    S[table-format=1.2e-1, table-text-alignment=left, round-mode=places, round-precision=2] 
    S[table-format=1.2e-1, table-text-alignment=left, round-mode=places, round-precision=2] 
    S[table-format=1.2e-1, table-text-alignment=left, round-mode=places, round-precision=2] 
    S[table-format=1.2e-1, table-text-alignment=left, round-mode=places, round-precision=2] 
    S[table-format=1.2e-1, table-text-alignment=left, round-mode=places, round-precision=2] 
    }
\toprule
Kernel ($\omega$) &  Grid ($\alpha$) & {RMSE (1D)} & {RMSE (2D)} & {RMSE (3D)} & {RMSE (4D)} & {RMSE (all)} \\ \midrule
                2 &                2 & 3.333e-04   & 7.193e-04   & 9.641e-03   & 1.788e-02   & 1.585e-02    \\
                  &                3 & 3.523e-03   & 3.829e-03   & 5.967e-03   & 9.944e-03   & 8.939e-03    \\
                  &                4 & 4.188e-03   & 5.162e-03   & 5.488e-03   & 5.481e-03   & 5.471e-03    \\ \midrule
                3 &                2 & 3.166e-04   & 7.118e-04   & 9.646e-03   & 1.791e-02   & 1.587e-02    \\
                  &                3 & 3.849e-04   & 5.950e-04   & 6.775e-04   & 4.124e-03   & 3.496e-03    \\
                  &                4 & 8.739e-04   & 7.958e-04   & 8.879e-04   & 8.580e-04   & 8.639e-04    \\ \midrule
                4 &                2 & 3.166e-04   & 7.118e-04   & 9.646e-03   & 1.791e-02   & 1.587e-02    \\
                  &                3 & 3.844e-04   & 5.946e-04   & 6.776e-04   & 4.119e-03   & 3.492e-03    \\
                  &                4 & 4.257e-04   & 5.685e-04   & 5.964e-04   & 5.355e-04   & 5.527e-04    \\
\bottomrule
\end{tabular}
\label{tab:truncating_kernel_and_grid}%
\end{table}%


\newpage
\section{End-to-end scaling in high dimension} \label{appendix:end_to_end_scaling}

\setcounter{equation}{0}
\renewcommand{\theequation}{\thesection\arabic{equation}}

\setcounter{figure}{0}
\renewcommand{\thefigure}{\thesection\arabic{figure}}

\setcounter{table}{0}
\renewcommand{\thetable}{\thesection\arabic{table}}

\subsection{Synthetic datasets}

To study the end-to-end scaling of CUTS-GPR in high dimension (our target setting),
we have considered a number of synthetic datasets ranging from $D = 100, \widehat{N} = \num{1681}$
to $D = 500, \widehat{N} = \num{4494001}$.\footnote{The datasets are available at \url{https://anonymous.4open.science/r/DONT_GET_YOUR_KRONECKERS_IN_A_TWIST_DATASETS-F084}.}
The datasets are based on a class of synthetic second-order functions
of the following form:
\begin{align} \label{eq:synthetic_function}
  f(x_1, \ldots, x_D) = \sum_{i} g_i (x_i) + \sum_{i < j} h_{ij} (x_i, x_j)
\end{align}
The components are of the following form:
\begin{align}
  g_i (x_i) &=
    \begin{dcases}[l]
      a_i  \sin(\omega_i x_i)                                      \\[0.2em]
      a_i  \big(1 - \cos(\omega_i x_i)\big)                        \\[0.2em]
      a_i  \big(\sin(\omega_i x_i)\big)^3                          \\[0.2em]
      a_i  \big(\sin(\omega_i x_i) + 0.4 \sin(2 \omega_i x_i)\big) \\[0.2em]
    \end{dcases} 
  \\[1em]
  h_{ij} (x_i, x_j) &=
    \begin{dcases}[l]
      b_{ij} \sin(x_i) \sin(x_j)                                   \\[0.2em]
      b_{ij} \big(1 - \cos(x_i)\big) \big(1 - \cos(x_j)\big)       \\[0.2em]
      b_{ij} \sin(x_i)  \big(1 - \cos(x_j)\big)                    \\[0.2em]
      b_{ij} \big(\sin(x_i + x_j) - \sin(x_i) - \sin(x_j)\big)     \\[0.2em]
      b_{ij} \sin(x_i  x_j)                                        \\[0.2em]
    \end{dcases}
\end{align}
Note that the second-order components are not necessarily separable.
For each term in \eqref{eq:synthetic_function}, the function type as well as the coefficients
$a_i$, $\omega_i$ and $b_{ij}$ are drawn at random (using a seeded random number generator for reproducibility).
The coefficients are drawn uniformly from the following ranges:
\begin{align}
  a_i      &\in [0.5, 1.5], \\
  \omega_i &\in [0.2, 2.5], \\
  b_{ij}   &\in [-\eta a_{\mathrm{mean}}, +\eta a_{\mathrm{mean}}], \quad \eta = 0.2.
\end{align}

Having defined the functions, we construct the grids as in Appendix~\ref{appendix:exact_baselines}, i.e.
using $\alpha = 2$ and the 1D grids defined in Table~\ref{tab:1D_grids}. The sizes of the resulting
datasets are given in Table~\ref{tab:end_to_end_scaling_num_dim_points}.

\begin{table}[H]
  \centering
  \caption{Size of synthetic training and test sets.}
  \begin{tabular}{
    S[table-format=3     , table-text-alignment=right] 
    S[table-format=7     , table-text-alignment=right] 
    S[table-format=7     , table-text-alignment=right] 
    }
\toprule
{$D$} &   {$\widehat{N}$} & {$\widehat{N}_{*}$} \\ \midrule
 10   &              1681 &                2920 \\	
 25   &             10951 &               19300 \\	
 50   &             44401 &               78600 \\	
100   &            178801 &              317200 \\	
200   &            717601 &             1274400 \\	
300   &           1616401 &             2871600 \\	
400   &           2875201 &             5108800 \\	
500   &           4494001 &             7986000 \\	
\bottomrule
\end{tabular}
\label{tab:end_to_end_scaling_num_dim_points}%
\end{table}%

\subsection{Settings}
The CUTS-GPR calculations were run with the same settings as in Appendix~\ref{appendix:exact_baselines_settings}
except for the initial order variances, which were chosen as 
$\sigma_0^2 = 10^{-6}$, $\sigma_1^2 = 1$ and $\sigma_2^2 = 10^{-2}$.
The calculations were run on 36 cores of an AMD EPYC 9655 CPU using standard double precision. 

\subsection{Priors over hyperparameters} \label{appendix:priors}

Initial experiments showed that the order variances, $\sigma^2_{k}$, tend to become very 
large if the optimization is run without constraints.
This leads to deterioration of the
condition number of the kernel, which in turn reduces the overall computational efficiency.
At the same time, we observed no significant gain in accuracy when the order variances were allowed to grow freely.
To avoid this situation, we
place $\mathrm{Gamma}(1.0, 0.1)$ priors on the order variances, similar to \citet{luAdditiveGaussianProcesses2022}.

Similarly, we place a log-normal prior on each length scale
as described by \citet{hvarfnerVanillaBayesianOptimization2024}.
Since the average distance between two randomly sampled points in a
$D$-dimensional hypercube scales as $\sqrt{D}$, they recommend a 
prior of the form 
\begin{align}
  \ell^{(m)} \sim \mathrm{Lognormal}\left( \mu_0 + \log(\sqrt{D}), \sigma_0 \right)
\end{align}
whose mean and mode also scale as $\sqrt{D}$.
In our setup, the points are confined to low-dimensional subgrids/cuts rather
than being distributed over the entire hypercube. 
Assuming a cut level of $\alpha$, the effective dimensionality turns out to be
$2 \alpha < D$ (because two points on two non-overlapping $\alpha$-cuts differ along
exactly $2 \alpha$ dimensions). We therefore use
\begin{align}
  \ell^{(m)} \sim \mathrm{Lognormal}\left( \mu_0 + \log(\sqrt{2 \alpha}), \sigma_0 \right)
\end{align}
with $\mu_0 = \sqrt{2}$ and $\sigma_0 = \sqrt{3}$.

To maintain the regularizing effect as the 
number of training points increases, we normalize the \ac{mll} 
before adding the prior terms. The objective function
we optimize is thus
\begin{align}
  \mathcal{L}' = \frac{\mathcal{L}}{\widehat{N}} + \text{prior terms}
\end{align}

\subsection{Results}

Table~\ref{tab:end_to_end_scaling_data} shows the results and timings for the synthetic datasets.
We note that the range normalized RMSE is fairly consistent across the eight cases, indicating
that modelling accuracy is preserved in a truly high-dimensional setting.
The total CPU times have been fitted in Figures~\ref{fig:end_to_end_scaling_D_appendix} and \ref{fig:end_to_end_scaling_N_appendix}
to obtain empirical end-to-end complexities of $\mathcal{O}(D^{2.76})$ and $\mathcal{O}(\widehat{N}^{1.37})$.


\begin{table}[H]
  \centering
  \caption{Range normalized RMSE and CPU times for the synthetic datasets.
  }
  \begin{tabular}{
    S[table-format=3     , table-text-alignment=right] 
    S[table-format=7     , table-text-alignment=right] 
    S[table-format=1.3e-1, table-text-alignment=left, round-mode=places, round-precision=3, print-zero-exponent=true] 
    S[table-format=1.3e1 , table-text-alignment=left, round-mode=places, round-precision=3, print-zero-exponent=true] 
    S[table-format=1.3e-1, table-text-alignment=left, round-mode=places, round-precision=3, print-zero-exponent=true] 
    S[table-format=1.3e1 , table-text-alignment=left, round-mode=places, round-precision=3, print-zero-exponent=true] 
    S[table-format=1.3e1 , table-text-alignment=left, round-mode=places, round-precision=3, print-zero-exponent=true] 
    S[table-format=2.1   , table-text-alignment=center] 
    }
\toprule
{$D$} &   {$\widehat{N}$} & {RN RMSE}               & {$t_\mathrm{CPU}^\mathrm{opt}$ (s)} & {$t_\mathrm{CPU}^\mathrm{pred}$ (s)} & {$t_\mathrm{CPU}^\mathrm{tot}$ (s)}  & {$t_\mathrm{wall}^\mathrm{tot}$ (s)} & {$t_\mathrm{CPU}^\mathrm{tot} / t_\mathrm{wall}^\mathrm{tot}$} \\ \midrule
 10   &              1681 & 7.5827312247787248e-03	& 6.04474890e+01                      & 2.49866000e-01                       & 6.13104480e+01                       & 4.36698473e+00                       & 14.0                                                           \\
 25   &             10951 & 3.5454590916067070e-03	& 3.29661662e+02                      & 3.25273400e+00                       & 3.34668248e+02                       & 2.04454595e+01                       & 16.4                                                           \\
 50   &             44401 & 1.6947138428276973e-02	& 2.04444368e+03                      & 3.38334920e+01                       & 2.08272354e+03                       & 1.04273709e+02                       & 20.0                                                           \\
100   &            178801 & 8.7765542115769719e-03	& 3.40150001e+04                      & 6.19177003e+02                       & 3.46598830e+04                       & 1.18200289e+03                       & 29.3                                                           \\
200   &            717601 & 8.8548025495732135e-03	& 2.73638077e+05                      & 1.22798734e+04                       & 2.86066542e+05                       & 8.83203270e+03                       & 32.4                                                           \\
300   &           1616401 & 9.1251437387385604e-03	& 4.31968001e+05                      & 6.37425915e+04                       & 4.96141753e+05                       & 1.58912814e+04                       & 31.2                                                           \\
400   &           2875201 & 8.3360107812935069e-03	& 7.78178095e+05                      & 1.99459507e+05                       & 9.78366029e+05                       & 2.98099723e+04                       & 32.8                                                           \\
500   &           4494001 & 9.1316728687857313e-03	& 1.20749917e+06                      & 4.92581256e+05                       & 1.70131798e+06                       & 5.40648087e+04                       & 31.5                                                           \\
\bottomrule
\end{tabular}
\label{tab:end_to_end_scaling_data}%
\end{table}%


\afterpage{\clearpage}

\begin{figure}[p]
\centering
\includegraphics[width=0.7\linewidth]{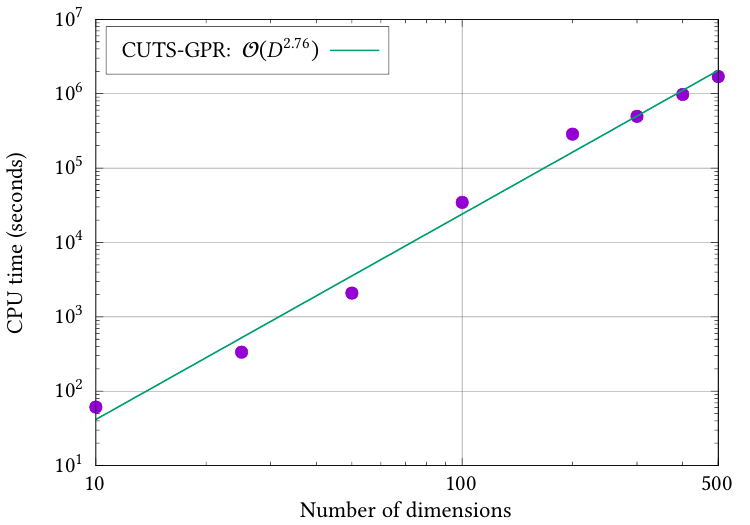}
\caption{Empirical $D$-scaling for CUTS-GPR (based on total CPU time).}
\label{fig:end_to_end_scaling_D_appendix}
\end{figure}

\begin{figure}[p]
\centering
\includegraphics[width=0.7\linewidth]{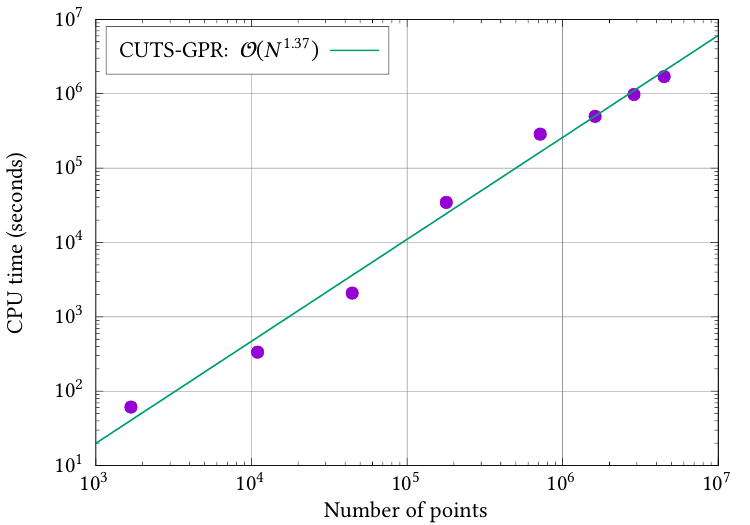}
\caption{Empirical $N$-scaling for CUTS-GPR (based on total CPU time).}
\label{fig:end_to_end_scaling_N_appendix}
\end{figure}


\newpage
\section{Application to PES Data} \label{appendix:application_to_pes}

\setcounter{equation}{0}
\renewcommand{\theequation}{\thesection\arabic{equation}}

\setcounter{figure}{0}
\renewcommand{\thefigure}{\thesection\arabic{figure}}

\setcounter{table}{0}
\renewcommand{\thetable}{\thesection\arabic{table}}

\subsection{Datasets}

As a practical test, we apply CUTS-GPR to the \acp{pes} of the ten molecules in Table~\ref{tab:test_molecules}.
The datasets were generated as in Appendix~\ref{appendix:generating_datasets}, except that
we used a third-order incomplete grid ($\alpha = 3$).
The resulting datasets have $\widehat{N} = \num{447 265}$ training points and $\widehat{N}_* = \num{1 604 576}$ test points.\footnote{The datasets are available at \url{https://anonymous.4open.science/r/DONT_GET_YOUR_KRONECKERS_IN_A_TWIST_DATASETS-F084}.}
We used an \ac{oak} kernel with maximum interaction order $\omega = \alpha = 3$ (see Appendix~\ref{appendix:kernel_centering} for details on kernel centering).

\begin{table}[H]
  \centering
  \caption{Ten organic molecules with ten atoms ($D = 24$).}
  \begin{tabular}{ll}
  \toprule
  Name                      & Chemical formula  \\
  \midrule
  Butadiene                 & $\mathrm{C_4 H_6}$        \\
  DMSO                      & $\mathrm{(CH_{3})_2 SO}$  \\
  Ethylamine                & $\mathrm{CH_3 CH_2 NH_2}$ \\
  Ethylene glycol           & $\mathrm{(CH_2 OH)_2}$    \\
  Nitroethane               & $\mathrm{CH_3 CH_2 NO_2}$ \\
  Propanal                  & $\mathrm{CH_3 CH_2 CHO}$  \\
  Pyrazine                  & $\mathrm{C_4 H_4 N_2}$    \\
  Pyrrole                   & $\mathrm{C_4 H_4 NH}$     \\
  Thioacetone               & $\mathrm{(CH_{3})_2 CS}$  \\
  Vinylformamide            & $\mathrm{C_3 H_5 NO}$     \\
  \bottomrule
  \end{tabular}
  \label{tab:test_molecules}%
  \end{table}%

\subsection{CUTS-GPR settings} \label{appendix:pes_cuts_gpr_settings}

Before running the CUTS-GPR calculations, the training inputs and training outputs were standardized to have zero mean 
and unit variance. After standardization, we used the same settings as in 
Appendices~\ref{appendix:exact_baselines_settings} and \ref{appendix:end_to_end_scaling}
except for the initial order variances, which were chosen according to Table~\ref{tab:initial_guess}.
The hyperparameter gradient norm was converged to a threshold of $\varepsilon_{\mathrm{opt}} = 10^{-3}$.
All CUTS-GPR calculations were run on 36 cores of an AMD EPYC 9565 CPU using standard double precision. 

\begin{table}[H]
  \centering
  \caption{Initial values for the hyperparameter optimization. For each mode, the initial length scale is 
  computed as twice the average distance between neighbouring 1D grid points. This distance is
  computed after the input data is standardized, so identical values are obtained for all modes
  and all molecules.}
  \begin{tabular}{
    c
    S[table-format=1.2e-1, table-text-alignment=center, round-mode=places, round-precision=2]
    }
\toprule
Hyperparameter & {Initial value} \\
\midrule
$\sigma_0^2$   & 1.0000000000000000e-04 \\[0.3em]
$\sigma_1^2$   & 1.2500000000000000e-01 \\[0.3em]
$\sigma_2^2$   & 2.5000000000000000e-01 \\[0.3em]
$\sigma_3^2$   & 5.0000000000000000e-01 \\[0.3em]
$\ell^{(m)}$   & 2.7708419471673569e+00 \\
\bottomrule
\end{tabular}
\label{tab:initial_guess}%
\end{table}%

\subsection{SVGP settings} \label{appendix:svgp_settings}
All calculations were run with the same priors and noise as CUTS-GPR, though inputs were standardized to the range $[0,1]$.
Hyperparameters were initialized with default values and optimized using Adam with a learning rate of $0.1$.
The optimization was run until the change in objective function was less that $0.1$ over $10$ iterations.
All other settings were left as default.
All SVGP calculations were run on 36 cores of an AMD EPYC 9565 CPU using standard double precision. 

\newpage
\subsection{Trace estimation}

\begin{table}[H]
  \centering
  \caption{Stochastic estimates of the trace terms, $\mathrm{tr}\big( \hatbold{C}^{-1} \partial \hatbold{C} / \partial \theta \big)$, for
  iteration 50 of the butadiene calculation. The \acf{sem} is shown as a measure of
  uncertainty. Note that the SEM is approximate since it does not take the possible \ac{cg} error into account.}
  \begin{tabular}{
    c
    S[table-format=-1.2e1, table-text-alignment=center, round-mode=places, round-precision=2, retain-explicit-plus]
    S[table-format=3.2   , table-text-alignment=center, round-mode=places, round-precision=2, scientific-notation=fixed, fixed-exponent=0]
    }
\toprule
Hyperparameter   &  {Estimate}               &  {SEM (\%)}          \\
\midrule
$\sigma_0^2$     &  +3.9927321174331018e3  &  5.18e+02  \\[0.3em]
$\sigma_1^2$     &  +3.4509942767147063e3  &  3.76e+01  \\[0.3em]
$\sigma_2^2$     &  +7.4704353646337171e4  &  7.20e-01  \\[0.3em]
$\sigma_3^2$     &  +5.0290990015020146e5  &  4.76e-02  \\[0.3em]
$\ell^{(1)}$     &  -8.3800570836307943e3  &  4.66e-01  \\[0.3em]
$\ell^{(2)}$     &  -7.4616076727452128e3  &  5.39e-01  \\[0.3em]
$\ell^{(3)}$     &  -1.1003005160047693e4  &  4.21e-01  \\[0.3em]
$\ell^{(4)}$     &  -5.8628199693748311e3  &  5.08e-01  \\[0.3em]
$\ell^{(5)}$     &  -8.0168981828620290e3  &  4.29e-01  \\[0.3em]
$\ell^{(6)}$     &  -9.6120353155569701e3  &  4.22e-01  \\[0.3em]
$\ell^{(7)}$     &  -9.9642368548629802e3  &  3.86e-01  \\[0.3em]
$\ell^{(8)}$     &  -7.5485471675108147e3  &  3.81e-01  \\[0.3em]
$\ell^{(9)}$     &  -8.5855899791369102e3  &  4.06e-01  \\[0.3em]
$\ell^{(10)}$    &  -1.0464668380424895e4  &  4.32e-01  \\[0.3em]
$\ell^{(11)}$    &  -1.0615818718587509e4  &  3.96e-01  \\[0.3em]
$\ell^{(12)}$    &  -5.9717280301740720e3  &  5.17e-01  \\[0.3em]
$\ell^{(13)}$    &  -1.0099587557413146e4  &  3.87e-01  \\[0.3em]
$\ell^{(14)}$    &  -9.2812740621386747e3  &  4.43e-01  \\[0.3em]
$\ell^{(15)}$    &  -7.0949185880162031e3  &  5.70e-01  \\[0.3em]
$\ell^{(16)}$    &  -7.3068146378728925e3  &  4.86e-01  \\[0.3em]
$\ell^{(17)}$    &  -5.8248745015511777e3  &  5.26e-01  \\[0.3em]
$\ell^{(18)}$    &  -6.0113928912824376e3  &  5.01e-01  \\[0.3em]
$\ell^{(19)}$    &  -2.1884328076762322e4  &  2.27e-01  \\[0.3em]
$\ell^{(20)}$    &  -2.2068315008327296e4  &  2.12e-01  \\[0.3em]
$\ell^{(21)}$    &  -2.0090469101475021e4  &  2.29e-01  \\[0.3em]
$\ell^{(22)}$    &  -1.9869700474420486e4  &  2.58e-01  \\[0.3em]
$\ell^{(23)}$    &  -2.0291427375509480e4  &  2.58e-01  \\[0.3em]
$\ell^{(24)}$    &  -2.0407090879189764e4  &  1.86e-01  \\
\bottomrule
\end{tabular}
\label{tab:trace_estimates}%
\end{table}%


\newpage
\subsection{CUTS-GPR timings}

\begin{table}[H]
  \centering
  \caption{Timings for hyperparameter optimization and calculation of the predictive mean in CUTS-GPR. The
  speedup factor ($t_{\mathrm{CPU}} / t_{\mathrm{wall}}$) is 30--34 for hyperparameter optimization and 34--35 for predictions.
  The calculations were run on 36 cores of an AMD EPYC 9565 CPU.}
  \begin{tabular}{
    l
    S[table-format=3.1, table-text-alignment=center, round-mode=places, round-precision=1, scientific-notation=fixed, fixed-exponent=0]
    S[table-format=1.2, table-text-alignment=center, round-mode=places, round-precision=2, scientific-notation=fixed, fixed-exponent=0]
    S[table-format=2.1, table-text-alignment=center, round-mode=places, round-precision=1, scientific-notation=fixed, fixed-exponent=0]
    S[table-format=1.2, table-text-alignment=center, round-mode=places, round-precision=2, scientific-notation=fixed, fixed-exponent=0]
    }
\toprule
                   & \multicolumn{2}{c}{Optimization}            & \multicolumn{2}{c}{Predictive mean}          \\
\cmidrule(lr{2pt}){2-3} \cmidrule(l{2pt}){4-5}
Molecule           &  {CPU (h)}           &  {Wall (h)}          &  {CPU (min)}         &  {Wall (min)}         \\
\midrule
Butadiene          &  6.357767277778e+01  &  1.950290622222e+00  &  6.477271733333e+01  &  1.831024350000e+00   \\
DMSO               &  6.864719055556e+01  &  2.151009277778e+00  &  5.529412200000e+01  &  1.600880648333e+00   \\
Ethylamine         &  5.961469472222e+01  &  1.943133611111e+00  &  6.087157950000e+01  &  1.726027233333e+00   \\
Ethylene glycol    &  7.365132250000e+01  &  2.250172466667e+00  &  5.642641350000e+01  &  1.654655243333e+00   \\
Nitroethane        &  6.421824694444e+01  &  1.907263355556e+00  &  6.583685866667e+01  &  1.876073816667e+00   \\
Propanal           &  7.217412638889e+01  &  2.163175372222e+00  &  5.631967050000e+01  &  1.602794611667e+00   \\
Pyrazine           &  6.560716583333e+01  &  1.956880238889e+00  &  6.515071966667e+01  &  1.840748666667e+00   \\
Pyrrole            &  6.766454138889e+01  &  2.067753933333e+00  &  6.420818566667e+01  &  1.848814216667e+00   \\
Thioacetone        &  1.091027925000e+02  &  3.578578861111e+00  &  5.498183083333e+01  &  1.583184993333e+00   \\
Vinylformamide     &  5.007012972222e+01  &  1.624287236111e+00  &  5.780270433333e+01  &  1.628728315000e+00   \\
\midrule
Mean               &  6.943278833333e+01  &  2.159254497500e+00  &  6.016648020000e+01  &  1.719293209500e+00   \\
\bottomrule
\end{tabular}
\label{tab:timings}%
\end{table}%


\newpage
\subsection{SVGP timings} \label{appendix:svgp_timings}

\begin{table}[H]
  \centering
  \caption{Wall time for hyperparameter optimization and predictions with SVGP. The calculations were
  run on 36 cores of an AMD EPYC 9565 CPU. The third and fourth columns indicate the interaction order, $\omega$, and
  the number of inducing points, $m$. The training batch size is always twice the number of inducing points.}
  \begin{tabular}{
    l
    l
    S[table-format=1.0, table-text-alignment=center]
    S[table-format=4.0, table-text-alignment=center]
    S[table-format=2.2, table-text-alignment=center, round-mode=places, round-precision=2, scientific-notation=fixed, fixed-exponent=0]
    S[table-format=1.2, table-text-alignment=center, round-mode=places, round-precision=2, scientific-notation=fixed, fixed-exponent=0]
    }
\toprule
Molecule         & Method   &  {$\omega$} & {$m$}   &  {Optimization (h)}     &  {Predictions (min)}  \\ 
\midrule
Butadiene        & SVGP     & 2         &	 256  &  5.804222e-01  &  4.688333e-01 \\[0.1em]
                 &          &           &	 512  &  9.219361e-01  &  6.900000e-01 \\[0.1em]
                 &          &           &	1024  &  1.610358e+00  &  1.221000e+00 \\[0.1em]
                 &          &           &	2048  &  2.023533e+00  &  2.311500e+00 \\[0.1em]
                 & SVGP     & 3         &	 256  &  1.576906e+00  &  1.396667e+00 \\[0.1em]
                 &          &           &	 512  &  2.277631e+00  &  2.319500e+00 \\[0.1em]
                 &          &           &	1024  &  7.652858e+00  &  4.080667e+00 \\[0.1em]
                 &          &           &	2048  &  1.021530e+01  &  8.400167e+00 \\[0.1em]\midrule
DMSO             & SVGP     & 2         &	 256  &  5.861389e-01  &  4.761667e-01 \\[0.1em]
                 &          &           &	 512  &  5.781750e-01  &  6.766667e-01 \\[0.1em]
                 &          &           &	1024  &  8.581528e-01  &  1.113833e+00 \\[0.1em]
                 &          &           &	2048  &  1.812819e+00  &  2.303333e+00 \\[0.1em]
                 & SVGP     & 3         &	 256  &  4.393086e+00  &  1.229333e+00 \\[0.1em]
                 &          &           &	 512  &  2.998842e+00  &  2.337667e+00 \\[0.1em]
                 &          &           &	1024  &  4.085011e+00  &  4.125500e+00 \\[0.1em]
                 &          &           &	2048  &  8.239606e+00  &  8.416833e+00 \\[0.1em]\midrule
Ethylamine       & SVGP     & 2         &	 256  &  8.456694e-01  &  4.838333e-01 \\[0.1em]
                 &          &           &	 512  &  9.020944e-01  &  7.133333e-01 \\[0.1em]
                 &          &           &	1024  &  9.708639e-01  &  1.345500e+00 \\[0.1em]
                 &          &           &	2048  &  2.457431e+00  &  2.382167e+00 \\[0.1em]
                 & SVGP     & 3         &	 256  &  2.938975e+00  &  1.477167e+00 \\[0.1em]
                 &          &           &	 512  &  2.166392e+00  &  2.041833e+00 \\[0.1em]
                 &          &           &	1024  &  5.198214e+00  &  3.739167e+00 \\[0.1em]
                 &          &           &	2048  &  1.123144e+01  &  8.315500e+00 \\[0.1em]\midrule
Ethylene glycol  & SVGP     & 2         &	 256  &  4.078139e-01  &  4.761667e-01 \\[0.1em]
                 &          &           &	 512  &  1.149000e+00  &  6.646667e-01 \\[0.1em]
                 &          &           &	1024  &  1.401278e+00  &  1.158833e+00 \\[0.1em]
                 &          &           &	2048  &  2.035106e+00  &  2.428833e+00 \\[0.1em]
                 & SVGP     & 3         &	 256  &  4.161792e+00  &  1.207833e+00 \\[0.1em]
                 &          &           &	 512  &  4.078967e+00  &  1.983167e+00 \\[0.1em]
                 &          &           &	1024  &  7.906875e+00  &  3.198833e+00 \\[0.1em]
                 &          &           &	2048  &  1.216186e+01  &  8.525167e+00 \\[0.1em]\midrule
Nitroethane      & SVGP     & 2         &	 256  &  7.214139e-01  &  5.110000e-01 \\[0.1em]
                 &          &           &	 512  &  4.794583e-01  &  6.920000e-01 \\[0.1em]
                 &          &           &	1024  &  1.028928e+00  &  1.300000e+00 \\[0.1em]
                 &          &           &	2048  &  1.910275e+00  &  2.396333e+00 \\[0.1em]
                 & SVGP     & 3         &	 256  &  2.353772e+00  &  1.374833e+00 \\[0.1em]
                 &          &           &	 512  &  3.124008e+00  &  2.341000e+00 \\[0.1em]
                 &          &           &	1024  &  5.672003e+00  &  3.381667e+00 \\[0.1em]
                 &          &           &	2048  &  9.758483e+00  &  8.505000e+00 \\
\bottomrule
\end{tabular}
\label{tab:timings_all_a}%
\end{table}%

\afterpage{\clearpage}
\begin{table}[p]
  \centering
  \caption{Wall time for hyperparameter optimization and predictions with SVGP. The calculations were
  run on 36 cores of an AMD EPYC 9565 CPU. The third and fourth columns indicate the interaction order, $\omega$, and
  the number of inducing points, $m$. The training batch size is always twice the number of inducing points.}
  \begin{tabular}{
    l
    l
    S[table-format=1.0, table-text-alignment=center]
    S[table-format=4.0, table-text-alignment=center]
    S[table-format=2.2, table-text-alignment=center, round-mode=places, round-precision=2, scientific-notation=fixed, fixed-exponent=0]
    S[table-format=1.2, table-text-alignment=center, round-mode=places, round-precision=2, scientific-notation=fixed, fixed-exponent=0]
    }
\toprule
Molecule         & Method   &  {$\omega$} & {$m$}   &  {Optimization (h)}     &  {Predictions (min)}  \\ 
\midrule
Propanal         & SVGP     & 2         &	 256  &  3.327083e-01  &  5.306667e-01 \\[0.1em]
                 &          &           &	 512  &  5.517611e-01  &  6.640000e-01 \\[0.1em]
                 &          &           &	1024  &  8.613833e-01  &  1.300167e+00 \\[0.1em]
                 &          &           &	2048  &  2.121272e+00  &  2.340500e+00 \\[0.1em]
                 & SVGP     & 3         &	 256  &  2.180139e+00  &  1.428667e+00 \\[0.1em]
                 &          &           &	 512  &  2.692981e+00  &  2.277000e+00 \\[0.1em]
                 &          &           &	1024  &  4.082942e+00  &  4.222500e+00 \\[0.1em]
                 &          &           &	2048  &  1.074096e+01  &  8.273000e+00 \\[0.1em]\midrule
Pyrazine         & SVGP     & 2         &	 256  &  8.519833e-01  &  5.313333e-01 \\[0.1em]
                 &          &           &	 512  &  4.628778e-01  &  6.531667e-01 \\[0.1em]
                 &          &           &	1024  &  7.554417e-01  &  1.313333e+00 \\[0.1em]
                 &          &           &	2048  &  1.594594e+00  &  2.391333e+00 \\[0.1em]
                 & SVGP     & 3         &	 256  &  1.679964e+00  &  1.278833e+00 \\[0.1em]
                 &          &           &	 512  &  3.568706e+00  &  2.327000e+00 \\[0.1em]
                 &          &           &	1024  &  4.084750e+00  &  4.251333e+00 \\[0.1em]
                 &          &           &	2048  &  7.810903e+00  &  8.236833e+00 \\[0.1em]\midrule
Pyrrole          & SVGP     & 2         &	 256  &  5.705861e-01  &  5.198333e-01 \\[0.1em]
                 &          &           &	 512  &  5.181722e-01  &  7.215000e-01 \\[0.1em]
                 &          &           &	1024  &  9.190306e-01  &  1.267667e+00 \\[0.1em]
                 &          &           &	2048  &  2.228411e+00  &  2.309000e+00 \\[0.1em]
                 & SVGP     & 3         &	 256  &  2.413397e+00  &  1.453167e+00 \\[0.1em]
                 &          &           &	 512  &  3.957408e+00  &  2.239000e+00 \\[0.1em]
                 &          &           &	1024  &  5.541931e+00  &  3.321833e+00 \\[0.1em]
                 &          &           &	2048  &  9.764067e+00  &  8.448500e+00 \\[0.1em]\midrule
Thioacetone      & SVGP     & 2         &	 256  &  3.489556e-01  &  5.166667e-01 \\[0.1em]
                 &          &           &	 512  &  4.881917e-01  &  6.396667e-01 \\[0.1em]
                 &          &           &	1024  &  8.620583e-01  &  1.247000e+00 \\[0.1em]
                 &          &           &	2048  &  1.823111e+00  &  2.300500e+00 \\[0.1em]
                 & SVGP     & 3         &	 256  &  2.014239e+00  &  1.257500e+00 \\[0.1em]
                 &          &           &	 512  &  2.064281e+00  &  2.282000e+00 \\[0.1em]
                 &          &           &	1024  &  4.068753e+00  &  3.994667e+00 \\[0.1em]
                 &          &           &	2048  &  8.802981e+00  &  8.224333e+00 \\[0.1em]\midrule
Vinylformamide   & SVGP     & 2         &	 256  &  4.223056e-01  &  4.851667e-01 \\[0.1em]
                 &          &           &	 512  &  6.387361e-01  &  6.568333e-01 \\[0.1em]
                 &          &           &	1024  &  1.076358e+00  &  1.307333e+00 \\[0.1em]
                 &          &           &	2048  &  2.028844e+00  &  2.263500e+00 \\[0.1em]
                 & SVGP     & 3         &	 256  &  1.743700e+00  &  1.383833e+00 \\[0.1em]
                 &          &           &	 512  &  4.305414e+00  &  2.413833e+00 \\[0.1em]
                 &          &           &	1024  &  6.049919e+00  &  3.418833e+00 \\[0.1em]
                 &          &           &	2048  &  1.085966e+01  &  7.067167e+00 \\
\bottomrule
\end{tabular}
\label{tab:timings_all_b}%
\end{table}%


\newpage
\subsection{Comparison of average errors} \label{appendix:average_errors}

\begin{table}[H]
  \centering
  \caption{Comparison of range normalized prediction errors (dimensionless).
  The second and third columns indicate the interaction order, $\omega$, and
  the number of inducing points, $m$, which is only applicable for the SVGP models.
  For each model, the average over the ten molecules in Table~\ref{tab:test_molecules} 
  is shown (with the standard deviation in parentheses).}
  \begin{tabular}{
    l
    S[table-format=1.0, table-text-alignment=center]
    S[table-format=4.0, table-text-alignment=center]
    S[table-format=1.3, table-text-alignment=center, round-mode=places, round-precision=3, table-space-text-pre={}]@{\hskip 0.3em}
    S[table-format=1.3, table-text-alignment=center, round-mode=places, round-precision=3, table-space-text-pre={(}, table-space-text-post={))}]
    S[table-format=1.2, table-text-alignment=center, round-mode=places, round-precision=2, table-space-text-pre={(}]@{\hskip 0.3em}
    S[table-format=1.2, table-text-alignment=center, round-mode=places, round-precision=2, table-space-text-pre={(}, table-space-text-post={))}]
    S[table-format=2.2, table-text-alignment=center, round-mode=places, round-precision=2, table-space-text-pre={(}]@{\hskip 0.3em}
    S[table-format=1.2, table-text-alignment=center, round-mode=places, round-precision=2, table-space-text-pre={(}, table-space-text-post={))}]
    }
\toprule
Method           & {$\omega$} & {$m$}    & \multicolumn{2}{c}{MAX}       & \multicolumn{2}{c}{MAE / $10^{-3}$}  & \multicolumn{2}{c}{RMSE / $10^{-3}$} \\
\midrule
SVGP             & 2          &  256     &  0.709837  & {(}0.110266{)}   &    7.333800  & {(}3.266813{)}        &    15.136000  & {(}7.346114{)}  \\[0.2em]
                 &            &  512     &  0.640857  & {(}0.109606{)}   &    6.359300  & {(}2.780215{)}        &    12.844900  & {(}6.193683{)}  \\[0.2em]
                 &            & 1024     &  0.570483  & {(}0.097676{)}   &    4.958600  & {(}1.950456{)}        &     9.678800  & {(}3.827351{)}  \\[0.2em]
                 &            & 2048     &  0.543365  & {(}0.072914{)}   &    4.023500  & {(}1.728422{)}        &     8.037100  & {(}3.177947{)}  \\[0.2em]
SVGP             & 3          &  256     &  0.631317  & {(}0.114443{)}   &    6.728800  & {(}2.322306{)}        &    12.101400  & {(}4.082859{)}  \\[0.2em]
                 &            &  512     &  0.618137  & {(}0.127565{)}   &    6.125700  & {(}2.414038{)}        &    10.712800  & {(}3.978293{)}  \\[0.2em]
                 &            & 1024     &  0.552184  & {(}0.128900{)}   &    4.740900  & {(}1.865272{)}        &     8.831000  & {(}3.301693{)}  \\[0.2em]
                 &            & 2048     &  0.530332  & {(}0.104711{)}   &    3.987200  & {(}1.879534{)}        &     7.578000  & {(}2.960273{)}  \\[0.2em]
CUTS-GPR         & 3          & {---}    &  0.168000  & {(}0.104815{)}   &    1.576339  & {(}0.843056{)}        &     4.307758  & {(}2.257304{)}  \\
\bottomrule
\end{tabular}
\label{tab:errors_compare}%
\end{table}%


\begin{figure}[H]
\centering
\includegraphics[width=0.7\linewidth]{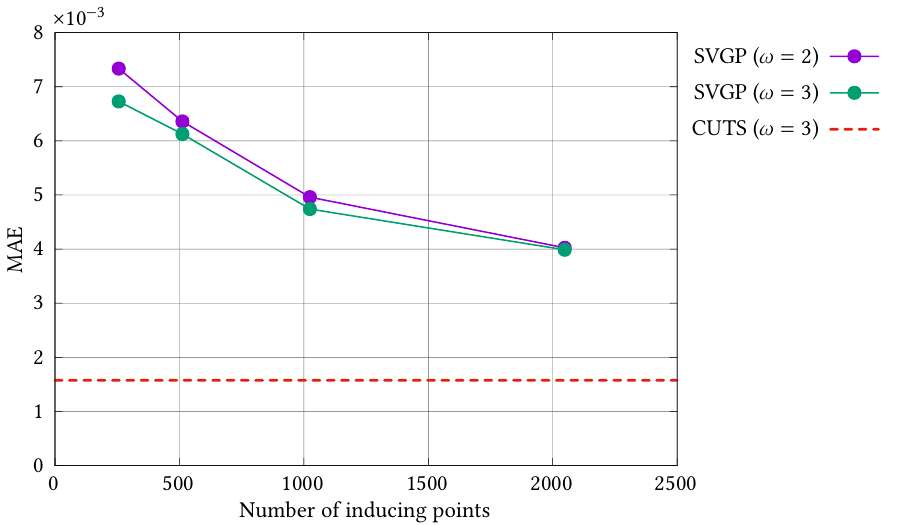}
\caption{Range normalized MAE as a function of the number of inducing points. The errors are averaged over the
ten molecules in Table~\ref{tab:test_molecules}.}
\label{fig:svgp_mae_errors}
\end{figure}


\newpage
\subsection{Errors for all models and molecules} \label{appendix:all_errors}

\begin{table}[H]
  \centering
  \caption{Range normalized prediction errors (dimensionless).}
  \begin{tabular}{
    l
    l
    S[table-format=1.0, table-text-alignment=center]
    S[table-format=4.0, table-text-alignment=center]
    S[table-format=1.3, table-text-alignment=center, scientific-notation = fixed, fixed-exponent = 0, round-mode=places, round-precision=3]
    S[table-format=1.2e-1, scientific-notation = true, table-text-alignment=center, round-mode=places, round-precision=2]
    S[table-format=1.2e-1, scientific-notation = true, table-text-alignment=center, round-mode=places, round-precision=2]
    }
\toprule
Molecule         & Method   &  {$\omega$} & $m$       &  {MAX}     &  {MAE}     &  {RMSE}   \\ 
\midrule
Butadiene        & SVGP     & 2         &   256     &  0.823490  &  0.004970  &  0.011009 \\[0.0em]
                 &          &           &   512     &  0.604848  &  0.004620  &  0.008437 \\[0.0em]
                 &          &           &  1024     &  0.674962  &  0.003847  &  0.007314 \\[0.0em]
                 &          &           &  2048     &  0.581632  &  0.003088  &  0.006365 \\[0.0em]
                 & SVGP     & 3         &   256     &  0.720637  &  0.004738  &  0.009033 \\[0.0em]
                 &          &           &   512     &  0.787404  &  0.004660  &  0.008544 \\[0.0em]
                 &          &           &  1024     &  0.545296  &  0.003483  &  0.006778 \\[0.0em]
                 &          &           &  2048     &  0.490465  &  0.002993  &  0.005816 \\[0.0em]
                 & CUTS-GPR & 3         & {---}     &  0.123374  &  0.000988  &  0.002907 \\[0.0em]\midrule
DMSO             & SVGP     & 2         &   256     &  0.693863  &  0.004678  &  0.008568 \\[0.0em]
                 &          &           &   512     &  0.684102  &  0.004691  &  0.008645 \\[0.0em]
                 &          &           &  1024     &  0.727272  &  0.003916  &  0.007231 \\[0.0em]
                 &          &           &  2048     &  0.606342  &  0.003232  &  0.005844 \\[0.0em]
                 & SVGP     & 3         &   256     &  0.766009  &  0.004649  &  0.008071 \\[0.0em]
                 &          &           &   512     &  0.732017  &  0.005175  &  0.007868 \\[0.0em]
                 &          &           &  1024     &  0.748619  &  0.003578  &  0.006196 \\[0.0em]
                 &          &           &  2048     &  0.607163  &  0.002916  &  0.005456 \\[0.0em]
                 & CUTS-GPR & 3         & {---}     &  0.120158  &  0.000745  &  0.001760 \\[0.0em]\midrule
Ethylamine       & SVGP     & 2         &   256     &  0.769491  &  0.007047  &  0.018287 \\[0.0em]
                 &          &           &   512     &  0.505623  &  0.006261  &  0.011903 \\[0.0em]
                 &          &           &  1024     &  0.516092  &  0.004797  &  0.010224 \\[0.0em]
                 &          &           &  2048     &  0.541808  &  0.003860  &  0.008250 \\[0.0em]
                 & SVGP     & 3         &   256     &  0.565894  &  0.008347  &  0.014309 \\[0.0em]
                 &          &           &   512     &  0.565033  &  0.006330  &  0.011379 \\[0.0em]
                 &          &           &  1024     &  0.480472  &  0.005187  &  0.010119 \\[0.0em]
                 &          &           &  2048     &  0.492620  &  0.003760  &  0.007716 \\[0.0em]
                 & CUTS-GPR & 3         & {---}     &  0.129442  &  0.001511  &  0.004179 \\[0.0em]\midrule
Ethylene glycol  & SVGP     & 2         &   256     &  0.465064  &  0.012691  &  0.021684 \\[0.0em]
                 &          &           &   512     &  0.507843  &  0.011965  &  0.020370 \\[0.0em]
                 &          &           &  1024     &  0.557742  &  0.009307  &  0.017637 \\[0.0em]
                 &          &           &  2048     &  0.540800  &  0.008120  &  0.015777 \\[0.0em]
                 & SVGP     & 3         &   256     &  0.391387  &  0.011775  &  0.021151 \\[0.0em]
                 &          &           &   512     &  0.493636  &  0.011632  &  0.019745 \\[0.0em]
                 &          &           &  1024     &  0.441998  &  0.009074  &  0.016484 \\[0.0em]
                 &          &           &  2048     &  0.462377  &  0.008434  &  0.014911 \\[0.0em]
                 & CUTS-GPR & 3         & {---}     &  0.151369  &  0.003755  &  0.010184 \\[0.0em]\midrule
Nitroethane      & SVGP     & 2         &   256     &  0.564943  &  0.005852  &  0.010268 \\[0.0em]
                 &          &           &   512     &  0.484045  &  0.004728  &  0.008826 \\[0.0em]
                 &          &           &  1024     &  0.424802  &  0.004076  &  0.007932 \\[0.0em]
                 &          &           &  2048     &  0.433927  &  0.003164  &  0.006649 \\[0.0em]
                 & SVGP     & 3         &   256     &  0.511396  &  0.005766  &  0.009958 \\[0.0em]
                 &          &           &   512     &  0.531693  &  0.005267  &  0.009086 \\[0.0em]
                 &          &           &  1024     &  0.455765  &  0.003667  &  0.007306 \\[0.0em]
                 &          &           &  2048     &  0.449288  &  0.003011  &  0.006177 \\[0.0em]
                 & CUTS-GPR & 3         & {---}     &  0.138868  &  0.001493  &  0.003802 \\
\bottomrule
\end{tabular}
\label{tab:errors_all_a}%
\end{table}%

\begin{table}[p]
  \centering
  \caption{Range normalized prediction errors (dimensionless).}
  \begin{tabular}{
    l
    l
    S[table-format=1.0, table-text-alignment=center]
    S[table-format=4.0, table-text-alignment=center]
    S[table-format=1.3, table-text-alignment=center, scientific-notation = fixed, fixed-exponent = 0, round-mode=places, round-precision=3]
    S[table-format=1.2e-1, scientific-notation = true, table-text-alignment=center, round-mode=places, round-precision=2]
    S[table-format=1.2e-1, scientific-notation = true, table-text-alignment=center, round-mode=places, round-precision=2]
    }
\toprule
Molecule         & Method   &  {$\omega$} & $m$       &  {MAX}     &  {MAE}     &  {RMSE}   \\ 
\midrule
Propanal         & SVGP     & 2         &   256     &  0.761022  &  0.007881  &  0.016713 \\[0.0em]
                 &          &           &   512     &  0.763852  &  0.006689  &  0.015249 \\[0.0em]
                 &          &           &  1024     &  0.440491  &  0.005259  &  0.010334 \\[0.0em]
                 &          &           &  2048     &  0.441731  &  0.004688  &  0.008866 \\[0.0em]
                 & SVGP     & 3         &   256     &  0.610449  &  0.007656  &  0.014105 \\[0.0em]
                 &          &           &   512     &  0.402601  &  0.007272  &  0.013008 \\[0.0em]
                 &          &           &  1024     &  0.379408  &  0.004790  &  0.009427 \\[0.0em]
                 &          &           &  2048     &  0.357303  &  0.004087  &  0.008054 \\[0.0em]
                 & CUTS-GPR & 3         & {---}     &  0.115197  &  0.001805  &  0.004614 \\[0.0em]\midrule
Pyrazine         & SVGP     & 2         &   256     &  0.746465  &  0.006298  &  0.012466 \\[0.0em]
                 &          &           &   512     &  0.685751  &  0.004525  &  0.009718 \\[0.0em]
                 &          &           &  1024     &  0.592029  &  0.003193  &  0.006973 \\[0.0em]
                 &          &           &  2048     &  0.597922  &  0.002638  &  0.005878 \\[0.0em]
                 & SVGP     & 3         &   256     &  0.668857  &  0.007312  &  0.012860 \\[0.0em]
                 &          &           &   512     &  0.716201  &  0.003637  &  0.007372 \\[0.0em]
                 &          &           &  1024     &  0.584247  &  0.003897  &  0.006771 \\[0.0em]
                 &          &           &  2048     &  0.592391  &  0.002600  &  0.005940 \\[0.0em]
                 & CUTS-GPR & 3         & {---}     &  0.161909  &  0.001026  &  0.003004 \\[0.0em]\midrule
Pyrrole          & SVGP     & 2         &   256     &  0.754684  &  0.004924  &  0.009746 \\[0.0em]
                 &          &           &   512     &  0.748586  &  0.003997  &  0.008568 \\[0.0em]
                 &          &           &  1024     &  0.644515  &  0.003498  &  0.006563 \\[0.0em]
                 &          &           &  2048     &  0.613950  &  0.002223  &  0.005168 \\[0.0em]
                 & SVGP     & 3         &   256     &  0.732082  &  0.004090  &  0.007847 \\[0.0em]
                 &          &           &   512     &  0.710308  &  0.003783  &  0.006799 \\[0.0em]
                 &          &           &  1024     &  0.710940  &  0.003067  &  0.005871 \\[0.0em]
                 &          &           &  2048     &  0.616195  &  0.002596  &  0.005260 \\[0.0em]
                 & CUTS-GPR & 3         & {---}     &  0.135887  &  0.001100  &  0.003514 \\[0.0em]\midrule
Thioacetone      & SVGP     & 2         &   256     &  0.780264  &  0.013734  &  0.032129 \\[0.0em]
                 &          &           &   512     &  0.764516  &  0.010782  &  0.026724 \\[0.0em]
                 &          &           &  1024     &  0.602710  &  0.007472  &  0.015222 \\[0.0em]
                 &          &           &  2048     &  0.615489  &  0.005490  &  0.010618 \\[0.0em]
                 & SVGP     & 3         &   256     &  0.675996  &  0.007692  &  0.014096 \\[0.0em]
                 &          &           &   512     &  0.706256  &  0.008254  &  0.014053 \\[0.0em]
                 &          &           &  1024     &  0.700989  &  0.006753  &  0.011870 \\[0.0em]
                 &          &           &  2048     &  0.722946  &  0.006131  &  0.010012 \\[0.0em]
                 & CUTS-GPR & 3         & {---}     &  0.463559  &  0.001716  &  0.004722 \\[0.0em]\midrule
Vinylformamide   & SVGP     & 2         &   256     &  0.739084  &  0.005263  &  0.010490 \\[0.0em]
                 &          &           &   512     &  0.659402  &  0.005335  &  0.010009 \\[0.0em]
                 &          &           &  1024     &  0.524211  &  0.004221  &  0.007358 \\[0.0em]
                 &          &           &  2048     &  0.460047  &  0.003732  &  0.006956 \\[0.0em]
                 & SVGP     & 3         &   256     &  0.670468  &  0.005263  &  0.009584 \\[0.0em]
                 &          &           &   512     &  0.536221  &  0.005247  &  0.009274 \\[0.0em]
                 &          &           &  1024     &  0.474101  &  0.003913  &  0.007488 \\[0.0em]
                 &          &           &  2048     &  0.512576  &  0.003344  &  0.006438 \\[0.0em]
                 & CUTS-GPR & 3         & {---}     &  0.140235  &  0.001624  &  0.004392 \\
\bottomrule
\end{tabular}
\label{tab:errors_all_b}%
\end{table}%

\end{document}

%% file: abbreviations.tex
\usepackage{acro}

\DeclareAcronym{rbf}{
   short = RBF ,
   long = radial basis function ,
}

\DeclareAcronym{svgp}{
   short = SVGP ,
   long = stochastic variational Gaussian process ,
}

\DeclareAcronym{gpr}{
   short = GPR ,
   long = Gaussian process regression ,
}

\DeclareAcronym{cud}{
   short = CUD ,
   long = closed under decrements ,
}

\DeclareAcronym{cuts}{
   short = CUTS ,
   long = closed under taking subsets ,
}

\DeclareAcronym{mc}{
   short = MC ,
   long = mode combination ,
}

\DeclareAcronym{mcr}{
   short = MCR ,
   long = mode combination range ,
}

\DeclareAcronym{mll}{
   short = MLL ,
   long = marginal log-likelihood ,
}

\DeclareAcronym{map}{
   short = MAP ,
   long = maximum \textit{a posteriori} ,
}

\DeclareAcronym{sop}{
   short = SOP ,
   long = sum-of-products ,
}

\DeclareAcronym{oak}{
   short = OAK ,
   long = orthogonal additive kernel ,
}

\DeclareAcronym{mvp}{
   short = MVP ,
   long = matrix-vector product ,
}

\DeclareAcronym{cg}{
   short = CG ,
   long = conjugate gradient ,
}

\DeclareAcronym{spd}{
   short = SPD ,
   long = symmetric positive definite ,
}

\DeclareAcronym{pes}{
   short = PES ,
   long = potential energy surface ,
}

\DeclareAcronym{pdf}{
   short = PDF ,
   long = probability density function ,
}

\DeclareAcronym{sem}{
   short = SEM ,
   long = standard error of the mean ,
}

\DeclareAcronym{elbo}{
   short = ELBO,
   long = evidence lower bound ,
}

%% file: mads.bib
@article{aquilanteAnalyticDerivativesCholesky2008,
  title = {Analytic Derivatives for the {{Cholesky}} Representation of the Two-Electron Integrals},
  author = {Aquilante, Francesco and Lindh, Roland and Pedersen, Thomas Bondo},
  year = 2008,
  month = jul,
  journal = {The Journal of Chemical Physics},
  volume = {129},
  number = {3},
  pages = {034106},
  issn = {0021-9606, 1089-7690},
  doi = {10.1063/1.2955755},
  urldate = {2026-02-09},
  langid = {english}
}

@inproceedings{balandatBoTorchFrameworkEfficient2020,
  title = {{{BoTorch}}: {{A Framework}} for {{Efficient Monte-Carlo Bayesian Optimization}}},
  booktitle = {Advances in {{Neural Information Processing Systems}}},
  author = {Balandat, Maximilian and Karrer, Brian and Jiang, Daniel R and Daulton, Samuel and Letham, Benjamin and Wilson, Andrew Gordon and Bakshy, Eytan},
  year = 2020,
  langid = {english}
}

@article{bannwarthExtendedTightbindingQuantum2021,
  title = {Extended Tight-Binding Quantum Chemistry Methods},
  shorttitle = {Extended},
  author = {Bannwarth, Christoph and Caldeweyher, Eike and Ehlert, Sebastian and Hansen, Andreas and Pracht, Philipp and Seibert, Jakob and Spicher, Sebastian and Grimme, Stefan},
  year = 2021,
  month = mar,
  journal = {WIREs Comput. Mol. Sci.},
  volume = {11},
  number = {2},
  pages = {e1493},
  issn = {1759-0876, 1759-0884},
  doi = {10.1002/wcms.1493},
  urldate = {2026-04-20},
  langid = {english}
}

@article{bannwarthGFN2xTBAnAccurateBroadly2019,
  title = {{{GFN2-xTB}}---{{An Accurate}} and {{Broadly Parametrized Self-Consistent Tight-Binding Quantum Chemical Method}} with {{Multipole Electrostatics}} and {{Density-Dependent Dispersion Contributions}}},
  author = {Bannwarth, Christoph and Ehlert, Sebastian and Grimme, Stefan},
  year = 2019,
  month = mar,
  journal = {J. Chem. Theory Comput.},
  volume = {15},
  number = {3},
  pages = {1652--1671},
  issn = {1549-9618, 1549-9626},
  doi = {10.1021/acs.jctc.8b01176},
  urldate = {2026-04-20},
  copyright = {http://pubs.acs.org/page/policy/authorchoice\_termsofuse.html},
  langid = {english}
}

@article{beebeSimplificationsGenerationTransformation1977,
  title = {Simplifications in the Generation and Transformation of Two-electron Integrals in Molecular Calculations},
  author = {Beebe, Nelson H. F. and Linderberg, Jan},
  year = 1977,
  month = oct,
  journal = {Int. J. Quantum Chem.},
  volume = {12},
  number = {4},
  pages = {683--705},
  issn = {0020-7608, 1097-461X},
  doi = {10.1002/qua.560120408},
  urldate = {2026-01-21},
  copyright = {http://onlinelibrary.wiley.com/termsAndConditions\#vor},
  langid = {english}
}

@article{bowmanMULTIMODECodeCalculate2003,
  title = {{{MULTIMODE}}: {{A}} Code to Calculate Rovibrational Energies of Polyatomic Molecules},
  shorttitle = {{{MULTIMODE}}},
  author = {Bowman, Joel M. and Carter, Stuart and Huang, Xinchuan},
  year = 2003,
  month = jul,
  journal = {Int. Rev. Phys. Chem.},
  volume = {22},
  number = {3},
  pages = {533--549},
  issn = {0144-235X, 1366-591X},
  doi = {10.1080/0144235031000124163},
  urldate = {2026-02-28},
  langid = {english}
}

@article{bungartzSparseGrids2004,
  title = {Sparse Grids},
  author = {Bungartz, Hans-Joachim and Griebel, Michael},
  year = 2004,
  journal = {Acta Numer.},
  volume = {13},
  pages = {147--269},
  issn = {0962-4929, 1474-0508},
  doi = {10.1017/S0962492904000182},
  urldate = {2026-02-23},
  copyright = {https://www.cambridge.org/core/terms},
  langid = {english}
}

@inproceedings{calandraManifoldGaussianProcesses2016,
  title = {Manifold {{Gaussian Processes}} for Regression},
  booktitle = {2016 {{International Joint Conference}} on {{Neural Networks}} ({{IJCNN}})},
  author = {Calandra, Roberto and Peters, Jan and Rasmussen, Carl Edward and Deisenroth, Marc Peter},
  year = 2016,
  month = jul,
  pages = {3338--3345},
  publisher = {IEEE},
  address = {Vancouver, BC, Canada},
  doi = {10.1109/IJCNN.2016.7727626},
  urldate = {2026-07-29},
  isbn = {978-1-5090-0620-5},
  langid = {english}
}

@article{carterVibrationalSelfconsistentField1997,
  title = {Vibrational Self-Consistent Field Method for Many-Mode Systems: {{A}} New Approach and Application to the Vibrations of {{CO}} Adsorbed on {{Cu}}(100)},
  shorttitle = {Vibrational Self-Consistent Field Method for Many-Mode Systems},
  author = {Carter, Stuart and Culik, Susan J. and Bowman, Joel M.},
  year = 1997,
  month = dec,
  journal = {J. Chem. Phys.},
  volume = {107},
  number = {24},
  pages = {10458--10469},
  issn = {0021-9606, 1089-7690},
  doi = {10.1063/1.474210},
  urldate = {2026-02-28},
  langid = {english}
}

@book{chenLanczosAlgorithmMatrix2024,
  title = {The {{Lanczos}} Algorithm for Matrix Functions: A Handbook for Scientists},
  author = {Chen, Tyler},
  year = 2024,
  doi = {10.48550/arXiv.2410.11090},
  langid = {english}
}

@article{chenRandomlyPivotedCholesky2025,
  title = {Randomly Pivoted {{Cholesky}}: {{Practical}} Approximation of a Kernel Matrix with Few Entry Evaluations},
  shorttitle = {Randomly Pivoted {{Cholesky}}},
  author = {Chen, Yifan and Epperly, Ethan N. and Tropp, Joel A. and Webber, Robert J.},
  year = 2025,
  journal = {Communications on Pure and Applied Mathematics},
  volume = {78},
  number = {5},
  pages = {995--1041},
  issn = {1097-0312},
  doi = {10.1002/cpa.22234},
  urldate = {2025-08-15},
  langid = {english}
}

@misc{christiansenMidasCppMolecularInteractions2025,
  title = {{{MidasCpp}} ({{Molecular Interactions}}, {{Dynamics}} and {{Simulation Chemistry Program Package}})},
  author = {Christiansen, Ove and Artiukhin, Denis G. and Bader, Frederik and Godtliebsen, Ian Heide and Gras, Eduard Matito and Gy{\H o}rffy, Werner and Hansen, Mikkel Bo and Hansen, Mads B{\o}ttger and H{\o}jlund, Mads Greisen and H{\o}yer, Nicolai Machholdt and Jensen, Rasmus Berg and Jensen, Andreas Buchgraitz and Klinting, Emil Lund and Kongsted, Jacob and K{\"o}nig, Carolin and Losilla, Sergio Alberto and Madsen, Diana and Madsen, Niels Kristian and Majland, Marco and Monrad, Kasper and {Lykke-M{\o}ller}, August Smart and Schmitz, Gunnar and Seidler, Peter and Sneskov, Kristian and Sparta, Manuel and Thomsen, Bo and Thomsen, Sebastian Riis and Toffoli, Daniele and Zoccante, Alberto},
  year = 2025,
  month = oct,
  address = {Aarhus University}
}

@article{christiansenSecondQuantizationFormulation2004,
  title = {A Second Quantization Formulation of Multimode Dynamics},
  author = {Christiansen, Ove},
  year = 2004,
  month = feb,
  journal = {J. Chem. Phys.},
  volume = {120},
  number = {5},
  pages = {2140--2148},
  issn = {0021-9606, 1089-7690},
  doi = {10.1063/1.1637578},
  urldate = {2022-05-17},
  langid = {english}
}

@article{constantineActiveSubspaceMethods2014,
  title = {Active {{Subspace Methods}} in {{Theory}} and {{Practice}}: {{Applications}} to {{Kriging Surfaces}}},
  shorttitle = {Active {{Subspace Methods}} in {{Theory}} and {{Practice}}},
  author = {Constantine, Paul G. and Dow, Eric and Wang, Qiqi},
  year = 2014,
  month = jan,
  journal = {SIAM J. Sci. Comput.},
  volume = {36},
  number = {4},
  pages = {A1500-A1524},
  issn = {1064-8275, 1095-7197},
  doi = {10.1137/130916138},
  urldate = {2026-07-29},
  langid = {english}
}

@inproceedings{deisenrothDistributedGaussianProcesses2015,
  title = {Distributed Gaussian Processes},
  booktitle = {Proceedings of the 32nd International Conference on Machine Learning},
  author = {Deisenroth, Marc and Ng, Jun Wei},
  year = 2015,
  month = jul,
  series = {Proceedings of Machine Learning Research},
  publisher = {PMLR},
  address = {Lille, France}
}

@article{durrandeAdditiveCovarianceKernels2012,
  title = {Additive Covariance Kernels for High-Dimensional Gaussian Process Modeling},
  author = {Durrande, Nicolas and Ginsbourger, David and Carraro, Laurent},
  year = 2012,
  month = apr,
  journal = {Annales de la facult\'e des sciences de Toulouse Math\'ematiques},
  volume = {21},
  number = {3},
  pages = {481--499},
  publisher = {Universit\'e Paul Sabatier, Toulouse},
  doi = {10.48550/arXiv.1111.6233},
  langid = {english}
}

@inproceedings{duvenaudAdditiveGaussianProcesses2011,
  title = {Additive {{Gaussian Processes}}},
  booktitle = {Advances in {{Neural Information Processing Systems}}},
  author = {Duvenaud, David K and Nickisch, Hannes and Rasmussen, Carl E},
  year = 2011,
  volume = {24},
  langid = {english}
}

@article{fengImplementationAnalyticGradients2019,
  title = {Implementation of Analytic Gradients for {{CCSD}} and {{EOM-CCSD}} Using {{Cholesky}} Decomposition of the Electron-Repulsion Integrals and Their Derivatives: {{Theory}} and Benchmarks},
  shorttitle = {Implementation of Analytic Gradients for {{CCSD}} and {{EOM-CCSD}} Using {{Cholesky}} Decomposition of the Electron-Repulsion Integrals and Their Derivatives},
  author = {Feng, Xintian and Epifanovsky, Evgeny and Gauss, J{\"u}rgen and Krylov, Anna I.},
  year = 2019,
  month = jul,
  journal = {The Journal of Chemical Physics},
  volume = {151},
  number = {1},
  pages = {014110},
  issn = {0021-9606, 1089-7690},
  doi = {10.1063/1.5100022},
  urldate = {2026-02-09},
  langid = {english}
}

@inproceedings{flaxmanFastKroneckerInference2015,
  title = {Fast {{Kronecker Inference}} in {{Gaussian Processes}} with Non-{{Gaussian Likelihoods}}},
  booktitle = {Proceedings of the 32nd {{International Conference}} on {{Machine Learning}}},
  author = {Flaxman, Seth and Wilson, Andrew Gordon and Neill, Daniel B and Nickisch, Hannes and Smola, Alexander J},
  year = 2015,
  langid = {english}
}

@inproceedings{gardnerGPyTorchBlackboxMatrixMatrix2018,
  title = {{{GPyTorch}}: {{Blackbox Matrix-Matrix Gaussian Process Inference}} with {{GPU Acceleration}}},
  booktitle = {Advances in {{Neural Information Processing Systems}} 31},
  author = {Gardner, Jacob R. and Pleiss, Geoff and Bindel, David and Weinberger, Kilian Q. and Wilson, Andrew Gordon},
  year = 2018,
  eprint = {1809.11165},
  primaryclass = {cs, stat},
  urldate = {2023-02-03},
  archiveprefix = {arXiv}
}

@inproceedings{gardnerProductKernelInterpolation2018,
  title = {Product {{Kernel Interpolation}} for {{Scalable Gaussian Processes}}},
  booktitle = {Proceedings of the 21st {{International Conference}} on {{Artificial Intelligence}} and {{Statistics}}},
  author = {Gardner, Jacob and Pleiss, Geoff and Wu, Ruihan and Weinberger, Kilian and Wilson, Andrew},
  year = 2018,
  month = mar,
  publisher = {PMLR},
  issn = {2640-3498},
  urldate = {2023-02-03},
  langid = {english}
}

@article{gilboaScalingMultidimensionalInference2015,
  title = {Scaling {{Multidimensional Inference}} for {{Structured Gaussian Processes}}},
  author = {Gilboa, Elad and Saatci, Yunus and Cunningham, John P.},
  year = 2015,
  month = feb,
  journal = {IEEE Trans. Pattern Anal. Mach. Intell.},
  volume = {37},
  number = {2},
  pages = {424--436},
  issn = {0162-8828, 2160-9292},
  doi = {10.1109/TPAMI.2013.192},
  urldate = {2026-01-19},
  copyright = {https://ieeexplore.ieee.org/Xplorehelp/downloads/license-information/IEEE.html}
}

@article{girardFastMonteCarloCrossvalidation1989,
  title = {A Fast `{{Monte-Carlo}} Cross-Validation' Procedure for Large Least Squares Problems with Noisy Data},
  shorttitle = {A Fast ?},
  author = {Girard, A.},
  year = 1989,
  month = jan,
  journal = {Numer. Math.},
  volume = {56},
  number = {1},
  pages = {1--23},
  issn = {0029-599X, 0945-3245},
  doi = {10.1007/BF01395775},
  urldate = {2026-03-31},
  copyright = {http://www.springer.com/tdm},
  langid = {english}
}

@article{gramacyLocalGaussianProcess2015,
  title = {Local {{Gaussian Process Approximation}} for {{Large Computer Experiments}}},
  author = {Gramacy, Robert B. and Apley, Daniel W.},
  year = 2015,
  month = apr,
  journal = {J. Comput. Graph. Stat.},
  volume = {24},
  number = {2},
  pages = {561--578},
  issn = {1061-8600, 1537-2715},
  doi = {10.1080/10618600.2014.914442},
  urldate = {2026-07-30},
  langid = {english}
}

@article{griebelDimensionwiseIntegrationHighdimensional2010,
  title = {Dimension-Wise Integration of High-Dimensional Functions with Applications to Finance},
  author = {Griebel, Michael and Holtz, Markus},
  year = 2010,
  month = oct,
  journal = {J. Complex.},
  volume = {26},
  number = {5},
  pages = {455--489},
  issn = {0885064X},
  doi = {10.1016/j.jco.2010.06.001},
  urldate = {2026-07-29},
  copyright = {https://www.elsevier.com/tdm/userlicense/1.0/},
  langid = {english}
}

@article{harbrechtLowrankApproximationPivoted2012,
  title = {On the Low-Rank Approximation by the Pivoted {{Cholesky}} Decomposition},
  author = {Harbrecht, Helmut and Peters, Michael and Schneider, Reinhold},
  year = 2012,
  month = apr,
  journal = {Appl. Numer. Math.},
  volume = {62},
  number = {4},
  pages = {428--440},
  issn = {01689274},
  doi = {10.1016/j.apnum.2011.10.001},
  urldate = {2023-02-06},
  langid = {english}
}

@inproceedings{hensmanGaussianProcessesBig2013,
  title = {Gaussian {{Processes}} for {{Big Data}}},
  booktitle = {Proceedings of the {{Twenty-Ninth Conference}} on {{Uncertainty}} in {{Artificial Intelligence}}},
  author = {Hensman, James and Fusi, Nicolo and Lawrence, Neil D.},
  year = 2013,
  pages = {282--290},
  langid = {english}
}

@article{hestenesMethodsConjugateGradients1952,
  title = {Methods of Conjugate Gradients for Solving Linear Systems},
  author = {Hestenes, M.R. and Stiefel, E.},
  year = 1952,
  month = dec,
  journal = {J. Res. Natl. Bur. Stan.},
  volume = {49},
  number = {6},
  pages = {409},
  issn = {0091-0635},
  doi = {10.6028/jres.049.044},
  urldate = {2026-03-31},
  langid = {english}
}

@inproceedings{holzmullerFastSparseGrid2021,
  title = {Fast {{Sparse Grid Operations Using}} the {{Unidirectional Principle}}: {{A Generalized}} and {{Unified Framework}}},
  shorttitle = {Fast {{Sparse Grid Operations Using}} the {{Unidirectional Principle}}},
  booktitle = {Sparse {{Grids}} and {{Applications}} -- {{Munich}} 2018},
  author = {Holzm{\"u}ller, David and Pfl{\"u}ger, Dirk},
  editor = {Bungartz, Hans-Joachim and Garcke, Jochen and Pfl{\"u}ger, Dirk},
  year = 2021,
  pages = {69--100},
  publisher = {Springer International Publishing},
  doi = {10.1007/978-3-030-81362-8_4},
  isbn = {978-3-030-81362-8},
  langid = {english}
}

@article{huangProjectionEstimationMultiple1998,
  title = {Projection Estimation in Multiple Regression with Application to Functional {{ANOVA}} Models},
  author = {Huang, Jianhua Z.},
  year = 1998,
  month = feb,
  journal = {Ann. Statist.},
  volume = {26},
  number = {1},
  issn = {0090-5364},
  doi = {10.1214/aos/1030563984},
  urldate = {2026-07-29},
  langid = {english}
}

@article{hutchinsonStochasticEstimatorTrace1989,
  title = {A {{Stochastic Estimator}} of the {{Trace}} of the {{Influence Matrix}} for {{Laplacian Smoothing Splines}}},
  author = {Hutchinson, M.F.},
  year = 1989,
  month = jan,
  journal = {Communications in Statistics - Simulation and Computation},
  volume = {18},
  number = {3},
  pages = {1059--1076},
  issn = {0361-0918, 1532-4141},
  doi = {10.1080/03610918908812806},
  urldate = {2026-03-31},
  langid = {english}
}

@inproceedings{hvarfnerVanillaBayesianOptimization2024,
  title = {Vanilla {{Bayesian Optimization Performs Great}} in {{High Dimensions}}},
  booktitle = {Proceedings of the 41st {{International Conference}} on {{Machine Learning}}},
  author = {Hvarfner, Carl and Hellsten, Erik Orm and Nardi, Luigi},
  year = 2024,
  month = dec,
  eprint = {2402.02229},
  primaryclass = {cs},
  publisher = {arXiv},
  doi = {10.48550/arXiv.2402.02229},
  urldate = {2026-04-22},
  archiveprefix = {arXiv}
}

@misc{ishidaHierarchicalAdditiveInteraction2025,
  title = {Hierarchical Additive Interaction Modelling with {{Gaussian}} Process Prior and Its Efficient Implementation for Multidimensional Grid Data},
  author = {Ishida, Sahoko and Panero, Francesca and Bergsma, Wicher},
  year = 2025,
  month = oct,
  number = {arXiv:2305.07073},
  eprint = {2305.07073},
  primaryclass = {stat},
  publisher = {arXiv},
  doi = {10.48550/arXiv.2305.07073},
  urldate = {2025-12-03},
  archiveprefix = {arXiv},
  langid = {english}
}

@article{jolliffePrincipalComponentAnalysis2016,
  title = {Principal Component Analysis: A Review and Recent Developments},
  author = {Jolliffe, Ian T and Cadima, Jorge},
  year = 2016,
  journal = {Phil. Trans. R. Soc. A.},
  volume = {374},
  number = {2065},
  doi = {10.1098/rsta.2015.0202},
  langid = {english}
}

@inproceedings{kingmaAdamMethodStochastic2017,
  title = {Adam: {{A Method}} for {{Stochastic Optimization}}},
  shorttitle = {Adam},
  booktitle = {Proceedings of the 3rd {{International Conference}} on {{Learning Representations}}},
  author = {Kingma, Diederik P. and Ba, Jimmy},
  year = 2015,
  eprint = {1412.6980},
  primaryclass = {cs},
  doi = {10.48550/arXiv.1412.6980},
  urldate = {2026-04-20},
  archiveprefix = {arXiv}
}

@article{koldaTensorDecompositionsApplications2009,
  title = {Tensor {{Decompositions}} and {{Applications}}},
  author = {Kolda, Tamara G. and Bader, Brett W.},
  year = 2009,
  month = aug,
  journal = {SIAM Rev.},
  volume = {51},
  number = {3},
  pages = {455--500},
  issn = {0036-1445, 1095-7200},
  doi = {10.1137/07070111X},
  urldate = {2023-12-18},
  langid = {english}
}

@article{kongstedAutomaticGenerationForce2006,
  title = {Automatic Generation of Force Fields and Property Surfaces for Use in Variational Vibrational Calculations of Anharmonic Vibrational Energies and Zero-Point Vibrational Averaged Properties},
  author = {Kongsted, Jacob and Christiansen, Ove},
  year = 2006,
  month = sep,
  journal = {J. Chem. Phys.},
  volume = {125},
  number = {12},
  pages = {124108},
  issn = {0021-9606, 1089-7690},
  doi = {10.1063/1.2352734},
  urldate = {2026-02-28},
  langid = {english}
}

@article{lanczosIterationMethodSolution1950,
  title = {An Iteration Method for the Solution of the Eigenvalue Problem of Linear Differential and Integral Operators},
  author = {Lanczos, C.},
  year = 1950,
  month = oct,
  journal = {J. Res. Natl. Bur. Stan.},
  volume = {45},
  number = {4},
  pages = {255},
  issn = {0091-0635},
  doi = {10.6028/jres.045.026},
  urldate = {2026-03-05},
  langid = {english}
}

@article{lawrenceProbabilisticNonlinearPrincipal2005,
  title = {Probabilistic {{Non}}-Linear {{Principal Component Analysis}} with {{Gaussian Process Latent Variable Models}}},
  author = {Lawrence, Neil},
  year = 2005,
  journal = {J. Mach. Learn. Res.},
  volume = {6},
  number = {60},
  pages = {1783--1816},
  langid = {english}
}

@inproceedings{ledererGaussianProcessbasedRealtime2021,
  title = {Gaussian Process-Based Real-Time Learning for Safety Critical Applications},
  booktitle = {Proceedings of the 38th International Conference on Machine Learning},
  author = {Lederer, Armin and Conejo, Alejandro J Ord{\'o}{\~n}ez and Maier, Korbinian A and Xiao, Wenxin and Umlauft, Jonas and Hirche, Sandra},
  year = 2021,
  volume = {139},
  publisher = {PMLR}
}

@article{liHighDimensionalModel2001,
  title = {High {{Dimensional Model Representations}}},
  author = {Li, Genyuan and Rosenthal, Carey and Rabitz, Herschel},
  year = 2001,
  month = aug,
  journal = {J. Phys. Chem. A},
  volume = {105},
  number = {33},
  pages = {7765--7777},
  issn = {1089-5639, 1520-5215},
  doi = {10.1021/jp010450t},
  urldate = {2026-07-29},
  langid = {english}
}

@article{liuWhenGaussianProcess2020,
  title = {When {{Gaussian Process Meets Big Data}}: {{A Review}} of {{Scalable GPs}}},
  shorttitle = {When {{Gaussian Process Meets Big Data}}},
  author = {Liu, Haitao and Ong, Yew-Soon and Shen, Xiaobo and Cai, Jianfei},
  year = 2020,
  month = nov,
  journal = {IEEE Trans. Neural Netw. Learning Syst.},
  volume = {31},
  number = {11},
  pages = {4405--4423},
  issn = {2162-237X, 2162-2388},
  doi = {10.1109/TNNLS.2019.2957109},
  urldate = {2026-07-30},
  copyright = {https://ieeexplore.ieee.org/Xplorehelp/downloads/license-information/IEEE.html},
  langid = {english}
}

@inproceedings{luAdditiveGaussianProcesses2022,
  title = {Additive {{Gaussian Processes Revisited}}},
  booktitle = {Proceedings of the 39th {{International Conference}} on {{Machine Learning}}},
  author = {Lu, Xiaoyu and Boukouvalas, Alexis and Hensman, James},
  year = 2022,
  langid = {english}
}

@article{maAdaptiveHighdimensionalStochastic2010,
  title = {An Adaptive High-Dimensional Stochastic Model Representation Technique for the Solution of Stochastic Partial Differential Equations},
  author = {Ma, Xiang and Zabaras, Nicholas},
  year = 2010,
  month = may,
  journal = {J. Comput. Phys.},
  volume = {229},
  number = {10},
  pages = {3884--3915},
  issn = {00219991},
  doi = {10.1016/j.jcp.2010.01.033},
  urldate = {2026-07-30},
  copyright = {https://www.elsevier.com/tdm/userlicense/1.0/},
  langid = {english}
}

@article{meurantLanczosConjugateGradient2006,
  title = {The {{Lanczos}} and Conjugate Gradient Algorithms in Finite Precision Arithmetic},
  author = {Meurant, G{\'e}rard and Strako{\v s}, Zden{\v e}k},
  year = 2006,
  month = may,
  journal = {Acta Numerica},
  volume = {15},
  pages = {471--542},
  issn = {0962-4929, 1474-0508},
  doi = {10.1017/S096249290626001X},
  urldate = {2025-11-27},
  copyright = {https://www.cambridge.org/core/terms},
  langid = {english}
}

@misc{muscoStabilityLanczosMethod2024,
  title = {Stability of the {{Lanczos Method}} for {{Matrix Function Approximation}}},
  author = {Musco, Cameron and Musco, Christopher and Sidford, Aaron},
  year = 2024,
  month = nov,
  number = {arXiv:1708.07788},
  eprint = {1708.07788},
  primaryclass = {cs},
  publisher = {arXiv},
  doi = {10.48550/arXiv.1708.07788},
  urldate = {2025-11-28},
  archiveprefix = {arXiv},
  langid = {english}
}

@inproceedings{pleissConstantTimePredictiveDistributions2018a,
  title = {Constant-{{Time Predictive Distributions}} for {{Gaussian Processes}}},
  booktitle = {Proceedings of the 35th {{International Conference}} on {{Machine Learning}}},
  author = {Pleiss, Geoff and Gardner, Jacob R and Weinberger, Kilian Q and Wilson, Andrew Gordon},
  year = 2018,
  langid = {english}
}

@article{rabitzEfficientInputOutput1999,
  title = {Efficient Input---Output Model Representations},
  author = {Rabitz, Herschel and Ali{\c s}, {\"O}mer F. and Shorter, Jeffrey and Shim, Kyurhee},
  year = 1999,
  month = mar,
  journal = {Comput. Phys. Comm.},
  volume = {117},
  number = {1-2},
  pages = {11--20},
  issn = {00104655},
  doi = {10.1016/S0010-4655(98)00152-0},
  urldate = {2026-07-30},
  copyright = {https://www.elsevier.com/tdm/userlicense/1.0/},
  langid = {english}
}

@article{rabitzGeneralFoundationsHighdimensional1999,
  title = {General Foundations of High-dimensional Model Representations},
  author = {Rabitz, Herschel and Alis, {\"O}mer F.},
  year = 1999,
  month = jun,
  journal = {J. Math. Chem.},
  volume = {25},
  number = {2-3},
  pages = {197--233},
  issn = {0259-9791, 1572-8897},
  doi = {10.1023/A:1019188517934},
  urldate = {2026-07-29},
  copyright = {https://www.springernature.com/gp/researchers/text-and-data-mining},
  langid = {english}
}

@article{rauhutEfficientCalculationPotential2004,
  title = {Efficient Calculation of Potential Energy Surfaces for the Generation of Vibrational Wave Functions},
  author = {Rauhut, Guntram},
  year = 2004,
  month = nov,
  journal = {J. Chem. Phys.},
  volume = {121},
  number = {19},
  pages = {9313--9322},
  issn = {0021-9606, 1089-7690},
  doi = {10.1063/1.1804174},
  urldate = {2021-02-17},
  langid = {english}
}

@article{rogersBypassingComputationalBottleneck2021,
  title = {Bypassing the Computational Bottleneck of Quantum-Embedding Theories for Strong Electron Correlations with Machine Learning},
  author = {Rogers, John and Lee, Tsung-Han and Pakdel, Sahar and Xu, Wenhu and Dobrosavljevi{\'c}, Vladimir and Yao, Yong-Xin and Christiansen, Ove and Lanat{\`a}, Nicola},
  year = 2021,
  month = jan,
  journal = {Phys. Rev. Research},
  volume = {3},
  number = {1},
  pages = {013101},
  issn = {2643-1564},
  doi = {10.1103/PhysRevResearch.3.013101},
  urldate = {2026-07-30},
  langid = {english}
}

@phdthesis{russiUncertaintyQuantificationExperimental2010,
  title = {Uncertainty {{Quantification}} with {{Experimental Data}} and {{Complex System Models}}},
  author = {Russi, Trent Michael},
  year = 2010,
  langid = {english},
  school = {University of California, Berkeley}
}

@book{saadIterativeMethodsSparse2007,
  title = {Iterative Methods for Sparse Linear Systems},
  author = {Saad, Youcef},
  year = 2007,
  edition = {2. ed.},
  publisher = {{SIAM, Society for Industrial and Applied Mathematics}},
  address = {Philadelphia, PA},
  isbn = {978-0-89871-534-7},
  langid = {english}
}

@phdthesis{saatciScalableInferenceStructured2012,
  title = {Scalable {{Inference}} for {{Structured Gaussian Process Models}}},
  author = {Saatci, Yunus},
  year = 2011,
  langid = {english},
  school = {University of Cambridge}
}

@article{seidlerAutomaticDerivationEvaluation2009,
  title = {Automatic Derivation and Evaluation of Vibrational Coupled Cluster Theory Equations},
  author = {Seidler, Peter and Christiansen, Ove},
  year = 2009,
  month = dec,
  journal = {J. Chem. Phys.},
  volume = {131},
  number = {23},
  pages = {234109},
  issn = {0021-9606, 1089-7690},
  doi = {10.1063/1.3272796},
  urldate = {2023-05-04},
  langid = {english}
}

@article{seidlerFastComputationsCorrelated2008,
  title = {Towards Fast Computations of Correlated Vibrational Wave Functions: {{Vibrational}} Coupled Cluster Response Excitation Energies at the Two-Mode Coupling Level},
  shorttitle = {Towards Fast Computations of Correlated Vibrational Wave Functions},
  author = {Seidler, Peter and Hansen, Mikkel Bo and Christiansen, Ove},
  year = 2008,
  month = apr,
  journal = {J. Chem. Phys.},
  volume = {128},
  number = {15},
  pages = {154113},
  issn = {0021-9606, 1089-7690},
  doi = {10.1063/1.2907860},
  urldate = {2024-06-10},
  langid = {english}
}

@article{simmonsComputingVibrationalSpectra2023,
  title = {Computing Vibrational Spectra Using a New Collocation Method with a Pruned Basis and More Points than Basis Functions: {{Avoiding}} Quadrature},
  shorttitle = {Computing Vibrational Spectra Using a New Collocation Method with a Pruned Basis and More Points than Basis Functions},
  author = {Simmons, Jesse and Carrington, Tucker},
  year = 2023,
  month = apr,
  journal = {J. Chem. Phys.},
  volume = {158},
  number = {14},
  pages = {144115},
  issn = {0021-9606, 1089-7690},
  doi = {10.1063/5.0146703},
  urldate = {2026-02-23},
  langid = {english}
}

@article{smithDifferentiationCholeskyAlgorithm1995,
  title = {Differentiation of the {{Cholesky Algorithm}}},
  author = {Smith, S. P.},
  year = 1995,
  month = jun,
  journal = {Journal of Computational and Graphical Statistics},
  volume = {4},
  number = {2},
  pages = {134--147},
  issn = {1061-8600, 1537-2715},
  doi = {10.1080/10618600.1995.10474671},
  urldate = {2026-03-31},
  langid = {english}
}

@inproceedings{snelsonLocalGlobalSparse2007,
  title = {Local and Global Sparse {{Gaussian}} Process Approximations},
  booktitle = {Proceedings of the {{Eleventh International Conference}} on {{Artificial Intelligence}} and {{Statistics}}},
  author = {Snelson, Edward and Ghahramani, Zoubin},
  year = 2007,
  month = mar,
  pages = {524--531},
  publisher = {PMLR},
  address = {San Juan, Puerto Rico}
}

@inproceedings{snelsonSparseGaussianProcesses2005,
  title = {Sparse {{Gaussian Processes}} Using {{Pseudo}}-Inputs},
  booktitle = {Advances in {{Neural Information Processing Systems}}},
  author = {Snelson, Edward and Ghahramani, Zoubin},
  year = 2005,
  volume = {18},
  langid = {english}
}

@article{spartaAdaptiveDensityguidedApproach2009,
  title = {An Adaptive Density-Guided Approach for the Generation of Potential Energy Surfaces of Polyatomic Molecules},
  author = {Sparta, Manuel and Toffoli, Daniele and Christiansen, Ove},
  year = 2009,
  month = aug,
  journal = {Theor. Chem. Acc.},
  volume = {123},
  number = {5-6},
  pages = {413--429},
  issn = {1432-881X, 1432-2234},
  doi = {10.1007/s00214-009-0532-1},
  urldate = {2024-06-10},
  copyright = {http://www.springer.com/tdm},
  langid = {english}
}

@article{stoneUsePolynomialSplines1994,
  title = {The Use of Polynomial Splines and Their Tensor Products in Multivariate Function Estimation},
  author = {Stone, Charles J.},
  year = 1994,
  journal = {Ann. Statist.},
  volume = {22},
  number = {1},
  pages = {118--171}
}

@inproceedings{titsiasBayesianGaussianProcess2010,
  title = {Bayesian {{Gaussian Process Latent Variable Model}}},
  booktitle = {Proceedings of the {{Thirteenth International Conference}} on {{Artificial Intelligence}} and {{Statistics}}},
  author = {Titsias, Michalis K and Lawrence, Neil D},
  year = 2010,
  langid = {english}
}

@inproceedings{titsiasVariationalLearningInducing2009,
  title = {Variational Learning of Inducing Variables in Sparse Gaussian Processes},
  booktitle = {Proceedings of the 12th {{International Conference}} on {{Artificial Intelligence}} and {{Statistics}}},
  author = {Titsias, Michalis},
  year = 2009,
  publisher = {PMLR}
}

@article{wangEffectiveDimensionQuasiMonte2003,
  title = {The Effective Dimension and Quasi-{{Monte Carlo}} Integration},
  author = {Wang, Xiaoqun and Fang, Kai-Tai},
  year = 2003,
  month = apr,
  journal = {J. Complex.},
  volume = {19},
  number = {2},
  pages = {101--124},
  issn = {0885064X},
  doi = {10.1016/S0885-064X(03)00003-7},
  urldate = {2026-08-24},
  copyright = {https://www.elsevier.com/tdm/userlicense/1.0/},
  langid = {english}
}

@article{wangEffectsDimensionReduction2006,
  title = {On the {{Effects}} of {{Dimension Reduction Techniques}} on {{Some High-Dimensional Problems}} in {{Finance}}},
  author = {Wang, Xiaoqun},
  year = 2006,
  month = dec,
  journal = {Oper. Res.},
  volume = {54},
  number = {6},
  pages = {1063--1078},
  issn = {0030-364X, 1526-5463},
  doi = {10.1287/opre.1060.0334},
  urldate = {2026-08-24},
  langid = {english}
}

@article{wangWhyAreHighDimensional2005,
  title = {Why {{Are High-Dimensional Finance Problems Often}} of {{Low Effective Dimension}}?},
  author = {Wang, Xiaoqun and Sloan, Ian H.},
  year = 2005,
  month = jan,
  journal = {SIAM J. Sci. Comput.},
  volume = {27},
  number = {1},
  pages = {159--183},
  issn = {1064-8275, 1095-7197},
  doi = {10.1137/S1064827503429429},
  urldate = {2026-08-24},
  langid = {english}
}

@inproceedings{wengerPreconditioningScalableGaussian2022,
  title = {Preconditioning for {{Scalable Gaussian Process Hyperparameter Optimization}}},
  booktitle = {Proceedings of the 39th {{International Conference}} on {{Machine Learning}}},
  author = {Wenger, Jonathan and Pleiss, Geoff and Hennig, Philipp and Cunningham, John P and Gardner, Jacob R},
  year = 2022,
  langid = {english}
}

@inproceedings{wilsonDeepKernelLearning2016,
  title = {Deep Kernel Learning},
  booktitle = {Proceedings of the 19th {{International Conference}} on {{Artificial Intelligence}} and {{Statistics}}},
  author = {Wilson, Andrew Gordon and Zhiting, Hu and Salakhutdinov, Ruslan and Xing, Eric P},
  year = 2016
}

@inproceedings{wilsonFastKernelLearning2014,
  title = {Fast {{Kernel Learning}} for {{Multidimensional Pattern Extrapolation}}},
  booktitle = {Advances in {{Neural Information Processing Systems}}},
  author = {Wilson, Andrew Gordon and Nehorai, Arye and Gilboa, Elad and Cunningham, John P},
  year = 2014,
  volume = {27},
  langid = {english}
}

@inproceedings{wilsonKernelInterpolationScalable2015,
  title = {Kernel {{Interpolation}} for {{Scalable Structured Gaussian Processes}} ({{KISS-GP}})},
  booktitle = {Proceedings of the 32nd {{International Conference}} on {{Machine Learning}}},
  author = {Wilson, Andrew Gordon and Nickisch, Hannes},
  year = 2015,
  volume = {37},
  langid = {english}
}

@article{wodraszkaPrunedCollocationbasedMulticonfiguration2019,
  title = {A Pruned Collocation-Based Multiconfiguration Time-Dependent {{Hartree}} Approach Using a {{Smolyak}} Grid for Solving the {{Schr\"odinger}} Equation with a General Potential Energy Surface},
  author = {Wodraszka, Robert and Carrington, Tucker},
  year = 2019,
  month = apr,
  journal = {J. Chem. Phys.},
  volume = {150},
  number = {15},
  pages = {154108},
  issn = {0021-9606, 1089-7690},
  doi = {10.1063/1.5093317},
  urldate = {2024-05-30},
  langid = {english}
}

@article{xuDecompositionMethodsStructural2005,
  title = {Decomposition Methods for Structural Reliability Analysis},
  author = {Xu, H. and Rahman, S.},
  year = 2005,
  month = jul,
  journal = {Probabilistic Eng. Mech.},
  volume = {20},
  number = {3},
  pages = {239--250},
  issn = {02668920},
  doi = {10.1016/j.probengmech.2005.05.005},
  urldate = {2026-07-30},
  copyright = {https://www.elsevier.com/tdm/userlicense/1.0/},
  langid = {english}
}

@article{yagiMultiresolutionPotentialEnergy2007,
  title = {Multiresolution Potential Energy Surfaces for Vibrational State Calculations},
  author = {Yagi, Kiyoshi and Hirata, So and Hirao, Kimihiko},
  year = 2007,
  month = aug,
  journal = {Theor. Chem. Acc.},
  volume = {118},
  number = {3},
  pages = {681--691},
  issn = {1432-881X, 1432-2234},
  doi = {10.1007/s00214-007-0363-x},
  urldate = {2026-02-28},
  copyright = {http://www.springer.com/tdm},
  langid = {english}
}

@article{yangImplementationPerformanceEvaluation2021,
  title = {Implementation and {{Performance Evaluation}} of a {{Bivariate Cut-HDMR Metamodel}} for {{Semiconductor Packaging Design Problems}} with a {{Large Number}} of {{Input Variables}}},
  author = {Yang, Yu-Hsiang and Wei, Hsiu-Ping and Han, Bongtae and Hu, Chao},
  year = 2021,
  month = aug,
  journal = {Materials},
  volume = {14},
  number = {16},
  pages = {4619},
  issn = {1996-1944},
  doi = {10.3390/ma14164619},
  urldate = {2026-07-30},
  langid = {english}
}

@article{zakUsingCollocationHierarchical2019,
  title = {Using Collocation and a Hierarchical Basis to Solve the Vibrational {{Schr\"odinger}} Equation},
  author = {Zak, Emil J. and Carrington, Tucker},
  year = 2019,
  month = may,
  journal = {J. Chem. Phys.},
  volume = {150},
  number = {20},
  pages = {204108},
  issn = {0021-9606, 1089-7690},
  doi = {10.1063/1.5096169},
  urldate = {2026-01-20},
  langid = {english}
}

@article{zhangResearchHighdimensionalModel2019,
  title = {Research on High-Dimensional Model Representation with Various Metamodels},
  author = {Zhang, Ning and Wang, Peng and Dong, Huachao},
  year = 2019,
  month = aug,
  journal = {Eng. Optim.},
  volume = {51},
  number = {8},
  pages = {1336--1351},
  issn = {0305-215X, 1029-0273},
  doi = {10.1080/0305215X.2018.1521398},
  urldate = {2026-07-30},
  langid = {english}
}

@article{zhaoDimensionalDecompositionaidedMetamodels2024,
  title = {Dimensional Decomposition-Aided Metamodels for Uncertainty Quantification and Optimization in Engineering: {{A}} Review},
  shorttitle = {Dimensional Decomposition-Aided Metamodels for Uncertainty Quantification and Optimization in Engineering},
  author = {Zhao, Heng and Fu, Chao and Zhang, Yaqiong and Zhu, Weidong and Lu, Kuan and Francis, Egbo M.},
  year = 2024,
  month = aug,
  journal = {Comput. Methods Appl. Mech. Eng.},
  volume = {428},
  pages = {117098},
  issn = {00457825},
  doi = {10.1016/j.cma.2024.117098},
  urldate = {2026-07-30},
  langid = {english}
}

@article{zhouKernelPrincipalComponent2020,
  title = {Kernel Principal Component Analysis-Based {{Gaussian}} Process Regression Modelling for High-Dimensional Reliability Analysis},
  author = {Zhou, Tong and Peng, Yongbo},
  year = 2020,
  month = dec,
  journal = {Comput. Struct.},
  volume = {241},
  pages = {106358},
  issn = {00457949},
  doi = {10.1016/j.compstruc.2020.106358},
  urldate = {2026-07-29},
  langid = {english}
}
